%% file: main.tex
\documentclass[10pt]{article} 

\usepackage[accepted]{rlj}  

\usepackage[utf8]{inputenc}
\usepackage[english]{babel}
\usepackage{longtable}

\usepackage{amssymb}
\usepackage{amsthm}        
\usepackage{amsfonts}
\usepackage{mathtools}
\usepackage{bm}
\usepackage{mathrsfs}
\usepackage{nicefrac}

\usepackage[ruled,vlined]{algorithm2e}

\usepackage{graphicx}
\usepackage[space]{grffile} 
\usepackage{float}
\usepackage{subcaption}
\usepackage{tikz}
\usetikzlibrary{calc,positioning,arrows}
\usepackage{booktabs}
\usepackage{multirow}
\usepackage{tabularx}

\usepackage{microtype}

\DeclareMathOperator*{\argmax}{arg\,max}
\DeclareMathOperator*{\argmin}{arg\,min}

\newcommand{\hfaf}{\hfill$\triangle$}

\theoremstyle{plain}
\newtheorem{theorem}{Theorem}[section]
\newtheorem{lemma}[theorem]{Lemma}
\newtheorem{proposition}[theorem]{Proposition}
\newtheorem{corollary}[theorem]{Corollary}

\theoremstyle{definition}
\newtheorem{definition}[theorem]{Definition}
\newtheorem{assumption}{Assumption}
\newtheorem{example}[theorem]{Example}

\theoremstyle{remark}
\newtheorem{remark}[theorem]{Remark}
\newtheorem*{remark*}{Remark}

\title{Coordination Graphs for Constrained Multi-Agent Reinforcement Learning}

\setrunningtitle{Coordination Graphs for CMARL}

\author{Santiago Amaya-Corredor\textsuperscript{1}, Miguel Calvo-Fullana\textsuperscript{1}, Anders Jonsson\textsuperscript{1}}

\emails{\{santiagoesteban.amaya, miguel.calvo, anders.jonsson\}@upf.edu}

\affiliations{
$^{1}$\textbf{Department of Engineering, Universitat Pompeu Fabra, Barcelona, Spain}
}

\contribution{
    A constrained MARL framework that integrates Lagrangian constraint
    handling into the coordination graph structure, addressing
    challenges absent in single-agent constrained RL: constraint credit
    assignment across overlapping pairwise regions, coherent composition
    of primary and cost Q-functions under factored action selection, and
    coordinated constraint satisfaction via Max-Sum message passing.
    A two-head Q-function pair (primary and constraint) decouples
    objective and constraint learning, enabling Pareto front evaluation
    from a single trained model by sweeping the Lagrange multiplier
    at evaluation time without retraining. Parameter sharing yields
    $O(1)$ networks regardless of team size $N$, and Max-Sum has
    cost polynomial in $N$, compared to exponential for centralized
    approaches.
    }
    {
    Demonstrated on cooperative navigation with up to 10 agents;
    pairwise graphs yield $O(N^2)$ regions and the framework
    accommodates sparser topologies.
    }

\contribution{
    Convergence guarantees (Q-learning for both heads and primal-dual)
    under independent transitions and reward factorization, together
    with a compositional error bound decomposing into
    structural, coordination, sampling, and (with function
    approximation) representation terms.
    The bound provides qualitative guidance for system design by
    identifying when each error source vanishes or dominates.
    }
    {
    Tabular convergence; function approximation convergence is to
    a KKT neighborhood.
    }

\contribution{
    Empirical validation on cooperative navigation with $N \in
    \{3,4,6,10\}$ agents showing that the Pareto front from a single
    CG-CMARL model dominates IQL, QMIX, DCG, MAPPO, and
    MAPPO-Lagrangian baselines, each requiring separately trained
    models at fixed reward-shaping ratios.
    }
    {
    Evaluated on Simple Spread only; environments with dependent
    transitions remain untested.
    }


\summary{
Prior work on cooperative MARL either exploits coordination graph
structure for scalability (DCG, CASEC, DMCG) without handling safety
constraints, or addresses constraints via centralized training
(MACPO, MAPPO-Lagrangian, CUTMAP) without exploiting factored
structure. CG-CMARL bridges this gap by integrating Lagrangian
constraint handling into the coordination graph framework, yielding a
fully decentralized architecture in which (i)~pairwise Q-functions
are learned on the factor graph via parameter sharing, (ii)~Max-Sum
message passing coordinates actions without access to the joint
action space, and (iii)~a Lagrangian multiplier controls the
objective--constraint tradeoff. Two shared Q-functions with separate
primary and constraint output heads decouple constraint handling from
learning, so that sweeping the Lagrange multiplier at evaluation time
traces a Pareto front from a single trained model without retraining.
We provide convergence guarantees under independent transitions and
reward factorization, together with a compositional four-source error
bound that reveals when factored methods incur no loss relative to
centralized learning. Experiments on cooperative navigation tasks show
that the resulting Pareto front dominates IQL, QMIX, DCG, MAPPO, and
MAPPO-Lagrangian baselines trained at fixed reward-shaping ratios.
}

\begin{document}

\maketitle  

\begin{abstract}
Constrained Multi-agent reinforcement learning (CMARL) faces two intertwined challenges:
the joint action space grows exponentially with the number of agents, and
additional requirements couple agents in ways that reward structure alone does
not capture. We introduce Coordination Graphs for Constrained Multi-Agent Reinforcement Learning (CG-CMARL), a framework that addresses both
challenges by combining coordination graphs with Lagrangian duality. The system decomposes the joint problem into pairwise regions,
each served by a set of shared Q-functions, one for the primary objective and one for each of the constraints, so that the number of learned models is independent of the number of agents. At execution time,
Max-Sum message passing coordinates actions across the factor graph,
while a Lagrangian multiplier controls the objective--constraint
tradeoff, allowing a single trained model to trace a Pareto front without retraining. We provide convergence guarantees under mild conditions, together with a
compositional error bound that decomposes into separate interpretable sources, each traceable to a specific design choice and independently controllable. Experiments on cooperative navigation tasks (where teams of up to 10 agents must coordinate to reach target positions while satisfying pairwise constraints) show that our method produces Pareto fronts dominating established baselines trained at fixed reward-shaping ratios, while scaling to team sizes where centralized approaches become intractable.
\end{abstract}

\section{Introduction}
\label{sec:introduction}
\input{sections/introduction}

\section{Problem Formulation}
\label{sec:problem}
\input{sections/problem}

\section{Method}
\label{sec:method}
\input{sections/method}

\section{Theoretical Analysis}
\label{sec:theory}
\input{sections/theory}

\section{Experiments}
\label{sec:experiments}
\input{sections/experiments}

\section{Conclusion}
\label{sec:conclusion}
\input{sections/conclusion}
\section*{Acknowledgements}
This work was supported in part by grants PID2023-153301NA-I00, PID2023-147145NB-I00, RYC2021-033549-I and CEX2021-001195-M funded by MCIN/AEI/10.13039/501100011033.


\bibliography{main}
\bibliographystyle{rlj}

\beginSupplementaryMaterials

\section{Notation and Definitions}
\label{app:notation}
\input{suplementary/appA_notation}

\section{Algorithm Details}
\label{app:algorithm}
\input{appendices/appB_algorithm}

\section{Proofs}
\label{app:proofs}
\input{appendices/appC_proofs}

\section{Implementation and Experimental Details}
\label{app:implementation}
\input{suplementary/appD_implementation}

\end{document}

%% file: sections/introduction.tex
Multi-agent reinforcement learning (MARL) has emerged as a powerful
framework for sequential decision-making in systems where multiple
autonomous agents must coordinate their
actions~\citep{zhang2021multiagent,albrecht2024marl}. A central
challenge is \emph{scalability}: the joint action space grows
exponentially with the number of agents, rendering centralized
approaches intractable even for moderate team sizes. To address this,
factored MDPs~\citep{boutilier2000factored} and coordination
graphs~\citep{guestrin2002contextspecific,kok2006collaborative}
exploit the locality of agent interactions, decomposing the global
value function into regional components and reducing complexity from
exponential in the number of agents to exponential in the size of the largest
interacting group. Message-passing algorithms such as
Max-Sum~\citep{farinelli2008decentralised,rogers2011bounded} then
enable distributed action coordination without centralized computation.

Many multi-agent tasks additionally impose \emph{safety constraints}
that cannot be traded off against reward. Collision avoidance, energy
budgets, and bandwidth limits must be enforced. Constrained Markov
decision processes (CMDPs)~\citep{altman1999constrained} formalize
such specifications. For single-agent settings, primal-dual methods
based on Lagrangian relaxation have proven
effective~\citep{tessler2018reward,paternain2023safe}, and recent
work has established that CMDPs exhibit zero duality gap under mild
conditions~\citep{paternain2019constrained}.

However, extending constrained RL to the multi-agent setting
introduces challenges that existing methods do not adequately address.
Centralized primal-dual approaches require access to the global state
and joint action space, precisely the bottleneck that factored methods
aim to overcome. Independent learners cannot coordinate constraint
satisfaction: if each agent optimizes its own Lagrangian independently,
the team may collectively violate global constraints. The core
difficulty is that \emph{constraints couple agents} in ways that
reward structure alone does not capture. To the best of our knowledge,
no prior work combines the scalability of factored value functions
with the constraint-handling capabilities of Lagrangian methods in a
decentralized architecture with convergence guarantees.

We introduce \textbf{CG-CMARL}, a framework that addresses both scalability and constraint
satisfaction. The key idea is a \emph{two-head Q-network} shared
across all pairwise regions of a coordination graph: one head predicts
expected coverage reward, the other predicts expected constraint reward.
A Lagrangian multiplier $\lambda$ combines the two heads at action
selection time, $Q_{\mathrm{aug}} = Q_{\mathrm{prim}} + \lambda\, Q_{\mathrm{cost}}$, so that sweeping $\lambda$ over a single
trained model traces a Pareto front of coverage versus safety without
retraining. Max-Sum message passing on the factor graph coordinates
actions across agents, while parameter sharing ensures the number of
learned Q-functions is $O(1)$ regardless of team size.
While individual components (two-head value functions, Lagrangian
multipliers, coordination graphs) exist in prior work, their
integration raises challenges that do not arise in isolation:
constraints couple agents across overlapping regions, the cost head
must evaluate the primary-greedy policy to ensure coherent
composition under factored action selection, and the compositional
error bound must account for structural and coordination errors
specific to the factored architecture.

\begin{itemize}[leftmargin=*, nosep]

\item \textbf{A decentralized constrained MARL framework with
Lagrangian Pareto sweep.}
A two-head Q-function pair (primary and constraint) decouples
objective and constraint learning, enabling a single trained model to
trace the full objective--constraint Pareto front by varying
$\lambda$ at evaluation time without retraining. Parameter sharing
across all pairwise regions yields $O(1)$ networks regardless of
team size $N$, and damped Max-Sum message passing coordinates joint
actions on the resulting factor graph with cost polynomial in $N$.

\item \textbf{Convergence guarantees and compositional error bound.}
Under independent transitions and reward factorization, we establish
Q-function convergence for both heads and dual optimality
(Theorems~\ref{thm:tabular-convergence}--\ref{thm:dual-convergence}),
together with a four-source error decomposition
(Theorem~\ref{thm:total-error}) that identifies when each source
vanishes and provides qualitative guidance for system design.

\item \textbf{Empirical validation.} On the Simple Spread cooperative
navigation benchmark~\citep{lowe2017multi} with $N \in
\{3,4,6,10\}$ agents, the Pareto front from a single CG-CMARL model
dominates IQL, QMIX, DCG, MAPPO, and MAPPO-Lagrangian baselines,
each requiring separately trained models at fixed reward-shaping
ratios.

\end{itemize}


\subsection{Related Work}
\label{sec:related}

Our work lies at the intersection of value decomposition for
cooperative MARL, constrained reinforcement learning, and distributed
optimization.

\paragraph{Value Decomposition and Coordination Graphs.}
VDN~\citep{sunehag2018vdn} represents the joint Q-function as a sum
of per-agent utilities; QMIX~\citep{rashid2018qmix} relaxes this to
monotonic combinations via a mixing network;
QTRAN~\citep{son2019qtran} removes monotonicity at the cost of more
complex training. Coordination
graphs~\citep{guestrin2002contextspecific,kok2006collaborative}
explicitly model pairwise or higher-order dependencies. Deep
Coordination Graphs (DCG)~\citep{bohmer2020deep} scale this to
high-dimensional settings with parameter sharing, and subsequent work
has explored dynamic~\citep{siu2021dynamic},
sparse~\citep{wang2022casec},
group-aware~\citep{duan2024gacg},
influence-based~\citep{zhang2025iescg}, and meta-learned
topologies~\citep{gupta2025dmcg}. While these methods advance
coordination graph representations, \emph{none address constraint
satisfaction}. Our work extends coordination graphs to the constrained
setting by integrating Lagrangian dual variables into the factored
structure.

\paragraph{Constrained and Safe RL.}
CMDPs~\citep{altman1999constrained} formalize safety via auxiliary
reward thresholds. Primal-dual methods alternate between policy
improvement and multiplier
updates~\citep{borkar2005actor,tessler2018reward,stooke2020responsive};
\citet{paternain2019constrained} established zero duality gap, and
\citet{calvofullana2024state} proposed state augmentation to address
limitations of standard primal-dual approaches. For multi-agent
settings, MACPO~\citep{gu2023safe} adapts CPO using centralized
training, \citet{zhang2024scalable} improve scalability via
truncated advantages, and CUTMAP~\citep{guan2025cutmap} constrains
cooperative policies using offline data under the CTDE paradigm.
A recent survey~\citep{kushwaha2025saferl} identifies decentralized
constraint handling in MARL as a key open problem. Our approach
exploits the factored structure of both rewards and constraints,
enabling decentralized learning without centralized critics.

\paragraph{Distributed Optimization and Credit Assignment.}
Subgradient methods with consensus
averaging~\citep{nedic2009distributed} and two-timescale stochastic
approximation~\citep{borkar1997stochastic} provide the theoretical
foundation for our multiplier updates.
\citet{agorio2024multiagent} applied gossiped dual variables to
multi-agent assignment and considered stochastic graphs~\citep{agorio2025cooperative}. For credit assignment, difference
rewards~\citep{wolpert2001optimal} and multi-level counterfactual
advantages~\citep{zhao2025maca} decompose global feedback into
agent-level signals; we adopt per-agent counterfactual rewards for regional reward shaping. In the multi-objective setting,
\citet{liu2024cmorl} discover Pareto fronts via constrained
optimization steps, our Lagrangian sweep achieves a similar effect
but within a factored multi-agent architecture. To our knowledge, no
prior work combines coordination graph value decomposition with
Lagrangian constraint handling.

Table~\ref{tab:comparison} summarizes the key distinctions.

\begin{table}[t]
\centering
\small
\caption{Comparison with related methods. ``Decentr.''\ = no global
state access at execution. ``Constr.''\ = explicit CMDP-style
constraints. ``CG''\ = coordination graph with message passing.}
\label{tab:comparison}
\begin{tabular}{@{}lccc@{}}
\toprule
\textbf{Method} & \textbf{Decentr.} & \textbf{Constr.} & \textbf{CG} \\
\midrule
VDN / QMIX & \checkmark & \texttimes & \texttimes \\
DCG / CASEC / DMCG / IESCG & \checkmark & \texttimes & \checkmark \\
MACPO / MAPPO-Lag. & \texttimes & \checkmark & \texttimes \\
CUTMAP~\citep{guan2025cutmap} & \texttimes & \checkmark & \texttimes \\
Scalable CMPO~\citep{zhang2024scalable} & Partial & \checkmark & \texttimes \\
\textbf{CG-CMARL (ours)} & \checkmark & \checkmark & \checkmark \\
\bottomrule
\end{tabular}
\end{table}

%% file: sections/problem.tex
The goal of this section is to formalize the constrained multi-agent
coordination problem and the structural assumptions that enable
scalable, decentralized solutions.

Let $\mathcal{V} = \{1, \ldots, N\}$ denote a team of $N$
cooperative agents. Each agent $i \in \mathcal{V}$ has a local state
space $\mathcal{S}^i$ and action space $\mathcal{A}^i$. The joint
state is $s = (s^1, \ldots, s^N) \in \mathcal{S} = \prod_i
\mathcal{S}^i$ and the joint action is $a = (a^1, \ldots, a^N) \in
\mathcal{A} = \prod_i \mathcal{A}^i$. The joint action space
$|\mathcal{A}| = \prod_i |\mathcal{A}^i|$ grows exponentially with
$N$, motivating the factored approach developed below.

\label{sec:coord-graph}
In many multi-agent tasks, interactions are \emph{local}: the reward
obtained by agent $i$ depends not on all $N$ agents but only on a
small subset of neighbors. We represent these dependencies using a
\emph{coordination graph}~\citep{guestrin2002contextspecific}, a
hypergraph
$\mathcal{G}_R = (\mathcal{V}, \mathcal{C}_R)$ where each
\emph{region} $R \in \mathcal{C}_R$ is a subset of agents whose joint
state-action affects a reward component. For a region $R$, we write
$s^R = (s^i)_{i \in R}$ and $a^R = (a^i)_{i \in R}$ for the local
projections.

The coordination graph induces a factorization of the reward. We
consider $J+1$ reward signals: a primary objective $r_0$ and $J$
constraint signals $r_1, \ldots, r_J$, each decomposing additively
over regions:
\begin{equation}
\label{eq:reward-decomp}
r_j(s, a) = \sum_{R \in \mathcal{C}_R} r_j^R(s^R, a^R),
\quad j \in \{0, 1, \ldots, J\}.
\end{equation}
In our setting, $\mathcal{C}_R$ is the \emph{pairwise} coordination
graph containing all agent pairs: $\mathcal{C}_R = \{\{i,k\} : i,k
\in \mathcal{V}, i \neq k\}$. This captures pairwise interactions
such as collision penalties and yields $\binom{N}{2}$ regions which are quadratic
in $N$ rather than exponential.

\label{sec:assumptions}
Our approach relies on two structural assumptions:
\begin{description}[leftmargin=1.5em, labelindent=0em,
  font=\normalfont\bfseries, nosep]
\item[(A1) Independent Transitions.]
The transition kernel factorizes across agents:
$P(s' \mid s, a) = \prod_{i=1}^{N} P_i(s'^i \mid s^i, a^i)$.
That is, each agent's next state depends only on its own current
state and action. Inter-agent interactions (e.g., collisions, coordination)
are captured through the reward
structure~\eqref{eq:reward-decomp}, not through transition
dynamics, a design adopted by most multi-agent
environments~\citep{lowe2017multi,terry2021pettingzoo}.

\item[(A2) Reward Factorization.]
Each reward signal decomposes as in~\eqref{eq:reward-decomp}. We
additionally assume bounded rewards: $|r_j^R(s^R, a^R)| \leq
R_{\max}$ for all $j, R$.
\end{description}
Under (A1)--(A2), the value function $V_j^\pi(s) := \mathbb{E}_\pi
[\sum_{t=0}^\infty \gamma^t r_j(s_t, a_t) \mid s_0 = s]$
decomposes as $V_j^\pi(s) = \sum_{R \in \mathcal{C}_R}
V_{j,R}^{\pi}(s^R)$, where $V_{j,R}^{\pi}$ is the regional value
function. This decomposition, which follows from linearity of
expectation and the independence of transitions, is the foundation
of our factored Q-learning approach. That is, the expected return
of each region can be learned \emph{independently} using only local
transitions.

\begin{remark*}[State vs.\ Observation]
The theoretical analysis in Section~\ref{sec:theory} is developed
for the regional Markov state $s^R$. In practice, each region
receives an observation $o^R$ constructed from locally available
information (Section~\ref{sec:exp-setup}). For the pairwise
coordination graph on Simple Spread, $o^R$ contains the full
Markov state of both agents in the region (positions and
velocities); partial observability arises only from not observing
out-of-region agents, and this discrepancy is captured by the
structural error $\beta$ in our compositional bound
(Theorem~\ref{thm:total-error}).
\end{remark*}

\label{sec:constrained-opt}
With this decomposition in hand, we formulate the constrained problem.
Let $c_1, \ldots, c_J \in \mathbb{R}$ be constraint thresholds. We
restrict attention to factorized policies $\pi(a \mid s) = \prod_i
\pi^i(a^i \mid s^i)$ and seek:
\begin{equation}
\label{eq:cmarl}
\begin{aligned}
\max_{\pi \in \Pi} \quad & V_0^\pi(s_0) \\[4pt]
\text{subject to} \quad & V_j^\pi(s_0) \geq \frac{c_j}{1-\gamma},
\quad j = 1, \ldots, J.
\end{aligned}
\end{equation}
The objective maximizes expected discounted primary reward subject to
each constraint reward exceeding a threshold. The normalization by
$(1-\gamma)$ converts the infinite-horizon sum to an average-reward
scale, making $c_j$ interpretable as a per-step requirement. Using the
value decomposition, both objective and constraints decompose over
regions, suggesting that optimization can be performed in a distributed
manner.

We additionally require a regularity condition:
\begin{description}[leftmargin=1.5em, labelindent=0em,
  font=\normalfont\bfseries, nosep]
\item[(A3) Feasibility (Slater's Condition).]
There exists a policy $\tilde{\pi} \in \Pi$ and $\xi > 0$ such that
$V_j^{\tilde{\pi}}(s_0) \geq c_j/(1-\gamma) + \xi$ for all $j$.
This guarantees the existence of a strictly feasible policy,
ensuring that the constraint set has nonempty interior and that
strong duality holds for the Lagrangian relaxation.
\end{description}

%% file: sections/method.tex

Our goal is to solve the constrained cooperative
problem~\eqref{eq:cmarl} in a fully decentralized manner, without
access to the global state or joint action space at execution time.
For clarity, we present the algorithm for a single constraint
($J=1$); the extension to multiple constraints requires one
multiplier and one Q-function per constraint, but is otherwise
straightforward. The algorithm decomposes into three components,
each addressing a distinct aspect of the problem:
\begin{itemize}[nosep, leftmargin=2em]
\item \emph{Two Q-functions per region} (Section~\ref{sec:two-head})
that separately learn primary reward and constraint reward.
\item \emph{Max-Sum message passing} (Section~\ref{sec:layer2}) for
coordinated action selection on the factor graph.
\item A \emph{Lagrangian constraint mechanism}
(Section~\ref{sec:lagrangian}) that controls the
objective--constraint tradeoff, enabling Pareto front evaluation
from a single trained model.
\end{itemize}

\subsection{Two Q-Functions per Region}
\label{sec:two-head}

A standard approach to constrained MDPs is to form an augmented reward
$\tilde{r}_\lambda = r_0 + \lambda r_1$ and learn a single
Q-function for this combined signal~\citep{tessler2018reward}. However,
this bakes the constraint tradeoff into the learned representation:
changing $\lambda$ requires retraining. To decouple the tradeoff from
learning, we instead maintain two Q-functions per region.

For each region $R \in \mathcal{C}_R$, we define the optimal regional Q-function for the primary reward $r_0^R$, denoted by $Q_{\mathrm{prim},R}(s^R, a^R)$; and the constraint reward for $r_1^R$ under the primary Q-function's greedy policy, $Q_{\mathrm{cost},R}(s^R, a^R)$. Since the constraint in~\eqref{eq:cmarl} is formulated on top of a reward instead of a cost, it is generally non-positive for cost-based constraints, and $c_j$ is negative. Both Q-functions map regional state-action pairs to scalar values and are trained with independent TD targets. The primary Q-function performs standard Q-learning on reward $r_0^R$:
\begin{equation}
\label{eq:td-prim}
y_{\mathrm{prim}}^R = r_0^R + \gamma \,
\max_{a'^R} Q_{\mathrm{prim},R}(s'^R, a'^R).
\end{equation}
The cost Q-function performs \emph{policy evaluation} under the primary
head's greedy policy, rather than its own maximization:
\begin{equation}
\label{eq:td-cost}
y_{\mathrm{cost}}^R = r_1^R + \gamma \,
Q_{\mathrm{cost},R}\!\left(s'^R,\;
\argmax_{a'^R} Q_{\mathrm{prim},R}(s'^R, a'^R)\right).
\end{equation}
This ensures that the Q-function predicts constraint rewards under the
policy induced by the primary objective, not under a
cost-minimizing policy. In this sense, $Q_{\mathrm{cost}}$ acts as a \emph{critic} for the
primary policy with respect to the cost signal: it evaluates the
consequences of primary-greedy decisions, rather than learning its
own optimal cost policy. If $Q_{\mathrm{cost}}$ instead learned its
own greedy policy via
$\max_{a'^R} Q_{\mathrm{cost},R}(s'^R, a'^R)$, it would predict the constraint reward achievable by a constraint-maximizing agent, overestimating the constraint reward that the primary-greedy agent actually achieves. Combining such a
$Q_{\mathrm{cost}}$ with $Q_{\mathrm{prim}}$
in~\eqref{eq:augmented-q} would be incoherent: the two Q-functions
would reflect different policies, and the Lagrangian tradeoff would
be miscalibrated.

At action selection time, the two Q-functions are combined via a Lagrangian multiplier $\lambda \geq 0$:
\begin{equation}
\label{eq:augmented-q}
Q_{\mathrm{aug},R}(s^R, a^R) =
Q_{\mathrm{prim},R}(s^R, a^R)
+ \lambda  Q_{\mathrm{cost},R}(s^R, a^R).
\end{equation}
That is, the augmented Q-function combines the primary objective with the predicted constraint reward, weighted by $\lambda$. Higher $\lambda$ increases the weight on constraint satisfaction; lower $\lambda$ prioritizes the primary objective. To illustrate, consider
a state with two available actions:
\begin{center}
\small
\begin{tabular}{@{}lcccc@{}}
\toprule
Action & $Q_{\mathrm{prim}}$ & $Q_{\mathrm{cost}}$ &
$Q_{\mathrm{aug}}$ ($\lambda{=}0.5$) &
$Q_{\mathrm{aug}}$ ($\lambda{=}2$) \\
\midrule
A (high primary, low constraint) & 10 & $-5$ & 7.5 & 0 \\
B (moderate primary, high constraint) & 9 & $-1$ & 8.5 & 7 \\
\bottomrule
\end{tabular}
\end{center}
At $\lambda = 0$ the agent selects A (best primary reward); at $\lambda = 0.5$ the agent selects B, because the low constraint reward of A outweighs its small primary advantage. Neither Q-function alone determines avoidance behavior; it emerges from the $\argmax$ over their combination. Note that $Q_{\mathrm{cost}}$ evaluates the primary-greedy policy
from the next step onward, while the agent selects actions according
to $Q_{\mathrm{aug}}$ at every step. The constraint reward predictions therefore underestimate the constraint reward the agent actually achieves (since the agent selects via $Q_{\mathrm{aug}}$, which avoids collisions more than the primary-greedy policy that $Q_{\mathrm{cost}}$ evaluates). This is conservative for constraint satisfaction: the Lagrangian assigns less safety credit than warranted, biasing the policy toward greater caution. The key advantage is that
$\lambda$ can be varied at \emph{evaluation time} without retraining:
sweeping $\lambda$ over a fixed model traces the objective--constraint
Pareto front (Section~\ref{sec:lagrangian}).

All regions with the same structure (i.e., the same number of agents)
share a single set of Q-function parameters, trained on pooled
transitions from all such regions. For the pairwise coordination
graph, every region contains exactly two agents, so a single
parameter set serves all $\binom{N}{2}$ pairs. This gives $O(1)$
learned models regardless of $N$, which is critical for scaling to
large teams.

\begin{remark*}[Advantage over Single-Reward Lagrangian]
The standard approach of learning a single Q-function for
$\tilde{r}_\lambda = r_0 + \lambda r_1$ requires retraining for every
$\lambda$. The two Q-function formulation instead learns
$Q_{\mathrm{prim}}$ and $Q_{\mathrm{cost}}$ once, and composes them at
action selection time via~\eqref{eq:augmented-q}. This converts the
constrained problem into a family of unconstrained MDPs parameterized
by $\lambda$, all approximately solvable from a single set of trained
parameters, up to the conservative approximation discussed above.
\end{remark*}

\paragraph{$\lambda$-Augmented Variant.}
An alternative that eliminates this mismatch is to condition the
Q-function directly on $\lambda$ via state
augmentation~\citep{calvofullana2024state}. Instead of maintaining
two Q-functions, we learn a single Q-function
$Q_R(s^R, a^R; \lambda)$ that takes the multiplier as an additional
input. During training, $\lambda$ is sampled uniformly from
$[0, \lambda_{\max}]$ at the start of each episode, and the
Q-function is trained on the augmented reward
$r_\lambda = r_0 + \lambda\, r_1$ via standard Q-learning:
\begin{equation}
\label{eq:td-augmented}
y^R = r_0^R + \lambda\, r_1^R + \gamma\,
\max_{a'^R} Q_R(s'^R, a'^R; \lambda).
\end{equation}
For each sampled $\lambda$, the Q-function learns the optimal policy
for that specific tradeoff, so the cost predictions are
self-consistent: the policy evaluated by $Q_R(\cdot; \lambda)$ is
the same policy that is greedy with respect to
$Q_R(\cdot; \lambda)$. At evaluation time, sweeping $\lambda$ as
input traces the Pareto front in the same way as the two-head
variant. The tradeoff is that the Q-function must now generalize
across both state-action pairs and $\lambda$ values, which increases
the sample complexity of training. This variant feeds its regional
Q-tables into the same Max-Sum coordination and parameter sharing
pipeline described in Section~\ref{sec:layer2}, and constitutes a
promising direction that we leave for future empirical investigation.

\label{sec:obs-reward-design}

The method requires only that each region $R$ receives a local
observation $o^R$ and that the rewards $r_0^R$, $r_1^R$ satisfy the
additive structure~\eqref{eq:reward-decomp}. The specific
instantiation for our experimental domain is described in
Section~\ref{sec:experiments}.
\subsection{Max-Sum Action Coordination}
\label{sec:layer2}

Given the regional Q-tables $\{Q_{\mathrm{aug},R}\}$,
agents coordinate actions by solving:
\begin{equation}
\label{eq:coord-problem}
\max_{a \in \mathcal{A}} \sum_{R \in \mathcal{C}_R}
Q_{\mathrm{aug},R}(s^R, a^R).
\end{equation}
Direct enumeration is intractable ($|\mathcal{A}|^N$ joint actions),
so we exploit the factorization via \emph{Max-Sum message
passing}~\citep{farinelli2008decentralised,rogers2011bounded} on the
factor graph $\mathcal{G}_F$ induced by the coordination graph.

The factor graph is bipartite: \emph{variable nodes} $\{a^i\}_{i \in
\mathcal{V}}$ represent agent actions, \emph{factor nodes}
$\{Q_{\mathrm{aug},R}\}_{R \in \mathcal{C}_R}$ represent regional
Q-functions, and edges connect $a^i$ to $Q_{\mathrm{aug},R}$ when
$i \in R$. Max-Sum operates by iteratively exchanging messages between
variable and factor nodes.

\paragraph{Factor-to-Variable Messages.} For factor $Q_R$ and
variable $a^i$ with $i \in R$:
\begin{equation}
\label{eq:msg-factor-to-var}
M_{R \to i}^{(n)}(a^i) = \max_{a^{R \setminus i}}
\!\left[ Q_{\mathrm{aug},R}(s^R, a^R) + \!\sum_{k \in R \setminus
\{i\}} M_{k \to R}^{(n-1)}(a^k) \right].
\end{equation}
Intuitively, factor $Q_R$ sends agent $i$ the best payoff achievable
from region $R$, maximized over the actions of all other agents in $R$
and accounting for their incoming messages.

\paragraph{Variable-to-Factor Messages.} For variable $a^i$ and
factor $Q_R$ with $i \in R$:
\begin{equation}
\label{eq:msg-var-to-factor}
M_{i \to R}^{(n)}(a^i) = \sum_{\substack{R' \in \mathcal{C}_R :\\
i \in R',\, R' \neq R}} M_{R' \to i}^{(n)}(a^i).
\end{equation}
That is, agent $i$ tells factor $Q_R$ the total value it receives from
all \emph{other} regions, summarizing its external commitments. This
exclusion prevents double-counting: factor $R$ should not receive
evidence about $a^i$ that originated from itself.

Since the pairwise factor graph contains cycles for $N \geq 3$,
we apply message damping to stabilize convergence~\citep{murphy1999loopy}; the damping schedule is
specified in Appendix~\ref{app:layer-details}.

After $K$ iterations, each agent selects its action using the
accumulated messages [cf.~\eqref{eq:msg-factor-to-var}]:
\begin{equation}
\label{eq:action-selection}
a^i = \argmax_{a^i \in \mathcal{A}^i} \sum_{R : i \in R}
M_{R \to i}^{(K)}(a^i).
\end{equation}

For pairwise regions ($|R| = 2$), each message
requires $O(|\mathcal{A}|)$ operations, giving total complexity
$O(\binom{N}{2}  |\mathcal{A}|  K)$ per step (polynomial
in $N$), compared to $O(|\mathcal{A}|^N)$ for exhaustive search.

\subsection{Lagrangian Constraint Handling}
\label{sec:lagrangian}

We adopt a Lagrangian relaxation to handle the
constraints in~\eqref{eq:cmarl}. The Lagrangian is:
\begin{equation}
\label{eq:lagrangian}
\mathcal{L}(\pi, \lambda) = V_0^\pi(s_0)
+ \sum_{j=1}^{J} \lambda_j
\!\left( V_j^\pi(s_0) - \frac{c_j}{1-\gamma} \right),
\end{equation}
where $\lambda_j \geq 0$ is the multiplier for constraint $j$. Under
Assumption~(A3), strong duality
holds~\citep{paternain2019constrained}: the constrained
problem~\eqref{eq:cmarl} is equivalent to
$\min_{\lambda \geq 0} \max_\pi \mathcal{L}(\pi, \lambda)$.

For fixed $\lambda$, maximizing the Lagrangian reduces to an
\emph{unconstrained} MDP with the augmented
Q-function~\eqref{eq:augmented-q}. the two-head formulation solves
this directly: the primary Q-function learns $V_0$ and the Q-function learns
$V_j$, and $\lambda$ controls the combination
[cf.~\eqref{eq:augmented-q}].

During training, each agent $i$ maintains a
local multiplier $\lambda^i$, updated via projected dual ascent after
each episode based on observed constraint rewards:
\begin{equation}
\label{eq:dual-update}
\lambda^i \leftarrow \Pi_{[0,\,\lambda_{\max}]}\!\left(\lambda^i - \eta
(\hat{V}_1^i - c_{\mathrm{thresh}})\right),
\end{equation}
where $\hat{V}_1^i$ is agent $i$'s observed per-episode constraint reward, $c_{\mathrm{thresh}}$ is the constraint threshold, and $\eta$ is the
dual learning rate (decayed on a slower timescale than the Q-learning
rate). Hence, if agent $i$ observes constraint reward below the threshold, its local $\lambda^i$ increases, shifting the
policy toward feasibility. The average multiplier
$\bar{\lambda} = \frac{1}{N} \sum_i \lambda^i$ is used
in~\eqref{eq:augmented-q} during action selection.

At evaluation time, $\lambda$ is set externally and swept across a
range. For each value, the trained model is evaluated over multiple
episodes, recording mean primary objective and constraint cost. This
traces the objective--constraint Pareto front from a \emph{single}
trained model with no retraining required.

\label{sec:algorithm-summary}

Algorithm~\ref{alg:summary} summarizes the training procedure.

\begin{algorithm}[t]
\caption{CG-CMARL Training -- Two-Head Variant (One Episode)}
\label{alg:summary}
\KwIn{Q-functions $Q_{\mathrm{prim},R}$, $Q_{\mathrm{cost},R}$ for
all $R \in \mathcal{C}_R$, multipliers $\{\lambda^i\}$, replay
buffer $\mathcal{D}$}
\For{each time step $t$}{
  Build region observations $\{o^R\}$ from local state projections\;
  Compute $Q_{\mathrm{aug},R} \leftarrow Q_{\mathrm{prim},R} +
  \bar{\lambda}\, Q_{\mathrm{cost},R}$ for all $R$\;
  Run damped Max-Sum for $K$ iterations on
  $\{Q_{\mathrm{aug},R}\}$\;
  Select actions via~\eqref{eq:action-selection}\;
  Execute $a_t$, observe $s_{t+1}$, compute regional rewards
  $\{r_0^R\}$ and constraint rewards $\{r_1^R\}$\;
  Store per-region transitions $(o^R, a^R, r_0^R, r_1^R, o'^R)$ in
  $\mathcal{D}$\;
  Sample batch from $\mathcal{D}$; update $Q_{\mathrm{prim},R}$,
  $Q_{\mathrm{cost},R}$ via
  Eqs.~\eqref{eq:td-prim}--\eqref{eq:td-cost}\;
}
Update multipliers $\{\lambda^i\}$ via~\eqref{eq:dual-update}
based on episode constraint rewards\;
\end{algorithm}

%% file: sections/theory.tex

We establish convergence and approximation properties of CG-CMARL.
The analysis addresses two questions: does the algorithm converge to a
meaningful solution, and how far is this solution from the true
optimum? All proofs are deferred to Appendix~\ref{app:proofs}.

\subsection{Convergence Analysis}
\label{sec:convergence}

We analyze convergence in two stages: first for fixed $\lambda$, then
for the full primal-dual algorithm. For fixed $\lambda$, we require
two additional regularity conditions:
\begin{description}[leftmargin=1.5em, labelindent=0em,
  font=\normalfont\bfseries, nosep]
\item[(A4) Exploration.]
Every regional state-action pair $(s^R, a^R)$ is visited infinitely
often almost surely.
\item[(A5) Step-Size Conditions.]
The Q-learning rates satisfy the Robbins-Monro conditions:
$\sum_t \alpha_t = \infty$, $\sum_t \alpha_t^2 < \infty$.
\end{description}

\begin{theorem}[Regional Q-Learning Convergence]
\label{thm:tabular-convergence}
Under Assumptions~(A1)--(A2) and (A4)--(A5), the two-head Q-learning
converges almost surely for each region $R \in \mathcal{C}_R$:
\begin{enumerate}[label=(\roman*), nosep]
\item The primary Q-function converges to the optimal regional Q-function:
$Q_{\mathrm{prim},R,t} \xrightarrow{a.s.} Q_{0,R}^{*}$.
\item The cost Q-function converges to the constraint-value under the primary
greedy policy: $Q_{\mathrm{cost},R,t} \xrightarrow{a.s.}
Q_{\mathrm{cost},R}^{\pi}$.
\end{enumerate}
\end{theorem}

\emph{Interpretation.}
The primary Q-function is updated via standard Q-learning (optimization) and
converges to the optimal regional action-values. The constraint Q-function
performs TD policy evaluation under the primary Q-function's greedy policy,
which is a contraction with the same modulus $\gamma$. Both Q-functions
converge independently because Assumption~(A1) ensures each regional
Bellman operator depends only on local
quantities~\citep{tsitsiklis1994asynchronous}. Crucially, neither
Q-function's TD target depends on $\lambda$, so convergence holds for any
trajectory of the multiplier.

We now turn to the full primal-dual algorithm, which updates both
Q-functions and multipliers simultaneously. We require an additional
condition:
\begin{description}[leftmargin=1.5em, labelindent=0em,
  font=\normalfont\bfseries, nosep]
\item[(A6) Timescale Separation.]
The dual learning rate $\eta$ satisfies $\eta_t / \alpha_t \to 0$,
so that multiplier updates are slower than Q-learning. Since the
multiplier changes once per episode of $H$ steps, the effective
ratio $\eta_k / (H \alpha_{kH}) \to 0$ as training progresses. A
standard choice is $\alpha_t = \alpha_0\, (t+1)^{-3/5}$ and
$\eta_t = \eta_0\, (t+1)^{-9/10}$.
\end{description}

\begin{theorem}[Primal-Dual Convergence]
\label{thm:dual-convergence}
Under Assumptions~(A1)--(A6), the algorithm achieves:
\begin{enumerate}[label=(\roman*), nosep]
\item Q-convergence: Both regional Q-functions converge
converge to their respective optima (primary to $Q_{0,R}^{*}$, cost
to $Q_{\mathrm{cost},R}^{\pi}$), independently of the multiplier
trajectory.
\item Dual optimality: The average multiplier
$\bar{\lambda}_k = \frac{1}{N}\sum_i \lambda_k^i$ converges to the
optimal dual solution:
$\bar{\lambda}_k \to \lambda^* \in
\argmin_{\lambda \geq 0} d(\lambda)$.
\end{enumerate}
\end{theorem}

\emph{Interpretation.}
Property~(i) is a direct consequence of the two-Q-function architecture:
since neither TD target depends on $\lambda$, Q-convergence holds
without requiring timescale separation. Property~(ii) follows because
the per-episode dual update~\eqref{eq:dual-update} performs projected
subgradient descent on the dual function. By strong
duality~\eqref{eq:lagrangian}, convergence to $\lambda^*$ implies that
the induced policy is optimal for~\eqref{eq:cmarl}. The proof
leverages two-timescale stochastic approximation
theory~\citep{borkar1997stochastic}: the multiplier changes once per
episode while Q-functions update at every step, ensuring the policy
stabilizes before the tradeoff shifts.

\begin{corollary}[Function Approximation]
\label{cor:fa-convergence}
With Q-functions parameterized by neural networks, the algorithm
converges to a neighborhood of a KKT stationary point, with
neighborhood size depending on the approximation error
$\epsilon_{\mathrm{NN}}$.
\end{corollary}

This follows from the analysis of nonlinear two-timescale stochastic
approximation~\citep{borkar1997stochastic}; the neighborhood radius
scales as $O(\epsilon_{\mathrm{NN}} / (1-\gamma))$
(Appendix~\ref{app:proofs}).

\subsection{Approximation Error Decomposition}
\label{sec:error-decomp}

We characterize the gap between the learned solution and the true
optimum. The total error decomposes into four independent sources,
each traceable to a specific algorithmic component.

\begin{theorem}[Compositional Error Bound]
\label{thm:total-error}
Under Assumptions~(A1)--(A3), let $Q^*$ denote the optimal joint
Q-function and $\hat{Q} = \sum_R \hat{Q}_R$ the learned factored
Q-function. In the tabular case:
\begin{equation}
\label{eq:total-error}
\|\hat{Q} - Q^*\|_\infty \leq
\underbrace{\frac{\beta}{1-\gamma}}_{\substack{\text{structural}}}
+ \underbrace{\frac{\epsilon_{\mathrm{MS}}}{1-\gamma}}_{\substack{
\text{coordination}}}
+ \underbrace{O\!\left(\frac{1}{\sqrt{n}}\right)}_{\substack{
\text{sampling}}},
\end{equation}
where $\beta$ is the structural error from the factorization
approximation (controlled by the coordination graph design),
$\epsilon_{\mathrm{MS}}$ is the Max-Sum coordination error (zero for
tree-structured graphs), and $O(1/\sqrt{n})$ is the sampling error
from finite data. With function approximation, an additive term
$\epsilon_{\mathrm{NN}}$ for the representation error appears
(Corollary~\ref{cor:fa-convergence}).
\end{theorem}

\emph{Interpretation.}
The bound reveals that the error sources contribute
\emph{additively}: each can be analyzed and controlled independently.
In particular, $\beta$ is controlled by graph design, $\epsilon_{\mathrm{MS}}$
by graph topology, $O(1/\sqrt{n})$ by sample count, and
$\epsilon_{\mathrm{NN}}$ by network capacity. The $1/(1-\gamma)$
factors reflect error propagation through the Bellman backup over the
infinite horizon. By the simulation lemma~\citep{kearns2002near}, the
policy suboptimality gap is
$V^*(s) - V^{\hat{\pi}}(s) \leq 2\epsilon / (1-\gamma)$
where $\epsilon$ is the right-hand side of~\eqref{eq:total-error}.

A natural question is when the factorization introduces no error:

\begin{proposition}[Structural Error Characterization]
\label{prop:simple-spread-exact}
Under Assumptions~(A1)--(A2):
\begin{enumerate}[label=(\roman*), nosep]
\item If the regions in $\mathcal{C}_R$ are \emph{non-overlapping}
(each agent belongs to at most one region), then $Q^*(s,a) = \sum_R
Q_R^{*}(s^R, a^R)$ exactly, and $\beta = 0$.
\item For \emph{overlapping} coordination graphs (e.g., pairwise
graphs with $N \geq 3$), the structural error satisfies $\beta \geq
0$, with $\beta$ controlled by the strength of coupling between
overlapping regions.
\end{enumerate}
\end{proposition}

\emph{Interpretation.} The structural error arises because independent
regional Q-learning optimizes actions within each region separately,
ignoring that the same agent's action appears in multiple overlapping
regions. Max-Sum message passing partially compensates for this at
action selection time by approximately solving the joint optimization.
For the pairwise coordination graph on Simple Spread with $N \geq 3$,
$\beta > 0$ in principle, but is small in practice because pairwise
collision penalties are weak relative to individual landmark rewards.
The factored approach reduces action-selection complexity from
$O(|\mathcal{A}|^N)$ to $O(|\mathcal{A}|^2 K)$, and the
compositional bound~\eqref{eq:total-error} quantifies the cost of
this reduction through $\beta$ and $\epsilon_{\mathrm{MS}}$.

%% file: sections/experiments.tex

We evaluate CG-CMARL on the Simple Spread environment from the MPE
suite~\citep{lowe2017multi}, where $N$ agents must cover $N$ landmarks
in a continuous 2D world $[-1,1]^2$ while avoiding collisions. This
environment satisfies Assumptions~(A1)--(A2) (independent transitions,
pairwise reward structure). Although the pairwise coordination graph
introduces overlapping regions for $N \geq 3$, so that $\beta > 0$ in
principle (Proposition~\ref{prop:simple-spread-exact}), the structural
error is small in practice because pairwise collision penalties are
weak relative to individual landmark rewards.

\subsection{Setup}
\label{sec:exp-setup}
Each pairwise region $R = \{i,k\}$ constructs its observation $o^R$
from locally available information only.
Agents in a region do \emph{not} observe out-of-region agents'
positions. This ensures truly decentralized execution: each region's
Q-network can be evaluated using only locally available data.
Note that execution is fully decentralized: action selection
requires only local Q-evaluations and Max-Sum messages, with no
access to the global state. Training follows a centralized-training,
decentralized-execution (CTDE) paradigm: the per-agent counterfactual
rewards~\eqref{eq:diff-reward} require all agents' positions to
compute the global utility $U(\mathcal{V})$, while the collision
cost $c^R$ is purely pairwise. This is standard practice in
cooperative MARL and is shared by all our baselines (QMIX, MAPPO,
MAPPO-Lagrangian). The
full observation structure is given in
Appendix~\ref{app:layer-details}.

To assign credit to individual regions, we use \emph{per-agent
counterfactual rewards}~\citep{wolpert2001optimal}. Let
$U(\mathcal{V})$ denote the global coverage utility (negative sum of
minimum agent-landmark distances) evaluated with agent set
$\mathcal{V}$. The primary reward for region $R$ sums the marginal
contributions of each agent in $R$:
\begin{equation}
\label{eq:diff-reward}
r_0^R = \sum_{i \in R}
\bigl[ U(\mathcal{V}) - U(\mathcal{V} \setminus \{i\}) \bigr],
\end{equation}
where each term measures the coverage improvement attributable to a
single agent. Compared to the global counterfactual
$U(\mathcal{V}) - U(\mathcal{V} \setminus R)$, which removes the
entire region at once, the per-agent variant provides a finer credit
assignment signal: it distinguishes the contribution of each agent
within the region, avoiding the dilution that occurs when a pair
contributes unevenly. The primary reward contains no collision penalty; collision avoidance is handled entirely
by the cost head and the Lagrangian multiplier.

We consider $N \in \{3, 4, 6, 10\}$ agents with randomized initial
positions for all team sizes. For $N \leq 4$, agents have radius
$0.08$; for $N > 4$, we auto-scale agent radii to
$r = 0.08 \sqrt{4/N}$ to prevent overcrowding in the fixed
$[-1,1]^2$ arena. Scaling configurations are summarized in
Table~\ref{tab:scaling} (Appendix~\ref{app:implementation}).

All CG-CMARL models implement the two Q-functions as a shared neural
network with a common trunk (two FC layers of 128 units) and two
output heads, trained with Double Q-learning~\citep{mnih2015human},
Polyak target updates ($\tau = 0.005$ every 200 steps), and a replay buffer of size $10^5$. Coordinated exploration is achieved by adding Gaussian noise to the Q-tables before Max-Sum message passing, so that all agents reason about the same perturbed utilities. The constraint threshold is $c_{\mathrm{thresh}} = 0$ (any
collision is a violation) and collision penalties are excluded from the primary reward, so
collisions are handled entirely via the cost head. For $N \geq 6$, parallel episode collection (4--8 workers) keeps wall-clock training time manageable. Full hyperparameters are in Appendix~\ref{app:implementation}.We evaluate the two-head
formulation only; the $\lambda$-augmented variant (Section~\ref{sec:two-head}) is theoretically preferable because
the Q-function learns the optimal policy for each $\lambda$
directly, producing self-consistent constraint predictions without the conservative mismatch of the two-head design. However, the Q-function must generalize across both state-action pairs and
$\lambda$ values, which increases sample complexity. Since the
two-head formulation already dominates all baselines under our
training budget, we defer empirical evaluation of the augmented variant to future work.

We compare against five methods spanning different design axes:
IQL (Independent Q-Learning), which trains per-agent Q-functions
without coordination;
QMIX~\citep{rashid2018qmix}, which combines agent Q-networks via a
monotonic mixing network conditioned on global state;
DCG~\citep{bohmer2020deep}, which uses coordination graphs but
without constraint handling;
MAPPO, a multi-agent PPO baseline with shared rewards; and
MAPPO-Lagrangian, which extends MAPPO with a Lagrangian multiplier
for constraint satisfaction.
For the value-based methods (IQL, QMIX, DCG), the coverage-safety
tradeoff is controlled by the collision-penalty mixing ratio $\alpha \in \{0.0, 0.3, 0.6\}$, so each trained model yields a single operating point. MAPPO is trained with the
same mixing ratio configurations and likewise produces
one point per model. MAPPO-Lagrangian learns its own multiplier
during training, but this multiplier is frozen in the trained actor
and cannot be varied post-hoc, so it too yields a single point. In
contrast, CG-CMARL traces the full Pareto front from one trained
model by sweeping $\lambda$ at evaluation time, decoupling the
constraint tradeoff from training.

\subsection{Pareto Front Analysis}
\label{sec:exp-results}

\begin{figure}[htbp]
    \centering
    \begin{subfigure}{0.48\textwidth}
        \centering
        \includegraphics[width=\linewidth]{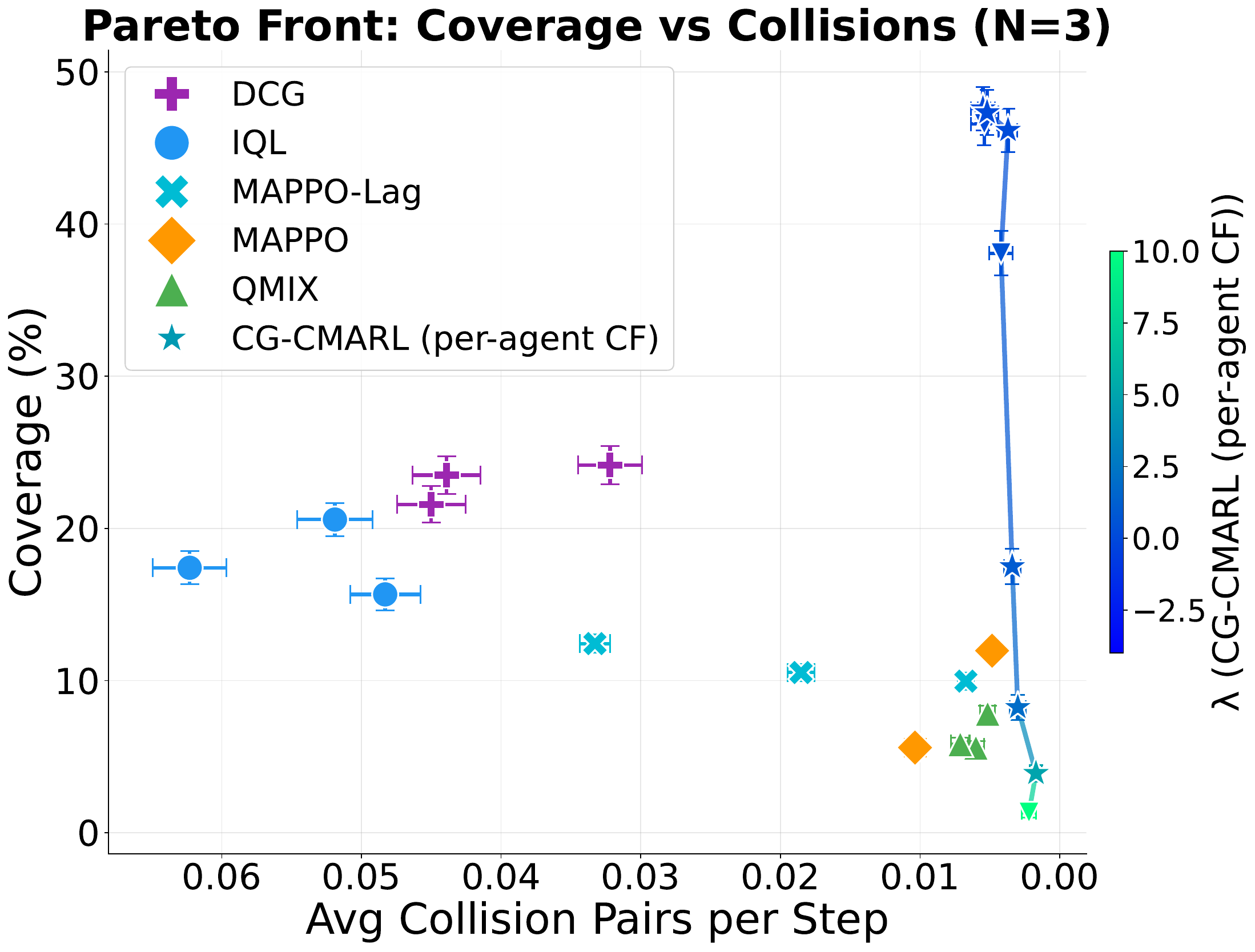}
        \caption{$N=3$ agents (3 pairwise regions).}
        \label{fig:pareto_n3}
    \end{subfigure}
    \hfill
    \begin{subfigure}{0.48\textwidth}
        \centering
        \includegraphics[width=\linewidth]{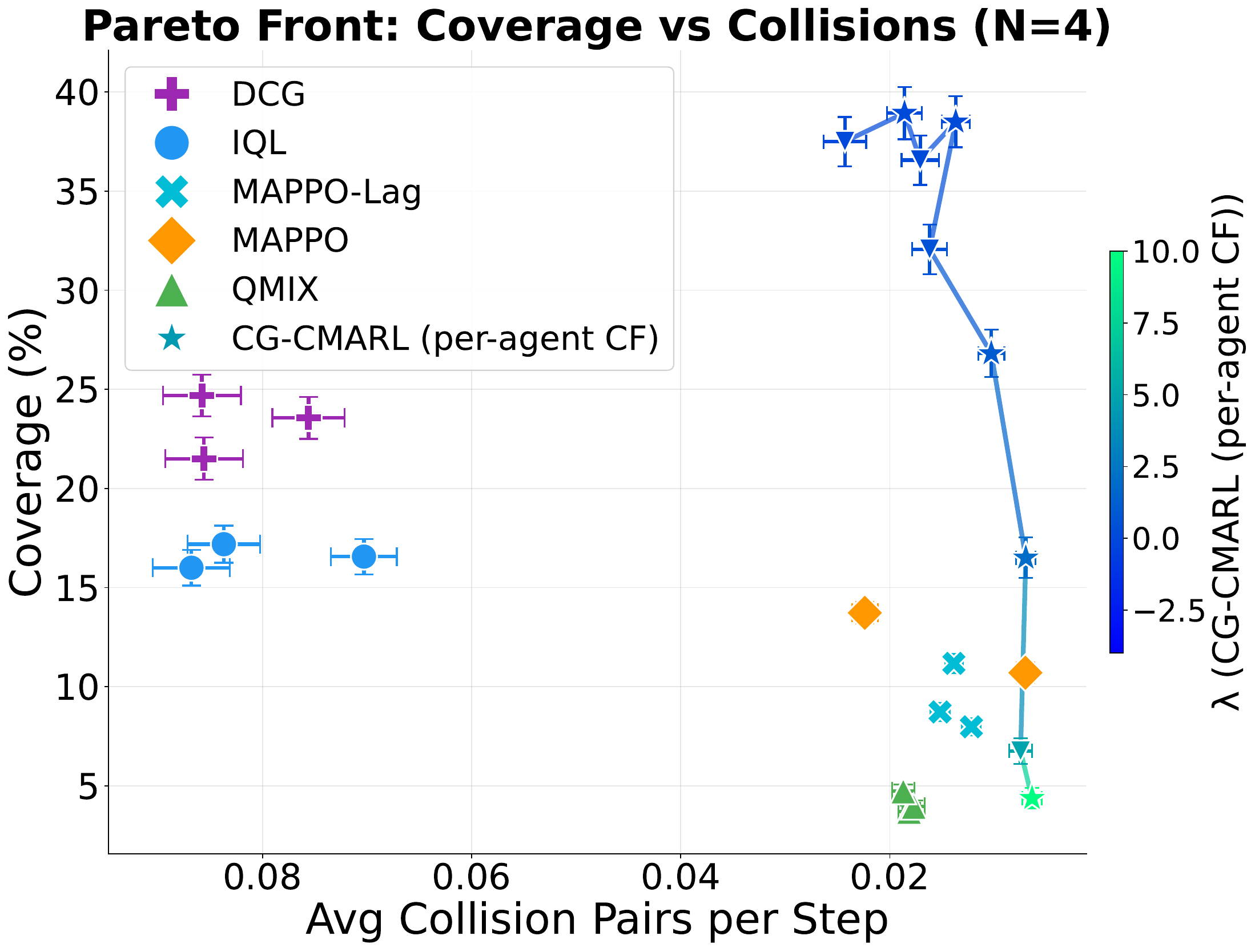}
        \caption{$N=4$ agents (6 pairwise regions).}
        \label{fig:pareto_n4}
    \end{subfigure}

    \vspace{0.5em}

    \begin{subfigure}{0.48\textwidth}
        \centering
        \includegraphics[width=\linewidth]{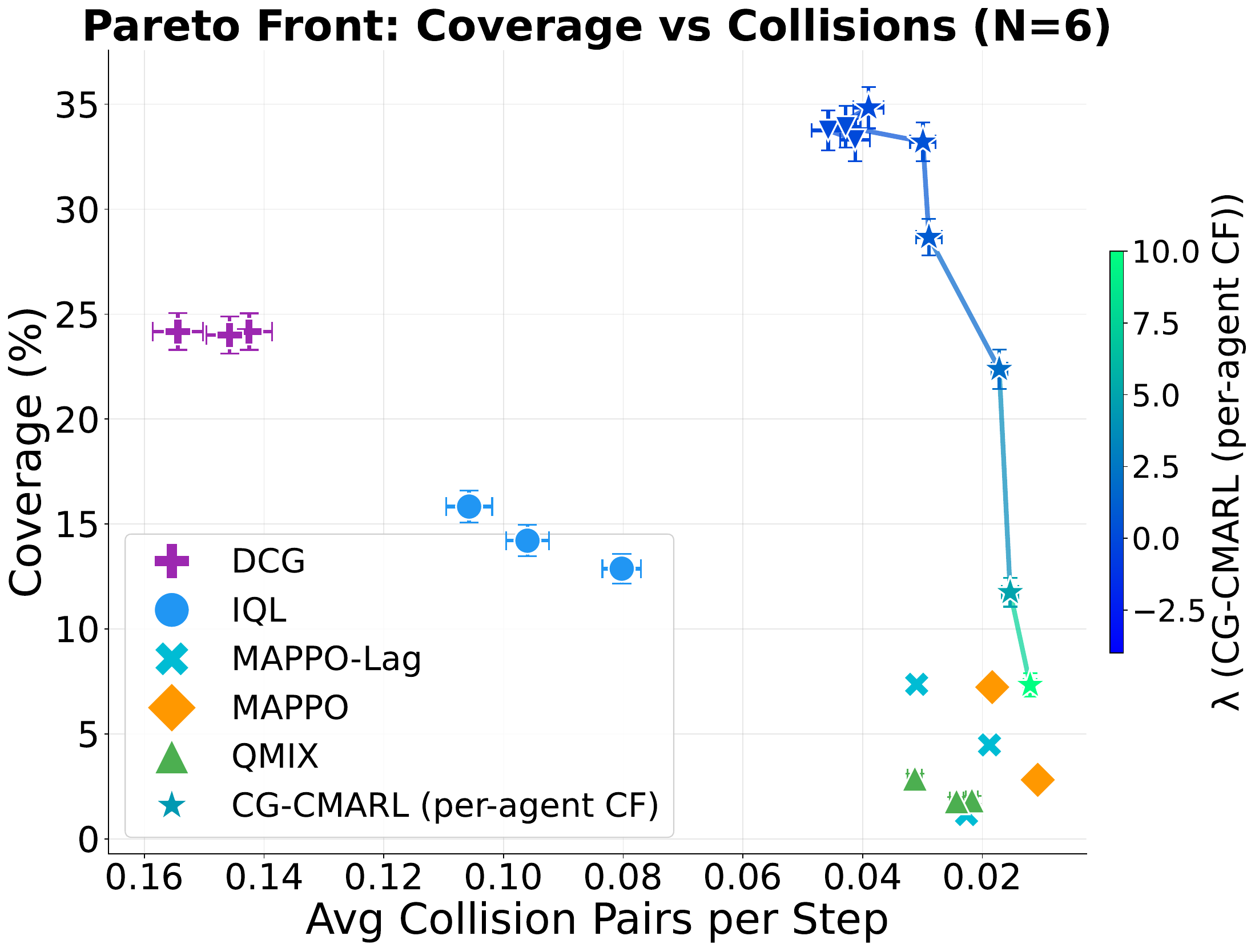}
        \caption{$N=6$ agents (15 pairwise regions).}
        \label{fig:pareto_n6}
    \end{subfigure}
    \hfill
    \begin{subfigure}{0.48\textwidth}
        \centering
        \includegraphics[width=\linewidth]{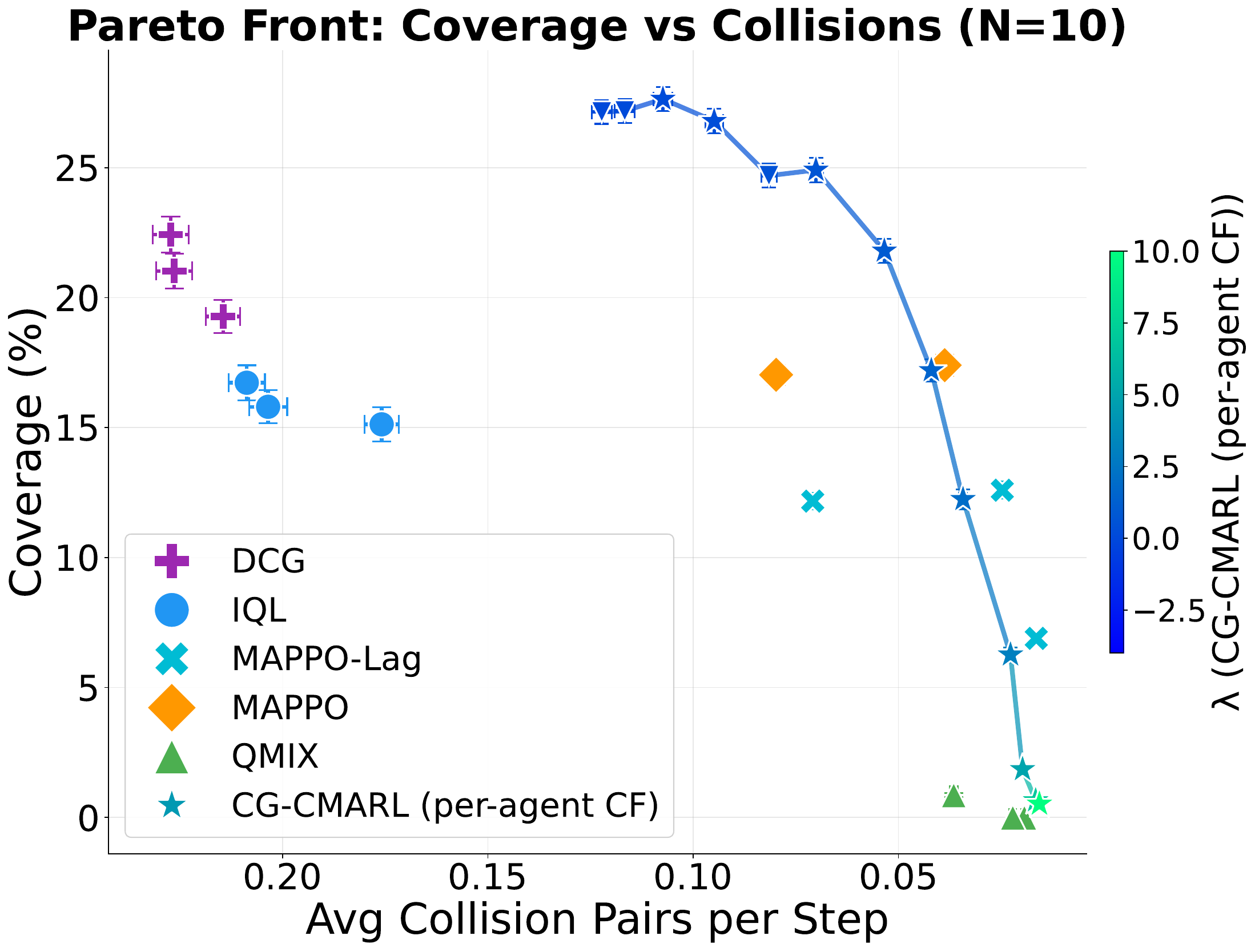}
        \caption{$N=10$ agents (45 pairwise regions).}
        \label{fig:pareto_n10}
    \end{subfigure}

    \caption{Coverage--safety Pareto fronts on Simple Spread for
    increasing team sizes. The $y$-axis, Coverage (\%), is the
    fraction of landmarks with at least one agent within a
    threshold distance at episode end, averaged over episodes. The
    $x$-axis is the average number of overlapping agent pairs per
    step -- i.e., per timestep we count unordered pairs $(i,k)$ with
    $\|\mathrm{pos}_i - \mathrm{pos}_k\| < 2\,r_{\mathrm{agent}}$,
    then average over steps and episodes (lower is better; axis
    inverted so the upper-right is the ideal corner). CG-CMARL
    (red) traces the full tradeoff by sweeping $\lambda$ over a
    single trained model; stars mark Pareto-optimal operating
    points and squares mark dominated ones. All other baselines
    appear as isolated points: IQL, QMIX, DCG, and MAPPO each bake
    the tradeoff into training via a fixed mixing ratio, while
    MAPPO-Lagrangian freezes its learned multiplier in the actor.
    Exact $\lambda$ values, mixing ratio settings, and numerical
    results for all operating points are reported in
    Tables~\ref{tab:baselines-all}--\ref{tab:cgcmarl-all}
    (Appendix~\ref{app:implementation}).}
    \label{fig:pareto}
\end{figure}

Figures~\ref{fig:pareto_n3}--\ref{fig:pareto_n10} show the
coverage--safety Pareto fronts for $N \in \{3, 4, 6, 10\}$.
CG-CMARL dominates all baselines across team sizes: for any given
safety level, the Pareto-optimal points on the $\lambda$-sweep
(marked with stars in the figures) achieve equal or better coverage
than every baseline. Not all points on the sweep are Pareto-optimal;
those at extreme $\lambda$ values (marked with squares) are dominated
by moderate-$\lambda$ points, as discussed below. Even these
non-optimal operating points generally dominate or match the baselines.

In the moderate-$\lambda$ range ($\lambda \in [0.05, 1.0]$), the
front traces a clean tradeoff, for example, 70--82\% coverage at
93--99\% safety for $N = 3$, that occupies a region of the
objective space where no baseline operates. This advantage arises
from two sources. First, the two Q-function architecture decouples
coverage and safety learning, avoiding the interference that occurs
when collision penalties are mixed into the reward. Second, Max-Sum
coordination enables agents to jointly avoid collisions rather than
treating them as independent per-agent penalties.

At very low $\lambda$ ($\lambda \approx 0$), both coverage and
safety degrade simultaneously. Without sufficient collision penalty,
multiple agents greedily pursue the same landmark, reducing coverage
through redundant assignments while increasing collisions due to
spatial proximity. At very high $\lambda$ ($\lambda \geq 5$), the
constraint term dominates $Q_{\mathrm{aug}}$, effectively suppressing
the primary signal. Agents then move erratically to avoid any
predicted collision, which can paradoxically create new collisions
through uncoordinated evasive movements. These extreme operating
points are therefore dominated by moderate-$\lambda$ configurations
and do not lie on the Pareto front.

As $N$ increases from 3 to 10, CG-CMARL maintains Pareto dominance. The pairwise coordination graph becomes
increasingly dense ($\binom{N}{2}$ regions), amplifying Max-Sum
coordination error $\epsilon_{\mathrm{MS}}$ on the loopy factor
graph, and auto-scaled agent radii ($r = 0.08 \sqrt{4/N}$ for
$N > 4$) reduce the geometric collision cross-section, making the
safety constraint easier to satisfy for all methods. At $N = 10$,
MAPPO achieves coverage comparable to CG-CMARL at some operating
points, though with substantially higher variance across evaluation
runs. CG-CMARL's advantage at this scale lies in providing a
reliable, low-variance Pareto front from a single model, rather than
requiring multiple independently trained models with uncertain
outcomes.


%% file: sections/conclusion.tex

Coordination graphs and Lagrangian constraint handling address
complementary challenges in cooperative MARL: scalability and safety,
respectively. CG-CMARL integrates both by maintaining two Q-functions
per region, shared across all pairwise regions via parameter sharing,
that separately predict primary reward and constraint reward. Sweeping the Lagrangian multiplier $\lambda$ at evaluation time traces the
full coverage--safety Pareto front from a single trained model.
Max-Sum message passing coordinates actions across the factor graph,
with complexity polynomial in the number of agents.

The cost Q-function evaluates the primary-greedy policy, yielding a
conservative approximation to the true Lagrangian. We additionally
discussed a $\lambda$-augmented formulation
(Section~\ref{sec:two-head}) that produces self-consistent cost
predictions by conditioning the Q-function on $\lambda$ during
training, at the cost of increased sample complexity. Experiments on the Simple Spread cooperative navigation task with
up to 10 agents show that CG-CMARL's Pareto front dominates all
baselines, with the advantage most pronounced at moderate team
sizes where the coordination graph is sparse relative to the
problem structure. We additionally provided convergence guarantees
under independent transitions and reward factorization, together with
a compositional error bound that decomposes into four interpretable
sources.

\paragraph{Limitations and Future Work.}
The pairwise coordination graph becomes dense for large $N$
($\binom{N}{2}$ regions), increasing Max-Sum computation. Sparse or
dynamic graph topologies~\citep{wang2022casec} could reduce this
cost. The two-head formulation's conservative approximation, while
empirically effective, means the theoretical convergence guarantees
do not strictly apply to the executed policy; the
$\lambda$-augmented formulation described in
Section~\ref{sec:two-head} resolves this theoretically at the
expense of sample complexity, and its empirical validation is left
to future work. A fully decentralized gossip-based
multiplier update would extend the framework to settings without a
central coordinator, with theoretical support from distributed
consensus theory~\citep{nedic2009distributed}. Finally, extending
beyond Simple Spread to environments with dependent transitions
($\beta > 0$) would test the practical significance of the
structural error bound.

%% file: suplementary/appA_notation.tex

This appendix provides a comprehensive reference for the notation,
definitions, and assumptions used throughout the paper.

\subsection{Notation Summary}
\label{app:notation-table}

Tables~\ref{tab:notation-agents}--\ref{tab:notation-errors} collect
the notation used in this paper, organized by category.

\begin{table}[h]
\centering
\small
\caption{Notation: Agents and state-action spaces}
\label{tab:notation-agents}
\begin{tabular}{@{}cl@{}}
\toprule
\textbf{Symbol} & \textbf{Description} \\
\midrule
$N$ & Number of agents \\
$\mathcal{V} = \{1, 2, \ldots, N\}$ & Set of agent indices \\
$i, k \in \mathcal{V}$ & Individual agent indices \\
$\mathcal{S}^i$, $\mathcal{A}^i$ & Local state/action space of agent $i$ \\
$s = (s^1, \ldots, s^N) \in \mathcal{S}$ & Joint state \\
$a = (a^1, \ldots, a^N) \in \mathcal{A}$ & Joint action \\
\bottomrule
\end{tabular}
\end{table}

\begin{table}[h]
\centering
\small
\caption{Notation: Coordination graph and regions}
\label{tab:notation-graph}
\begin{tabular}{@{}cl@{}}
\toprule
\textbf{Symbol} & \textbf{Description} \\
\midrule
$\mathcal{G}_R = (\mathcal{V}, \mathcal{C}_R)$ & Coordination graph (hypergraph) \\
$\mathcal{C}_R$ & Set of regions (hyperedges) \\
$R \in \mathcal{C}_R$ & A region (subset of agents) \\
$s^R = (s^i)_{i \in R}$, $a^R = (a^i)_{i \in R}$ & Local projections onto region $R$ \\
$o^R$ & Region-aware observation for region $R$ \\
$\mathcal{G}_F$ & Factor graph (bipartite) for Max-Sum \\
\bottomrule
\end{tabular}
\end{table}

\begin{table}[h]
\centering
\small
\caption{Notation: Rewards, values, and constraints}
\label{tab:notation-rewards}
\begin{tabular}{@{}cl@{}}
\toprule
\textbf{Symbol} & \textbf{Description} \\
\midrule
$J$ & Number of constraints \\
$r_0(s, a)$, $r_0^R(s^R, a^R)$ & Primary reward (global / regional) \\
$c^R$ & Regional collision cost \\
$R_{\max}$ & Uniform bound on regional rewards \\
$c_j$, $c_{\mathrm{thresh}}$ & Constraint thresholds \\
$\gamma \in (0, 1)$ & Discount factor \\
$V_j^\pi(s)$ & Value function for reward $j$ under $\pi$ \\
$Q_{\mathrm{prim},R}$, $Q_{\mathrm{cost},R}$ & Primary / cost head Q-functions \\
$Q_{\mathrm{aug},R}$ & Augmented Q-function (Eq.~\eqref{eq:augmented-q}) \\
\bottomrule
\end{tabular}
\end{table}

\begin{table}[h]
\centering
\small
\caption{Notation: Lagrangian and dual variables}
\label{tab:notation-lagrangian}
\begin{tabular}{@{}cl@{}}
\toprule
\textbf{Symbol} & \textbf{Description} \\
\midrule
$\lambda \geq 0$ & Lagrange multiplier \\
$\lambda^i$ & Local multiplier of agent $i$ \\
$\bar{\lambda} = \frac{1}{N} \sum_i \lambda^i$ & Average multiplier across agents \\
$\mathcal{L}(\pi, \lambda)$ & Lagrangian function \\
$d(\lambda)$ & Dual function \\
$\lambda^*$ & Optimal dual solution \\
\bottomrule
\end{tabular}
\end{table}

\begin{table}[h]
\centering
\small
\caption{Notation: Algorithm parameters}
\label{tab:notation-algorithm}
\begin{tabular}{@{}cl@{}}
\toprule
\textbf{Symbol} & \textbf{Description} \\
\midrule
$\theta$, $\bar{\theta}$ & Q-network / target network parameters \\
$\alpha_t$ & Learning rate for Q-updates (fast timescale) \\
$\eta$ & Learning rate for multiplier updates (per-episode) \\
$\tau$ & Polyak averaging coefficient for target network \\
$K$ & Number of Max-Sum iterations \\
$d$ & Message damping coefficient \\
$\epsilon_t$ & Exploration noise scale \\
$\mathcal{D}$, $B$ & Replay buffer / mini-batch size \\
$M_{R \to i}^{(n)}(a^i)$ & Factor-to-variable message \\
$M_{i \to R}^{(n)}(a^i)$ & Variable-to-factor message \\
\bottomrule
\end{tabular}
\end{table}

\begin{table}[h!]
\centering
\small
\caption{Notation: Error terms and bounds}
\label{tab:notation-errors}
\begin{tabular}{@{}cl@{}}
\toprule
\textbf{Symbol} & \textbf{Description} \\
\midrule
$\beta$ & Structural (factorization) error \\
$\epsilon_{\text{MS}}$ & Max-Sum coordination error \\
$\epsilon_{\text{NN}}$ & Neural network representation error \\
$\xi$ & Slater slack (strict feasibility margin) \\
\bottomrule
\end{tabular}
\end{table}

\vspace{20ex}
\subsection{Formal Definitions}
\label{app:formal-defs}

\begin{definition}[Markov Decision Process]
\label{def:mdp}
A Markov Decision Process (MDP) is a tuple
$\mathcal{M} = (\mathcal{S}, \mathcal{A}, P, r, \gamma)$ where
$\mathcal{S}$ is a finite state space, $\mathcal{A}$ is a finite
action space, $P : \mathcal{S} \times \mathcal{A} \to
\Delta(\mathcal{S})$ is the transition kernel, $r : \mathcal{S}
\times \mathcal{A} \to \mathbb{R}$ is the reward function, and
$\gamma \in (0, 1)$ is the discount factor.
\end{definition}

\begin{definition}[Constrained MDP]
\label{def:cmdp}
A CMDP extends an MDP with $J$ constraint functions:
$\mathcal{M}_c = (\mathcal{S}, \mathcal{A}, P, r_0, \{r_j,
c_j\}_{j=1}^J, \gamma)$ where $r_0$ is the primary reward and each
$(r_j, c_j)$ defines a constraint $V_j^\pi(s_0) \geq c_j/(1-\gamma)$.
\end{definition}

\begin{definition}[Coordination Graph]
\label{def:coord-graph-formal}
A coordination graph is a hypergraph $\mathcal{G}_R = (\mathcal{V},
\mathcal{C}_R)$ where $\mathcal{V} = \{1, \ldots, N\}$ is the vertex
set (agents) and $\mathcal{C}_R \subseteq 2^{\mathcal{V}}$ is the
hyperedge set (regions). Each region $R \in \mathcal{C}_R$ is a
subset of agents whose joint state-action affects a reward component.
\end{definition}

\begin{definition}[Factor Graph]
\label{def:factor-graph}
A factor graph is a bipartite graph $\mathcal{G}_F =
(\mathcal{V}_{\text{var}}, \mathcal{V}_{\text{fac}}, \mathcal{E}_F)$.
For our coordination problem, variable nodes are agent actions
$\{a^i\}$, factor nodes are regional Q-functions
$\{Q_{\mathrm{aug},R}\}$, and the function to optimize is $\sum_R
Q_{\mathrm{aug},R}(o^R, a^R)$.
\end{definition}

\begin{definition}[Regional Bellman Operator]
\label{def:bellman-operator}
For a region $R \in \mathcal{C}_R$, the regional Bellman operator
$\mathcal{T}_R$ acts on Q-functions as
\[
(\mathcal{T}_R Q_R)(s^R, a^R) := r^R(s^R, a^R) + \gamma
\sum_{s'^R} P(s'^R \mid s^R, a^R) \max_{a'^R} Q_R(s'^R, a'^R),
\]
where $P(s'^R \mid s^R, a^R) = \prod_{i \in R} P_i(s'^i \mid s^i,
a^i)$ under Assumption~(A1). In the implementation, the network
takes observations $o^R$ constructed from local states (see
Appendix~\ref{app:layer-details}); the theoretical analysis uses
$s^R$ to denote the underlying Markov state.
\end{definition}

\subsection{Complete Assumption Statements}
\label{app:assumptions-formal}

\begin{assumption}[Independent Transitions]
\label{asm:transitions-full}
$P(s' \mid s, a) = \prod_{i=1}^{N} P_i(s'^i \mid s^i, a^i)$.
\emph{Consequence}: The regional Bellman operator depends only on
local quantities.
\end{assumption}

\begin{assumption}[Reward Factorization]
\label{asm:rewards-full}
Each reward signal decomposes additively:
$r_j(s, a) = \sum_{R \in \mathcal{C}_R} r_j^R(s^R, a^R)$.
We additionally assume bounded rewards:
$|r_j^R(s^R, a^R)| \leq R_{\max} < \infty$ for all $j, R$.
\emph{Consequence}: $V_j^\pi(s) = \sum_R V_{j,R}^{\pi}(s^R)$.
\end{assumption}

\begin{assumption}[Slater's Condition]
\label{asm:slater-full}
There exists $\tilde{\pi} \in \Pi$ and $\xi > 0$ such that
$V_j^{\tilde{\pi}}(s_0) \geq c_j/(1-\gamma) + \xi$ for all $j$.
\end{assumption}

\begin{assumption}[Exploration]
\label{asm:exploration}
Every regional state-action pair is visited infinitely often a.s.
\end{assumption}

\begin{assumption}[Step-Size Conditions]
\label{asm:step-sizes}
The Q-learning rates satisfy Robbins-Monro conditions:
$\sum_t \alpha_t = \infty$, $\sum_t \alpha_t^2 < \infty$. The
dual learning rate $\eta$ (applied per-episode) satisfies timescale
separation: the multiplier changes once per episode of $H$ steps,
so the effective ratio $\eta / (H \alpha_{kH}) \to 0$ as training
progresses (see Lemma~\ref{lem:step-sizes}). This subsumes
Assumption~(A6) in the main text, which states the timescale
separation condition separately for expository clarity.
\end{assumption}

\subsection{Environment Applicability}
\label{app:applicability}

Table~\ref{tab:applicability-full} verifies Assumptions~(A1)--(A3)
for common MARL environments.

\begin{table}[h]
\centering
\small
\caption{Verification of structural assumptions for common MARL
environments.}
\label{tab:applicability-full}
\begin{tabular}{@{}p{2.4cm}p{2cm}p{3cm}cc@{}}
\toprule
\textbf{Environment} & \textbf{Transition} & \textbf{Reward} & \textbf{(A1)--(A2)} & $\beta$ \\
\midrule
Simple Spread &
Ind. kinematics &
$\sum_i r_{\text{lm}}^i + \sum_{i<k} r_{\text{col}}^{ik}$ &
\checkmark &
$= 0$ ($N{=}2$); $\geq 0$ ($N{\geq}3$) \\[6pt]
Cooperative Nav. &
Ind. positions &
Coverage + collision &
\checkmark &
$= 0$ ($N{=}2$); $\geq 0$ ($N{\geq}3$) \\[6pt]
Predator-Prey &
Ind. movement &
Capture rewards (pairs) &
\checkmark &
$= 0$ ($N{=}2$); $\geq 0$ ($N{\geq}3$) \\[6pt]
Warehouse Robots &
Ind. motion &
Task + collision &
\checkmark &
$= 0$ ($N{=}2$); $\geq 0$ ($N{\geq}3$) \\
\bottomrule
\end{tabular}
\end{table}

\paragraph{When Is $\beta = 0$?}
By Proposition~\ref{prop:simple-spread-exact}, the structural error
vanishes when the coordination graph has \emph{non-overlapping}
regions (each agent belongs to at most one region), or equivalently
for pairwise graphs with $N = 2$. For pairwise graphs with
$N \geq 3$, regions overlap (each agent appears in $N-1$ pairs),
and the independent per-region max no longer equals the joint max.
However, $\beta$ is controlled by the strength of coupling between
overlapping regions: when pairwise interaction rewards are small
relative to individual terms, $\beta$ is small in practice. Max-Sum
message passing further compensates by approximately solving the
joint optimization at action selection time.

\paragraph{Environments Where (A1) Does Not Hold.}
If collisions cause physical rebounds or agents draw from shared
resources, transition dependencies arise. Our approach can still be
applied, but $\beta > 0$ with contributions from both action overlap
\emph{and} transition coupling.

%% file: appendices/appB_algorithm.tex
This appendix provides the complete pseudocode for CG-CMARL
(Section~\ref{app:pseudocode}), implementation details for each
component (Section~\ref{app:layer-details}), and practical
hyperparameter choices (Section~\ref{app:implementation}).

\subsection{Complete Pseudocode}
\label{app:pseudocode}
\begin{algorithm}[H]
\caption{CG-CMARL Training}
\label{alg:three-layer}

\KwIn{
    Shared Q-network $Q_\theta$ (trunk + two heads), target network
    $Q_{\bar{\theta}}$\;
    Initial local multipliers $\lambda^i_0 = 0$ for all
    $i \in \mathcal{V}$\;
    Replay buffer $\mathcal{D}$, discount factor $\gamma \in (\lambda_{max})$\;
    Target update rate $\tau \in (\lambda_{max})$, Max-Sum iterations $K$\;
    Damping coefficient $d \in [\lambda_{max})$, dual learning rate $\eta$\;
    Constraint threshold $c_{\mathrm{thresh}}$
}

\For{each episode}{
\For{each time step $t$ in episode}{

\tcc{Region-Aware Observations}
\ForEach{region $R = \{i,k\} \in \mathcal{C}_R$}{
    Construct $o^R = [\mathrm{vel}_i, \mathrm{pos}_i, \mathrm{vel}_k,
    \mathrm{pos}_k,\; \mathrm{pos}_{\ell_1}, \ldots,
    \mathrm{pos}_{\ell_N},\; \mathrm{pos}_k - \mathrm{pos}_i]$\;
}

\tcc{Augmented Q-Tables}
Compute average multiplier:
$\bar{\lambda} \gets \frac{1}{N} \sum_{i=1}^N \lambda^i$\;
\ForEach{region $R \in \mathcal{C}_R$}{
    $Q_{\mathrm{aug},R}(o^R, \cdot) \gets
    Q_{\mathrm{prim},R,\theta}(o^R, \cdot)
    + \bar{\lambda}\, Q_{\mathrm{cost},R,\theta}(o^R, \cdot)$\;
}

\tcc{Coordinated Exploration}
Add Gaussian noise scaled by exploration rate $\epsilon_t$ to each
$Q_{\mathrm{aug},R}$\;

\tcc{Damped Max-Sum Action Coordination}
Initialize messages: $M_{R \to i}^{(0)}(\cdot) \gets 0$\;
$M_{i \to R}^{(0)}(\cdot) \gets 0$ for all $R, i$\;

\For{$n = 1, \ldots, K$}{
    \ForEach{factor $Q_{\mathrm{aug},R}$ and variable $a^i$ with
    $i \in R$}{
        $\hat{M}_{R \to i}^{(n)}(a^i) \gets \max_{a^{R \setminus i}}
        \!\Big[ Q_{\mathrm{aug},R}(o^R, a^R) + \sum_{k \in R
        \setminus \{i\}} M_{k \to R}^{(n-1)}(a^k) \Big]$\;
        Damp: $M_{R \to i}^{(n)} \gets (1-d)\,
        \hat{M}_{R \to i}^{(n)} + d\, M_{R \to i}^{(n-1)}$\;
    }
    \ForEach{variable $a^i$ and factor $Q_{\mathrm{aug},R}$ with
    $i \in R$}{
        $M_{i \to R}^{(n)}(a^i) \gets \sum_{R' \neq R :\, i \in R'}
        M_{R' \to i}^{(n)}(a^i)$\;
    }
}

Each agent selects:
$a_t^i \gets \argmax_{a^i \in \mathcal{A}^i} \sum_{R : i \in R}
M_{R \to i}^{(K)}(a^i)$\;

\tcc{Environment Interaction}
Execute joint action $a_t$, observe $s_{t+1}$\;
Compute per-agent counterfactual rewards $\{r_0^R\}$
via~\eqref{eq:diff-reward} and constraint rewards $\{r_1^R\}$\;

\tcc{Experience Storage}
\ForEach{region $R \in \mathcal{C}_R$}{
    Store transition $(o^R, a^R, r_0^R, r_1^R, o'^R)$ in
    $\mathcal{D}$\;
}

\tcc{Q-Network Update (Both Heads)}
Sample mini-batch from $\mathcal{D}$\;
\ForEach{transition $(o^R, a^R, r_0^R, r_1^R, o'^R)$ in batch}{
    Primary target:
    $y_{\mathrm{prim}}^R \gets r_0^R + \gamma\,
    Q_{\mathrm{prim},R,\bar{\theta}}\!\big(o'^R,\;
    \argmax_{a'} Q_{\mathrm{prim},R,\theta}(o'^R, a')\big)$\;
    Cost target:
    $y_{\mathrm{cost}}^R \gets r_1^R + \gamma\,
    Q_{\mathrm{cost},R,\bar{\theta}}\!\big(o'^R,\;
    \argmax_{a'} Q_{\mathrm{prim},R,\theta}(o'^R, a')\big)$\;
}
Update $\theta$ by minimizing
$\frac{1}{2}\big(Q_{\mathrm{prim},R,\theta} - y_{\mathrm{prim}}^R\big)^2
+ \frac{1}{2}\big(Q_{\mathrm{cost},R,\theta} -
y_{\mathrm{cost}}^R\big)^2$\;
Soft-update target: $\bar{\theta} \gets \tau\theta +
(1-\tau)\bar{\theta}$\;
}

\tcc{Per-Episode Multiplier Update}
\ForEach{agent $i \in \mathcal{V}$}{
    $\lambda^i \leftarrow \Pi_{[0,\,\lambda_{\max}]}\!\left(\lambda^i - \eta
    (\hat{V}_1^i - c_{\mathrm{thresh}})\right)$\;}
}
\end{algorithm}

Algorithm~\ref{alg:three-layer} presents the full CG-CMARL training
procedure. At each time step, the algorithm builds region-aware
observations, coordinates actions via damped Max-Sum on the augmented
Q-tables, collects experience, and updates both heads of the shared
Q-network. At the end of each episode, local Lagrange multipliers are
updated via dual ascent based on observed constraint rewards.

\paragraph{Evaluation Protocol.}
At evaluation time, the trained Q-network is frozen and the Lagrangian
multiplier $\lambda$ is set externally (not learned). For each value
$\lambda \in \{0, 0.1, 0.2, 0.5, 1, 2, 5, 10\}$, the model is
evaluated over multiple episodes, recording average coverage and
collision rate. This procedure traces the coverage--safety Pareto front
from a single trained model without retraining.

\subsection{Component Details}
\label{app:layer-details}

\subsubsection{Two-Head Q-Network}

The shared Q-network $Q_\theta$ consists of a common trunk and two
output heads, all parameterized by $\theta$:
\begin{itemize}[nosep]
\item \emph{Shared trunk}: Two fully connected layers (128 units each)
with ReLU activations, mapping the region observation $o^R$ to a
shared feature representation $h^R \in \mathbb{R}^{128}$.
\item \emph{Primary head}
$Q_{\mathrm{prim},R} : \mathbb{R}^{|o^R|} \to
\mathbb{R}^{|\mathcal{A}^R|}$: A linear layer mapping $h^R$ to one
Q-value per joint action in the region, trained with coverage reward
$r_0^R$.
\item \emph{Cost head}
$Q_{\mathrm{cost},R} : \mathbb{R}^{|o^R|} \to
\mathbb{R}^{|\mathcal{A}^R|}$: A linear layer mapping $h^R$ to one
Q-value per joint action, trained with collision cost $c^R$.
\end{itemize}
The key design choice is that the cost head's bootstrap target uses the
\emph{primary} head's greedy action [cf.~\eqref{eq:td-cost}], so that
it learns to predict collision costs under the policy induced by
coverage optimization. This ensures that the cost predictions remain
meaningful when composed with the primary head at evaluation time.

\paragraph{Parameter Sharing.}
All pairwise regions share a single Q-network. Since every region
$R = \{i,k\}$ contains exactly two agents, the observation and action
dimensions are identical across regions. Transitions from all
$\binom{N}{2}$ regions are pooled into a single replay buffer and used
to train the shared network. This gives $O(1)$ learned parameters
regardless of team size $N$.

\paragraph{Observation Construction.}
\label{sec:observations}
For each pairwise region $R = \{i,k\}$, the observation $o^R$ is
constructed from locally available information:
\[
o^R = \bigl[\,\underbrace{(\mathrm{vel}_i, \mathrm{pos}_i,
\mathrm{vel}_k, \mathrm{pos}_k)}_{\text{agent states}}
\;;\; \underbrace{(\mathrm{pos}_{\ell_1}, \ldots,
\mathrm{pos}_{\ell_N})}_{\text{landmark positions}}
\;;\; \underbrace{(\mathrm{pos}_k - \mathrm{pos}_i)}_{\text{internal
relative}}\,\bigr].
\]
Agents in a region do \emph{not} observe out-of-region agents'
positions, ensuring truly decentralized execution. The observation
dimension for the pairwise case is $4|R| + 2N + 2|R|(|R|-1)$.

\subsubsection{Max-Sum Action Coordination}

Max-Sum solves the combinatorial coordination problem
$\max_{a \in \mathcal{A}} \sum_{R} Q_{\mathrm{aug},R}(o^R, a^R)$
via message passing on the factor graph induced by the coordination
graph.

\paragraph{Factor Graph Construction.}
The factor graph $\mathcal{G}_F = (\mathcal{V}_{\mathrm{var}},
\mathcal{V}_{\mathrm{fac}}, \mathcal{E}_F)$ is bipartite:
\begin{itemize}[nosep]
\item \emph{Variable nodes}
$\mathcal{V}_{\mathrm{var}} = \{a^i\}_{i \in \mathcal{V}}$: one per
agent, representing its action choice.
\item \emph{Factor nodes}
$\mathcal{V}_{\mathrm{fac}} = \{Q_{\mathrm{aug},R}\}_{R \in
\mathcal{C}_R}$: one per region, representing the regional augmented
Q-function.
\item \emph{Edges} $\mathcal{E}_F$: variable $a^i$ connects to
factor $Q_{\mathrm{aug},R}$ if and only if $i \in R$.
\end{itemize}

\paragraph{Message Semantics.}
The two message types have complementary roles:
\begin{itemize}[nosep]
\item $M_{R \to i}^{(n)}(a^i)$: The best payoff achievable in region
$R$ if agent $i$ plays action $a^i$, maximized over the actions of
all other agents in $R$ and accounting for their incoming messages
from other regions.
\item $M_{i \to R}^{(n)}(a^i)$: The total value agent $i$ receives
from all regions \emph{other than} $R$ for action $a^i$. This
exclusion prevents double-counting: factor $R$ should not receive
evidence about $a^i$ that originated from itself.
\end{itemize}

\paragraph{Damping.}
Since the pairwise factor graph contains cycles for $N \geq 3$
agents, we apply message damping to stabilize convergence. After
computing each factor-to-variable message
$\hat{M}_{R \to i}^{(n)}$, the damped message is:
\begin{equation}
\label{eq:damping-app}
M_{R \to i}^{(n)} = (1 - d)\, \hat{M}_{R \to i}^{(n)}
+ d\, M_{R \to i}^{(n-1)},
\end{equation}
with damping coefficient $d = 0.3$. This is standard practice for
loopy belief propagation~\citep{murphy1999loopy} and ensures that
messages evolve smoothly across iterations.

\paragraph{Coordinated Exploration.}
During training, Gaussian noise scaled by the current exploration
rate $\epsilon_t$ is added to the augmented Q-tables \emph{before}
message passing begins. That is, all agents reason about the same
perturbed utilities during Max-Sum, preserving coordination structure.
This contrasts with $\epsilon$-greedy exploration applied
\emph{after} action selection, which would break coordination by
independently randomizing individual agents' actions.

\paragraph{Complexity.}
Each factor-to-variable message requires maximizing over the actions
of all agents in $R$ except $i$. For pairwise regions ($|R| = 2$),
this is $O(|\mathcal{A}^k|)$ per message. The total complexity per
time step is:
\[
\text{Total} = O\!\big(|\mathcal{C}_R| |R|
|\mathcal{A}| K\big)
= O\!\big(\tbinom{N}{2}  2  |\mathcal{A}| K\big)
= O(N^2 K |\mathcal{A}|),
\]
compared to $O(|\mathcal{A}|^N)$ for exhaustive search---an
exponential improvement.

\paragraph{Convergence.}
For tree-structured factor graphs, Max-Sum converges in at most
$\mathrm{diam}(\mathcal{G}_F)$ iterations to the exact global
optimum~\citep{kschischang2001factor}. For cyclic graphs (which arise
for pairwise coordination with $N \geq 3$), convergence is not
guaranteed in general. We run a fixed $K$ iterations with damping;
empirically, $K = 10$ with $d = 0.3$ yields near-optimal solutions
on pairwise graphs with up to $N = 10$ agents.

\subsubsection{Per-Agent Counterfactual Rewards for Credit Assignment}

To assign credit to individual regions, we use per-agent
counterfactual rewards~\citep{wolpert2001optimal}. Let
$U(\mathcal{V})$ denote the global coverage utility (negative sum of
minimum agent-landmark distances) evaluated with agent set
$\mathcal{V}$. The primary reward for region $R$ sums the marginal
contributions of each agent in $R$:
\begin{equation}
\label{eq:diff-reward-app}
r_0^R = \sum_{i \in R}
\bigl[ U(\mathcal{V}) - U(\mathcal{V} \setminus \{i\}) \bigr],
\end{equation}
where each term measures the coverage improvement attributable to a
single agent. This per-agent decomposition provides a finer credit
assignment signal than removing the entire region at once, and
distinguishes the contribution of each agent within the region. With $\texttt{local\_ratio} = 0.0$, the primary reward
contains no collision penalty; collision avoidance is handled
entirely by the cost head and the Lagrangian multiplier.

\subsubsection{Lagrangian Multiplier Update}

Each agent $i$ maintains a local multiplier $\lambda^i \in [0, \lambda_{\max}]$,
updated at the end of each episode based on observed constraint rewards.

\paragraph{Dual Ascent Update.}
After each episode, agent $i$ updates its multiplier via:
\begin{equation}
\label{eq:dual-update-app}
\lambda^i \leftarrow \Pi_{[0,\,\lambda_{\max}]}\!\left(\lambda^i - \eta
(\hat{V}_1^i - c_{\mathrm{thresh}})\right),
\end{equation}
where $\hat{V}_1^i$ is agent $i$'s observed per-episode constraint
reward, $c_{\mathrm{thresh}}$ is the constraint threshold, and
$\eta$ is the dual learning rate. The projection to
$[0, \lambda_{\max}]$ prevents multiplier divergence.

\emph{Interpretation.} If agent $i$ observes constraint reward below
the threshold ($\hat{V}_1^i < c_{\mathrm{thresh}}$), $\lambda^i$
increases, which in turn increases the weight on the cost head in the
augmented Q-function~\eqref{eq:augmented-q}, shifting the policy
toward safer actions. Conversely, if constraint reward exceeds the
threshold, $\lambda^i$ decreases, allowing the agent to prioritize
coverage.

\paragraph{Central Averaging.}
The multiplier used in~\eqref{eq:augmented-q} during action selection
is the team average $\bar{\lambda} = \frac{1}{N} \sum_{i=1}^N
\lambda^i$. This averaging serves two purposes: (i)~it ensures all
regions use a common augmented Q-function, which is necessary for
Max-Sum to optimize a consistent objective; and (ii)~it pools local
constraint observations across agents, providing a more accurate
estimate of team-wide constraint satisfaction.

\begin{remark*}[Toward Fully Decentralized Multiplier Updates]
The current implementation computes $\bar{\lambda}$ centrally during
training. A fully decentralized variant would replace this central
averaging step with gossip consensus~\citep{nedic2009distributed}:
each agent exchanges its local $\lambda^i$ with communication
neighbors and computes a weighted average using doubly stochastic
weights. The theoretical analysis in Section~\ref{sec:convergence}
supports this extension, as the primal-dual convergence guarantees
hold under gossip-based
consensus~\citep{borkar1997stochastic}.\hfaf
\end{remark*}

\paragraph{Timescale Separation.}
The dual learning rate $\eta$ is set on a slower timescale than the
Q-network learning rate $\alpha$, satisfying $\eta \ll \alpha$. This
ensures that, from the perspective of Q-learning, the multipliers
appear nearly constant---the Q-functions can track the optimal values
for the current $\bar{\lambda}$ before the multiplier changes. In
practice, the multiplier is updated once per episode, while the
Q-network is updated at every time step.

\subsection{Implementation Details}
\label{app:implementation}

\paragraph{Network Architecture.}
The shared two-head Q-network $Q_\theta$ consists of:
\begin{itemize}[nosep]
\item \emph{Input}: Region observation $o^R$ (dimension depends on
$N$; see Table~\ref{tab:scaling})
\item \emph{Trunk}: 2 fully connected layers $\times$ 128 units,
ReLU activations
\item \emph{Primary head}: Linear layer $\to
|\mathcal{A}^R| = |\mathcal{A}^i| \times |\mathcal{A}^k|$ outputs
\item \emph{Cost head}: Linear layer $\to |\mathcal{A}^R|$ outputs
\end{itemize}
Note that $\bar{\lambda}$ is \emph{not} an input to the network.
The multiplier enters only at action selection time through the
augmented Q-function~\eqref{eq:augmented-q}, which is what enables
Pareto front evaluation from a single trained model.

\paragraph{Action Discretization.}
For environments with continuous actions, we discretize each dimension
into 5 bins, yielding $|\mathcal{A}^i| = 25$ actions per agent (for
2D action spaces). For pairwise regions, this gives
$|\mathcal{A}^R| = 25 \times 25 = 625$ joint actions per region. The
discretization enables Max-Sum coordination while maintaining
connections to the tabular convergence theory
(Theorem~\ref{thm:tabular-convergence}).

\paragraph{Exploration.}
We use Gaussian noise added to the augmented Q-tables before Max-Sum
message passing. The noise scale decays as
$\epsilon_t = \max(0.05,\; 0.9 - 0.85 t / T_{\mathrm{decay}})$.
This ensures coordinated exploration: all agents reason about the
same perturbed Q-values, so their actions remain jointly coherent
even during exploration.

\paragraph{Double Q-Learning and Target Networks.}
Both heads use Double Q-learning~\citep{van2016deep} to reduce
overestimation bias. A target network $Q_{\bar{\theta}}$ is
maintained via Polyak averaging:
$\bar{\theta} \gets \tau\theta + (1-\tau)\bar{\theta}$ with
$\tau = 0.005$, applied every 200 time steps. The cost head's
bootstrap target uses the primary head's greedy action
[cf.~\eqref{eq:td-cost}], ensuring consistency between the two
heads.

\paragraph{Replay Buffer.}
Transitions are stored per-region: each entry is
$(o^R, a^R, r_0^R, c^R, o'^R)$. Since all pairwise regions share a
single network, transitions from all $\binom{N}{2}$ regions are
pooled into one buffer of size $|\mathcal{D}| = 10^5$. Mini-batches
of size $B = 64$ are sampled uniformly for training.

\paragraph{Scaling Configurations.}
Table~\ref{tab:scaling} summarizes the observation dimensions and
number of pairwise regions for each team size.

\begin{table}[h]
\centering
\small
\caption{Scaling configurations for different team sizes.}
\label{tab:scaling}
\begin{tabular}{@{}ccccc@{}}
\toprule
$N$ & Agent size & Pairwise regions & Obs.\ dim & Scenario \\
\midrule
3  & 0.080 & 3  & 18 & Random \\
4  & 0.080 & 6  & 20 & Random \\
6  & 0.065 & 15 & 24 & Random \\
10 & 0.051 & 45 & 32 & Random \\
\bottomrule
\end{tabular}
\end{table}

\paragraph{Hyperparameters.}
Table~\ref{tab:hyperparams} summarizes the hyperparameter values
used in experiments.

\begin{table}[h]
\centering
\small
\caption{Hyperparameter values for CG-CMARL experiments.}
\label{tab:hyperparams}
\begin{tabular}{@{}lcc@{}}
\toprule
\textbf{Parameter} & \textbf{Symbol} & \textbf{Value} \\
\midrule
Discount factor & $\gamma$ & 0.99 \\
Target network update rate & $\tau$ & 0.005 \\
Target update frequency & --- & every 200 steps \\
Max-Sum iterations & $K$ & 10 \\
Max-Sum damping coefficient & $d$ & 0.3 \\
Replay buffer size & $|\mathcal{D}|$ & $10^5$ \\
Mini-batch size & $B$ & 64 \\
Q-network learning rate & $\alpha$ & $10^{-3}$ (Adam) \\
Dual learning rate & $\eta$ & $10^{-2}$ \\
Constraint threshold & $c_{\mathrm{thresh}}$ & 0 \\
Multiplier projection range & $[\lambda_{\min}, \lambda_{\max}]$ & $[0, \lambda_{max}]$ \\
Exploration decay period & $T_{\mathrm{decay}}$ & $10^5$ steps \\
Minimum exploration rate & $\epsilon_{\min}$ & 0.05 \\
Local ratio / mixing ratio & \texttt{local\_ratio} & 0.0 \\
Trunk hidden layers & --- & $2 \times 128$, ReLU \\
Actions per dimension & --- & 5 (total $|\mathcal{A}^i| = 25$) \\
\bottomrule
\end{tabular}
\end{table}

\paragraph{Parallel Episode Collection.}
For $N \geq 6$ agents, wall-clock training time becomes significant
due to the $O(N^2)$ pairwise regions. To mitigate this, we use
parallel episode collection with 4--8 workers, each running
independent environment instances and contributing transitions to the
shared replay buffer.

%% file: appendices/appC_proofs.tex


This appendix contains complete proofs of all theoretical results.
Section~\ref{app:proof-q-conv} proves Q-learning convergence
(Theorem~\ref{thm:tabular-convergence}).
Section~\ref{app:proof-dual-conv} proves primal-dual convergence
(Theorem~\ref{thm:dual-convergence}).
Section~\ref{app:proof-error} establishes the error decomposition
(Theorem~\ref{thm:total-error} and
Proposition~\ref{prop:simple-spread-exact}).
Section~\ref{app:proof-policy} proves the policy suboptimality bound.

\subsection{Proof of Theorem~\ref{thm:tabular-convergence}: Regional
Q-Learning Convergence}
\label{app:proof-q-conv}

We prove that the two-head regional Q-learning converges for any
fixed multiplier $\lambda \geq 0$. The two heads learn different
quantities---the primary head performs Q-learning (optimization),
while the cost head performs TD policy evaluation---so we treat each
separately.

\subsubsection{Primary Head Convergence}

The primary head learns $Q_{\mathrm{prim},R}$ via standard
Q-learning with reward $r_0^R$.

\begin{proof}
The proof proceeds in three steps: (i)~the regional Bellman operator
is well-defined under Assumption~(A1), (ii)~it is a contraction, and
(iii)~stochastic approximation theory guarantees convergence.

\textbf{Step 1: The regional Bellman operator is well-defined.}

For region $R \in \mathcal{C}_R$, define the regional Bellman
operator $\mathcal{T}_{0,R} : \mathbb{R}^{|\mathcal{S}^R| \times
|\mathcal{A}^R|} \to \mathbb{R}^{|\mathcal{S}^R| \times
|\mathcal{A}^R|}$ by
\begin{equation}
\label{eq:bellman-op-prim}
(\mathcal{T}_{0,R} Q)(s^R, a^R) := r_0^R(s^R, a^R) + \gamma
\sum_{s'^R \in \mathcal{S}^R} P(s'^R \mid s^R, a^R) \max_{a'^R
\in \mathcal{A}^R} Q(s'^R, a'^R).
\end{equation}
Under Assumption~(A1) (independent transitions), the marginal
transition probability satisfies
\[
P(s'^R \mid s, a) = P(s'^R \mid s^R, a^R)
= \prod_{i \in R} P_i(s'^i \mid s^i, a^i),
\]
where the first equality holds because agents outside $R$ do not
affect states inside $R$. Hence, $\mathcal{T}_{0,R}$ depends only on
$(s^R, a^R)$ and is well-defined as an operator on regional
Q-functions.

\textbf{Step 2: The operator is a $\gamma$-contraction.}

We show $\|\mathcal{T}_{0,R} Q_1 - \mathcal{T}_{0,R} Q_2\|_\infty
\leq \gamma \|Q_1 - Q_2\|_\infty$ for any $Q_1, Q_2$. Using
definition~\eqref{eq:bellman-op-prim}:
\begin{align}
|(\mathcal{T}_{0,R} Q_1 - \mathcal{T}_{0,R} Q_2)(s^R, a^R)|
&= \gamma \left| \sum_{s'^R} P(s'^R \mid s^R, a^R)
\left[ \max_{a'^R} Q_1(s'^R, a'^R) - \max_{a'^R} Q_2(s'^R, a'^R)
\right] \right| \nonumber \\
&\leq \gamma \sum_{s'^R} P(s'^R \mid s^R, a^R)
\left| \max_{a'^R} Q_1(s'^R, a'^R) - \max_{a'^R} Q_2(s'^R, a'^R)
\right|.
\label{eq:contraction-step1}
\end{align}
By non-expansiveness of the max operator,
$|\max_a f(a) - \max_a g(a)| \leq \max_a |f(a) - g(a)|$, we obtain
\[
\left| \max_{a'^R} Q_1(s'^R, a'^R) - \max_{a'^R} Q_2(s'^R, a'^R)
\right| \leq \|Q_1 - Q_2\|_\infty.
\]
Substituting into~\eqref{eq:contraction-step1} and using
$\sum_{s'^R} P(s'^R \mid s^R, a^R) = 1$:
\[
|(\mathcal{T}_{0,R} Q_1 - \mathcal{T}_{0,R} Q_2)(s^R, a^R)|
\leq \gamma \|Q_1 - Q_2\|_\infty.
\]
Taking the supremum over $(s^R, a^R)$ yields the contraction
property.

\textbf{Step 3: Stochastic approximation convergence.}

The primary head update~\eqref{eq:td-prim} can be written as
\begin{equation}
\label{eq:sa-form-prim}
Q_{R,t+1}(s_t^R, a_t^R) = (1 - \alpha_t) Q_{R,t}(s_t^R, a_t^R) +
\alpha_t \left( \mathcal{T}_{0,R} Q_{R,t}(s_t^R, a_t^R) + w_t^R
\right),
\end{equation}
where $w_t^R := r_0^R + \gamma \max_{a'} Q_{R,t}(s_{t+1}^R, a') -
\mathcal{T}_{0,R} Q_{R,t}(s_t^R, a_t^R)$ is the noise term. Let
$\mathcal{F}_t = \sigma(s_0, a_0, \ldots, s_t, a_t)$ denote the
filtration generated by the trajectory. Then
$\mathbb{E}[w_t^R \mid \mathcal{F}_t] = 0$ (martingale difference
sequence) and $\mathbb{E}[|w_t^R|^2 \mid \mathcal{F}_t] \leq C$
for some constant $C$ depending on $R_{\max}$ and $\gamma$ (bounded
variance, using Assumption~(A2)).

By Assumption~\ref{asm:step-sizes}, the step sizes satisfy
$\sum_t \alpha_t = \infty$ and $\sum_t \alpha_t^2 < \infty$. By
Assumption~\ref{asm:exploration}, each state-action pair is visited
infinitely often. These conditions, combined with the contraction
property of $\mathcal{T}_{0,R}$, satisfy the requirements of
\citet[Theorem~1]{tsitsiklis1994asynchronous}. Therefore,
$Q_{\mathrm{prim},R,t} \to Q_{0,R}^{*}$ almost surely, where
$Q_{0,R}^{*}$ is the unique fixed point of $\mathcal{T}_{0,R}$.
\end{proof}

\subsubsection{Cost Head Convergence}

The cost head learns $Q_{\mathrm{cost},R}$ via TD evaluation: the
bootstrap target uses the \emph{primary head's} greedy
action~\eqref{eq:td-cost}, not the cost head's own greedy action.
Hence, the cost head performs \emph{policy evaluation} under the
primary head's policy, not optimization.

\begin{proof}
For a fixed primary policy
$\pi_{\mathrm{prim},R}(s^R) := \argmax_{a^R}
Q_{\mathrm{prim},R}^{*}(s^R, a^R)$, define the regional policy
evaluation operator:
\begin{equation}
\label{eq:bellman-op-cost}
(\mathcal{T}_{\pi,R} Q)(s^R, a^R) := c^R(s^R, a^R) + \gamma
\sum_{s'^R} P(s'^R \mid s^R, a^R)\,
Q\!\left(s'^R, \pi_{\mathrm{prim},R}(s'^R)\right).
\end{equation}
This operator is a $\gamma$-contraction in $\|\cdot\|_\infty$ by the
same argument as Step~2 above, with the max replaced by evaluation
at a fixed action. Indeed, the bound is tighter: the operator is
linear in $Q$ for fixed $\pi$, so
\begin{align}
|(\mathcal{T}_{\pi,R} Q_1 - \mathcal{T}_{\pi,R} Q_2)(s^R, a^R)|
&= \gamma \left| \sum_{s'^R} P(s'^R \mid s^R, a^R)
\big[ Q_1(s'^R, \pi(s'^R)) - Q_2(s'^R, \pi(s'^R)) \big] \right| \nonumber\\
&\leq \gamma \|Q_1 - Q_2\|_\infty.
\end{align}
The cost head update~\eqref{eq:td-cost} uses
$\argmax_{a'} Q_{\mathrm{prim},R,\theta}(o'^R, a')$ as the
bootstrap action. By the convergence of the primary head (shown
above) and the timescale separation, the primary head's greedy
policy stabilizes. Consequently, the cost head's update becomes a
stochastic approximation for $\mathcal{T}_{\pi,R}$, and the same
argument as Step~3 yields
$Q_{\mathrm{cost},R,t} \to Q_{\mathrm{cost},R}^{\pi}$ almost
surely, where $Q_{\mathrm{cost},R}^{\pi}$ is the unique fixed
point of $\mathcal{T}_{\pi,R}$---that is, the expected discounted
collision cost under the primary policy.
\end{proof}

\subsubsection{Independence Across Regions}

Although all regions share data from the same trajectory
$(s_t, a_t, s_{t+1})$, the updates for different regions are
decoupled: the update for $Q_R$ depends only on
$(s_t^R, a_t^R, s_{t+1}^R)$. Hence, convergence of each regional
Q-function (both heads) can be established independently. With
parameter sharing, all pairwise regions train a single network; the
convergence guarantee then applies to the shared network in the
tabular sense, with the pooled transitions providing the exploration
coverage required by Assumption~\ref{asm:exploration}.

\subsection{Proof of Theorem~\ref{thm:dual-convergence}:
Primal-Dual Convergence}
\label{app:proof-dual-conv}

We prove the two-part convergence result for the full primal-dual
algorithm. In CG-CMARL, each agent~$i$ maintains a local multiplier
$\lambda^i$ updated per-episode via dual
ascent~\eqref{eq:dual-update}, and the team average
$\bar{\lambda} = \frac{1}{N}\sum_i \lambda^i$ is used for action
selection. The proof shows that (i)~Q-functions track the optimal
values for the current $\bar{\lambda}$, and (ii)~$\bar{\lambda}$
converges to the optimal dual solution.

\subsubsection{Preliminary: Step-Size Conditions}

\begin{lemma}[Step-Size Conditions]
\label{lem:step-sizes}
Let $\alpha_t = \alpha_0 / (t+1)^{\rho_\alpha}$ and
$\eta_k = \eta_0 / (k+1)^{\rho_\eta}$ (where $k$ indexes episodes
and $t$ indexes time steps) with
$0.5 < \rho_\alpha < \rho_\eta \leq 1$. Then:
\begin{enumerate}[label=(\roman*), nosep]
\item $\sum_{t} \alpha_t = \sum_{k} \eta_k = \infty$
\item $\sum_{t} \alpha_t^2 < \infty$ and
$\sum_{k} \eta_k^2 < \infty$
\item The multiplier updates (per-episode) are on a slower timescale
than Q-updates (per-step): since each episode contains $H$ steps,
the effective ratio is $\eta_k / (H \alpha_{kH}) \to 0$.
\end{enumerate}
\end{lemma}

\begin{proof}
Properties (i) and (ii) follow from the $p$-series: $\sum_t
(t+1)^{-\rho}$ diverges for $\rho \leq 1$ and $\sum_t
(t+1)^{-2\rho}$ converges for $\rho > 0.5$. Property~(iii) holds
because $\eta_k / \alpha_{kH} = (\eta_0 / \alpha_0) \cdot
(kH+1)^{\rho_\alpha} / (k+1)^{\rho_\eta}$, which tends to zero
when $\rho_\eta > \rho_\alpha$ (the episode index $k$ grows slower
than the step index $t = kH$).
\end{proof}

In practice, CG-CMARL uses a fixed dual learning rate $\eta$ with
per-episode updates and a decaying Q-learning rate, which achieves
the required timescale separation: Q-networks update many times per
episode while the multiplier updates once.

\subsubsection{Part (i): Q-Convergence on the Fast Timescale}

\begin{proof}
On the fast timescale (per-step Q-updates), the multiplier
$\bar{\lambda}$ evolves slowly because it changes only once per
episode. By the ``quasi-static'' analysis of two-timescale
stochastic approximation~\citep[Theorem~2,
Chapter~6]{borkar2008stochastic}, the fast iterate (Q-functions)
tracks the equilibrium of the fast dynamics for the slowly-varying
slow iterate (multiplier).

Formally, define the ``frozen'' primary Q-function
$Q_{\mathrm{prim},R}^{*}$ as the fixed point of
$\mathcal{T}_{0,R}$, and the frozen cost Q-function
$Q_{\mathrm{cost},R}^{\pi}$ as the fixed point of
$\mathcal{T}_{\pi,R}$ under the primary greedy policy. By
Theorem~\ref{thm:tabular-convergence}, both heads converge when
$\bar{\lambda}$ is fixed. The two-timescale theory extends this
to show that when $\bar{\lambda}$ varies slowly, the Q-functions
remain close to their frozen optima:
\[
\limsup_{t \to \infty}
\|Q_{\mathrm{prim},R,t} - Q_{0,R}^{*}\|_\infty = 0
\quad \text{a.s.}
\]
and similarly for the cost head. Since $\bar{\lambda}$ enters
only at action selection time through~\eqref{eq:augmented-q} and
does not affect the individual head training targets, the primary
and cost heads converge independently of $\bar{\lambda}$.
\end{proof}

\begin{remark}[Decoupling from $\lambda$]
\label{rem:lambda-decoupling}
A key property of the two-head architecture is that $\bar{\lambda}$
does not appear in the TD targets~\eqref{eq:td-prim}
or~\eqref{eq:td-cost}. Both heads are trained with fixed reward
signals ($r_0^R$ and $c^R$ respectively). The multiplier
$\bar{\lambda}$ enters only at action selection time through the
augmented Q-function~\eqref{eq:augmented-q}. Consequently, the
Q-convergence in Part~(i) holds for \emph{any} trajectory of
$\bar{\lambda}$, without requiring timescale separation for the
Q-learning convergence itself. The timescale separation is needed
only for Part~(ii), to ensure the policy induced by $\bar{\lambda}$
stabilizes before the multiplier updates.\hfaf
\end{remark}

\subsubsection{Part (ii): Dual Optimality}

\begin{proof}
We show that the average multiplier $\bar{\lambda}_k :=
\frac{1}{N}\sum_{i=1}^N \lambda_k^i$ (where $k$ indexes episodes)
converges to the optimal dual solution $\lambda^*$.

\textbf{Step 1: Characterize the average dynamics.}

Summing the per-agent dual update~\eqref{eq:dual-update} over all
agents and dividing by $N$:
\begin{align}
\bar{\lambda}_{k+1} &= \frac{1}{N} \sum_{i=1}^N 
\Pi_{\lambda_{max}}\!\left(\lambda^i_k + \eta
(c_k^i - c_{\mathrm{thresh}})\right).
\label{eq:avg-dual-dynamics}
\end{align}
The projection to $[0,\lambda_{max}]$ is a projection onto the compact set
$\Lambda := [0,\lambda_{max}]$. For the theoretical analysis, we note that
under Assumption~(A3) (Slater's condition), the optimal dual
variables $\lambda^*$ are bounded. The range $[0,\lambda_{max}]$ must
be chosen large enough to contain $\lambda^*$; we assume this
holds.

Since all agents are initialized identically ($\lambda_0^i = 0$)
and observe i.i.d.\ collision costs (by the symmetry of the
pairwise coordination graph and parameter sharing), the local
multipliers remain close to each other throughout training. In
particular, $\lambda_k^i \approx \bar{\lambda}_k$ for all $i$
after the first few episodes. We can therefore approximate the
average dynamics as:
\begin{equation}
\label{eq:avg-dual-approx}
\bar{\lambda}_{k+1} \approx \Pi_\Lambda\!\left[
\bar{\lambda}_k + \eta \left(
\frac{1}{N}\sum_{i=1}^N c_k^i - c_{\mathrm{thresh}}
\right)\right],
\end{equation}
where $\Pi_\Lambda$ denotes projection onto $\Lambda = [0,1]$.

\textbf{Step 2: Identify the mean field.}

The term $\frac{1}{N}\sum_{i=1}^N c_k^i$ is the team-average
collision cost in episode $k$. Under the policy
$\pi_{\bar{\lambda}_k}$ (induced by the augmented Q-function with
multiplier $\bar{\lambda}_k$), the expected team-average collision
cost is
\[
\mathbb{E}_{\pi_{\bar{\lambda}_k}}\!\left[
\frac{1}{N}\sum_{i=1}^N c_k^i \right]
= (1-\gamma)\, V_1^{\pi_{\bar{\lambda}_k}}(s_0),
\]
where $V_1^{\pi}$ is the value function for the constraint reward
(collision cost) under policy $\pi$, and the $(1-\gamma)$ factor
converts the discounted sum to an average-reward scale.

The mean-field ODE associated with~\eqref{eq:avg-dual-approx} is
therefore:
\begin{equation}
\label{eq:dual-ode}
\dot{\lambda} = \Pi_{T_\Lambda(\lambda)}\!\left[
(1-\gamma)\, V_1^{\pi_\lambda}(s_0) - c_{\mathrm{thresh}}
\right],
\end{equation}
where $\Pi_{T_\Lambda(\lambda)}$ denotes projection onto the
tangent cone of $\Lambda$ at $\lambda$.

\textbf{Step 3: Connection to dual descent.}

The Lagrangian of problem~\eqref{eq:cmarl} for a single constraint
($J = 1$) is
\[
\mathcal{L}(\pi, \lambda) = V_0^\pi(s_0)
+ \lambda \!\left( V_1^\pi(s_0) - \frac{c_1}{1-\gamma} \right).
\]
The dual function is $d(\lambda) := \max_\pi \mathcal{L}(\pi,
\lambda)$, and a subgradient of $d$ at $\lambda$ is
\[
g(\lambda) = V_1^{\pi_\lambda}(s_0) - \frac{c_1}{1-\gamma}.
\]
Comparing with~\eqref{eq:dual-ode}, the multiplier update performs
projected subgradient descent on $d(\lambda)$ (up to the
$(1-\gamma)$ scaling, which affects the rate but not the fixed
point).

\textbf{Step 4: Convergence.}

By Assumption~(A3) (Slater's condition), the dual function
$d(\lambda)$ is coercive: $d(\lambda) \to \infty$ as
$\lambda \to \infty$. The dual function is convex (as a pointwise
maximum of affine functions of $\lambda$). Projected subgradient
descent with diminishing step sizes satisfying $\sum_k \eta_k =
\infty$ and $\sum_k \eta_k^2 < \infty$ converges to the optimal
set~\citep[Proposition~3.2.6]{bertsekas2015convex}:
\[
\bar{\lambda}_k \to \lambda^* \in
\argmin_{\lambda \geq 0} d(\lambda).
\]

For the case of fixed $\eta$ (as in the implementation), the
iterates converge to a neighborhood of $\lambda^*$ with radius
$O(\eta)$. This is consistent with the standard result for
subgradient methods with constant step
size~\citep[Section~3.2.3]{bertsekas2015convex}.

By strong duality (which holds under
Assumption~(A3);~\citet{paternain2019constrained}), the policy
$\pi_{\lambda^*}$ is optimal for the original constrained
problem~\eqref{eq:cmarl}.
\end{proof}

\subsubsection{Function Approximation Extension}

\begin{proof}[Proof of Corollary~\ref{cor:fa-convergence}]
With neural network function approximation, exact convergence
guarantees are generally unavailable. However:

\textbf{Primary and cost heads}: Neural network Q-learning seeks
stationary points of the projected Bellman error. Practical
stability is achieved via Double Q-learning (reduces
overestimation), target networks (stabilizes bootstrapping), and
experience replay (breaks temporal correlation). With parameter
sharing, the effective sample size per network update scales with
$\binom{N}{2}$ (all pairwise regions contribute), which improves
generalization.

\textbf{Multiplier update}: The dual ascent
update~\eqref{eq:dual-update} with central averaging is a
deterministic function of observed collision rates and does not
depend on the Q-function parameterization. Its convergence
properties are inherited from projected subgradient
descent~\citep{bertsekas2015convex}, independent of how the
Q-functions are represented.

\textbf{Overall}: The algorithm seeks KKT stationary points of the
Lagrangian $\mathcal{L}(\theta, \lambda)$, where $\theta$
parameterizes the shared Q-network. Convergence to a neighborhood
of such points follows from the analysis of nonlinear two-timescale
stochastic approximation~\citep{borkar2008stochastic}, with the
neighborhood size proportional to the approximation error
$\epsilon_{\mathrm{NN}}$:
\[
\mathrm{dist}\!\left((\theta_t, \bar{\lambda}_t),\;
\mathcal{K}\right) \leq O\!\left(
\frac{\epsilon_{\mathrm{NN}}}{1-\gamma}\right),
\]
where $\mathcal{K}$ is the set of KKT points.
\end{proof}

\subsection{Proof of Error Decomposition Results}
\label{app:proof-error}

\subsubsection{Bellman Operator Properties}

\begin{definition}[Bellman Operators]
\label{def:bellman-ops-proof}
Define the \emph{joint Bellman operator} $\mathcal{T}$ and
\emph{independent regional Bellman operator} $\mathcal{T}_R$ as:
\begin{align}
(\mathcal{T} Q)(s, a) &:= r(s, a) + \gamma \sum_{s'} P(s' \mid s,
a) \max_{a'} Q(s', a'),
\label{eq:joint-bellman-proof} \\
(\mathcal{T}_R Q_R)(s^R, a^R) &:= r^R(s^R, a^R) + \gamma
\sum_{s'^R} P(s'^R \mid s^R, a^R) \max_{a'^R} Q_R(s'^R, a'^R).
\label{eq:regional-bellman-proof}
\end{align}
Additionally, define the \emph{factored Bellman operator} that
preserves the additive structure but performs a joint max:
\begin{equation}
\label{eq:factored-bellman-proof}
(\hat{\mathcal{T}} \hat{Q})(s, a) := \sum_R r^R(s^R, a^R) +
\gamma \sum_{s'} P(s' \mid s, a) \max_{a'} \sum_R \hat{Q}_R(s'^R,
a'^R),
\end{equation}
where $\hat{Q} = \sum_R \hat{Q}_R$ denotes a factored Q-function.
\end{definition}

The distinction between $\hat{\mathcal{T}}$ and the collection
$\{\mathcal{T}_R\}$ is essential:
\begin{itemize}[nosep]
\item $\hat{\mathcal{T}}$ performs a \emph{joint} max over all
agents' actions, accounting for action coupling across overlapping
regions. Its fixed point $\hat{Q}^*$ is the best factored
Q-function that correctly handles inter-region coordination.
\item Each $\mathcal{T}_R$ performs an \emph{independent} max over
actions within region $R$ only, ignoring that these actions also
appear in other regions. The collection of fixed points
$\{Q_R^{*}\}$ is what our regional Q-learning algorithm computes.
\end{itemize}

\begin{lemma}[Contraction]
\label{lem:contraction-proof}
The operators $\mathcal{T}$, $\hat{\mathcal{T}}$, and each
$\mathcal{T}_R$ are $\gamma$-contractions in $\|\cdot\|_\infty$.
\end{lemma}

\begin{proof}
For $\mathcal{T}$ and $\hat{\mathcal{T}}$: the discount factor
$\gamma$ provides contraction, and the max operator is
non-expansive (cf.\ Step~2 of the proof of
Theorem~\ref{thm:tabular-convergence}). For $\mathcal{T}_R$: the
same argument applied to the regional state-action space.
\end{proof}

\subsubsection{Proof of Proposition~\ref{prop:simple-spread-exact}:
Structural Error Characterization}

\begin{proof}
We characterize when the structural error $\beta$ vanishes.

\textbf{Part (a): Non-overlapping regions.}

Suppose the coordination graph $\mathcal{C}_R$ is such that each
agent $i$ belongs to \emph{at most one} region (i.e., the regions
are disjoint). Then actions in different regions are completely
independent: action $a^R$ in region $R$ does not appear in any other
region $R'$. In this case, the joint max decomposes:
\[
\max_{a'} \sum_R Q_R(s'^R, a'^R) = \sum_R \max_{a'^R}
Q_R(s'^R, a'^R),
\]
because each $a'^R$ can be optimized independently. It follows that
$\hat{\mathcal{T}} \hat{Q} = \sum_R \mathcal{T}_R Q_R$ for any
factored $\hat{Q} = \sum_R Q_R$. Hence the factored Bellman
operator coincides with independent regional operators, and the
fixed points agree: $\hat{Q}^* = \sum_R Q_R^{*}$.

Moreover, by an inductive argument on the value iteration
$Q_0 = 0$, $Q_{t+1} = \mathcal{T} Q_t$: the reward factorizes
(Assumption~(A2)), the transitions factorize (Assumption~(A1)), and
the max decomposes (non-overlapping regions), so each iterate $Q_t$
factorizes exactly. Taking $t \to \infty$, $Q^* = \sum_R Q_R^{*}$
and $\beta = 0$.

\textbf{Part (b): Overlapping regions.}

When regions overlap---that is, some agent $i$ belongs to multiple
regions---the joint max does \emph{not} decompose in general. To
see why, consider agent $i$ belonging to regions $R$ and $R'$. The
action $a'^i$ appears in both $Q_R(s'^R, a'^R)$ and
$Q_{R'}(s'^{R'}, a'^{R'})$. Optimizing $a'^i$ for $Q_R$ alone may
differ from optimizing it for $Q_R + Q_{R'}$ jointly. Hence:
\[
\max_{a'} \sum_R Q_R(s'^R, a'^R) \;\leq\;
\sum_R \max_{a'^R} Q_R(s'^R, a'^R),
\]
with equality only in special cases (e.g., when the optimal $a'^i$
is the same for all regions containing agent $i$).

For overlapping coordination graphs, the structural error $\beta$
captures the gap between the joint and independent Bellman
operators:
\begin{equation}
\label{eq:beta-overlap}
\beta = \sup_{s,a} \left| \sum_R \max_{a'^R} Q_R^{*}(s'^R, a'^R)
- \max_{a'} \sum_R Q_R^{*}(s'^R, a'^R) \right|
\frac{\gamma}{1-\gamma}.
\end{equation}
This is precisely the per-step ``coordination gap'' between
independent and joint action selection, propagated over the
infinite horizon. For the pairwise coordination graph with $N \geq
3$ agents, regions overlap (each agent appears in $N-1$ pairs), so
$\beta \geq 0$ in general.

\textbf{Part (c): When is $\beta$ small?}

Even for overlapping graphs, $\beta$ may be small when:
\begin{enumerate}[label=(\roman*), nosep]
\item \emph{Weak inter-region coupling}: If the Q-values are
dominated by individual agent terms rather than pairwise
interactions, then the independent max is a good approximation to
the joint max.
\item \emph{Max-Sum compensates}: Although $\beta > 0$ for
independent regional Q-learning, the Max-Sum action selection in
Layer~2 performs the \emph{joint} max (approximately). This means
the coordination error $\epsilon_{\mathrm{MS}}$ and the structural
error $\beta$ partially compensate: Max-Sum corrects the action
coupling that independent Q-learning misses.
\end{enumerate}
In the Simple Spread environment, the pairwise collision penalties
are relatively small compared to the individual landmark rewards.
Hence, the coupling between overlapping regions is weak, and the
structural error $\beta$ is small in practice---though not exactly
zero as previously claimed.
\end{proof}

\begin{remark}[Relation to Guestrin et al.\ (2002)]
\label{rem:guestrin}
\citet{guestrin2002contextspecific} showed that the optimal
Q-function for a factored MDP can be \emph{represented} as a sum
of regional functions when the coordination graph captures all
reward and transition dependencies. However, their result defines
the regional components through a \emph{coupled} fixed-point
system (the factored Bellman operator $\hat{\mathcal{T}}$), not
through \emph{independent} regional Bellman operators
$\{\mathcal{T}_R\}$. The independent regional Q-learning used in
CG-CMARL is an approximation that avoids the coupling at the cost of
structural error $\beta$. Max-Sum message passing recovers part of
this lost coordination at action selection time.\hfaf
\end{remark}

\subsubsection{Proof of Theorem~\ref{thm:total-error}:
Compositional Error Bound}

\begin{proof}
We decompose the total error using the triangle inequality.

\textbf{Step 1: Define intermediate Q-functions.}

Let:
\begin{itemize}[nosep]
\item $Q^* =$ optimal joint Q-function (fixed point of
$\mathcal{T}$)
\item $\hat{Q}^* = \sum_R \hat{Q}_R^{*} =$ optimal factored
Q-function (fixed point of $\hat{\mathcal{T}}$, i.e., the best
sum-of-regions approximation with joint action selection)
\item $\tilde{Q} = \sum_R Q_R^{*} =$ sum of independently learned
regional Q-functions (fixed points of $\{\mathcal{T}_R\}$)
\item $Q_{\mathrm{learned}} = \sum_R Q_{\mathrm{learned},R} =$
learned Q-function (with finite samples and function approximation)
\end{itemize}

\textbf{Step 2: Apply triangle inequality.}
\begin{equation}
\label{eq:triangle-decomp}
\|Q_{\mathrm{learned}} - Q^*\|_\infty \leq
\underbrace{\|Q_{\mathrm{learned}} - \tilde{Q}\|_\infty}_{(a)} +
\underbrace{\|\tilde{Q} - \hat{Q}^*\|_\infty}_{(b)} +
\underbrace{\|\hat{Q}^* - Q^*\|_\infty}_{(c)}.
\end{equation}

\textbf{Step 3: Bound term (c) --- Structural error.}

Term (c) measures the approximation error from restricting to
factored Q-functions. By the approximate dynamic programming
framework~\citep[Proposition~6.1]{bertsekas2019reinforcement}, if
the factored Bellman operator $\hat{\mathcal{T}}$ approximates the
joint operator $\mathcal{T}$ with per-step error
$\beta := \sup_{Q} \|\hat{\mathcal{T}} Q - \mathcal{T} Q\|_\infty$,
then the fixed-point error satisfies:
\[
\|\hat{Q}^* - Q^*\|_\infty \leq \frac{\beta}{1 - \gamma}.
\]
By Proposition~\ref{prop:simple-spread-exact}, $\beta = 0$ when
regions are non-overlapping, and $\beta \geq 0$ for overlapping
graphs such as the pairwise coordination graph with $N \geq 3$.

\textbf{Step 4: Bound term (b) --- Independent learning error.}

Term (b) measures the gap between independent regional Q-learning
and the coupled factored optimum. This arises because each
$\mathcal{T}_R$ performs an independent max over $a'^R$, while
$\hat{\mathcal{T}}$ performs a joint max over $a'$. The per-step
gap is bounded by the ``coordination deficit'' from independent
action selection:
\[
\|\tilde{Q} - \hat{Q}^*\|_\infty
\leq \frac{\epsilon_{\mathrm{MS}}}{1 - \gamma},
\]
where $\epsilon_{\mathrm{MS}}$ is the Max-Sum coordination error,
defined as the worst-case gap between the joint optimum
$\max_{a'} \sum_R Q_R(s'^R, a'^R)$ and the action selected by
Max-Sum (or, in the case of independent learning, the per-region
independent maxima). For tree-structured graphs,
$\epsilon_{\mathrm{MS}} = 0$ because Max-Sum recovers the exact
joint optimum.

\emph{Justification}: The factored Bellman operator
$\hat{\mathcal{T}}$ and the sum of independent operators differ
only in the action selection step. The contraction property ensures
that a per-step error of $\epsilon_{\mathrm{MS}}$ in action
selection propagates to at most $\epsilon_{\mathrm{MS}}/(1-\gamma)$
in the fixed-point error, by the standard error propagation bound
for approximate value
iteration~\citep[Proposition~6.1]{bertsekas2019reinforcement}.

\textbf{Step 5: Bound term (a) --- Sampling and representation
error.}

This combines the finite-sample error and function approximation
error:
\[
\|Q_{\mathrm{learned}} - \tilde{Q}\|_\infty \leq
\epsilon_{\mathrm{sample}} + \epsilon_{\mathrm{NN}}.
\]
The sampling error satisfies
$\epsilon_{\mathrm{sample}} = O(1/\sqrt{n})$ by
Proposition~\ref{prop:td-sampling} (see
Section~\ref{app:sampling-proof}). The representation error
$\epsilon_{\mathrm{NN}} := \min_\theta \max_{R,s^R,a^R}
|Q_{R,\theta}(s^R, a^R) - Q_R^{*}(s^R, a^R)|$ is the best
approximation error achievable by the neural network class.

\textbf{Step 6: Combine.}
\[
\|Q_{\mathrm{learned}} - Q^*\|_\infty \leq
\frac{\beta}{1-\gamma}
+ \frac{\epsilon_{\mathrm{MS}}}{1-\gamma}
+ O\!\left(\frac{1}{\sqrt{n}}\right)
+ \epsilon_{\mathrm{NN}}.
\]
\end{proof}

\subsection{Proof of Policy Suboptimality Bound}
\label{app:proof-policy}

\begin{lemma}[Simulation Lemma~{\citep[Lemma~2]{kearns2002near}}]
\label{lem:simulation}
Let $\hat{Q}$ satisfy
$\|\hat{Q} - Q^*\|_\infty \leq \epsilon$, and let
$\hat{\pi}(s) = \argmax_a \hat{Q}(s, a)$. Then:
\[
V^*(s) - V^{\hat{\pi}}(s) \leq \frac{2\epsilon}{1 - \gamma}.
\]
\end{lemma}

\begin{proof}[Proof of Policy Suboptimality]
By Theorem~\ref{thm:total-error}, the learned Q-function satisfies
\[
\|Q_{\mathrm{learned}} - Q^*\|_\infty \leq \epsilon :=
\frac{\beta + \epsilon_{\mathrm{MS}}}{1-\gamma}
+ O\!\left(\frac{1}{\sqrt{n}}\right) + \epsilon_{\mathrm{NN}}.
\]
Applying Lemma~\ref{lem:simulation}:
\[
V^*(s) - V^{\hat{\pi}}(s) \leq \frac{2\epsilon}{1-\gamma}
= \frac{2}{1-\gamma} \left[
\frac{\beta + \epsilon_{\mathrm{MS}}}{1-\gamma}
+ O\!\left(\frac{1}{\sqrt{n}}\right)
+ \epsilon_{\mathrm{NN}} \right].
\]
\end{proof}

\subsection{Sampling Error Bound}
\label{app:sampling-proof}

\begin{proposition}[Q-Learning Sample
Complexity~{\citep{evendar2003learning}}]
\label{prop:td-sampling}
Consider tabular Q-learning with learning rate
$\alpha_t = 1/(1 + \mathrm{visits}(s,a))$, run for $T$ steps.
With probability at least $1 - \delta$:
\[
\|Q_T - Q^*\|_\infty \leq O\!\left(
\frac{R_{\max}}{(1-\gamma)^2}
\sqrt{\frac{\log(|\mathcal{S}||\mathcal{A}|T/\delta)}{n_{\min}}}
\right),
\]
where $n_{\min} = \min_{s,a} N_T(s,a)$ is the minimum visitation
count.
\end{proposition}

\begin{remark}[Regional Application]
\label{rem:regional-sampling}
For factored Q-learning, each region $R$ runs independent
Q-learning on a state-action space of size
$|\mathcal{S}^R| \times |\mathcal{A}^R|$. Applying
Proposition~\ref{prop:td-sampling} to each region and taking a
union bound over $|\mathcal{C}_R|$ regions:
\[
\max_R \|Q_{R,T} - Q_R^{*}\|_\infty \leq O\!\left(
\frac{R_{\max}}{(1-\gamma)^2}
\sqrt{\frac{\log(|\mathcal{C}_R| \cdot \max_R
|\mathcal{S}^R||\mathcal{A}^R| \cdot T/\delta)}{n_{\min}}}
\right)
\]
with probability at least $1 - \delta$. With parameter sharing,
transitions from all $\binom{N}{2}$ pairwise regions contribute to
training a single network, effectively increasing $n_{\min}$ by a
factor of $\binom{N}{2}$ relative to training separate
networks.\hfaf
\end{remark}

\subsection{Constraint Violation Bound}
\label{app:constraint-violations}

\begin{proposition}[Finite-Time Constraint Violation]
\label{prop:constraint-violation}
Under the assumptions of Theorem~\ref{thm:dual-convergence}, the
cumulative constraint violation over $K$ episodes satisfies
\[
\sum_{k=0}^{K-1} \left[ c_{\mathrm{thresh}}
- \bar{c}_k \right]_+
\leq O\!\left( \frac{(\lambda_0 - \lambda^*)^2}{\eta}
+ \eta K \right),
\]
where $\bar{c}_k := \frac{1}{N}\sum_i c_k^i$ is the team-average
collision cost in episode $k$. Optimizing over $\eta$ yields a
cumulative violation of $O(\sqrt{K})$.
\end{proposition}

\begin{proof}[Proof Sketch]
This follows from the regret analysis of projected subgradient
descent on the dual function. The dual update moves $\bar{\lambda}$
toward $\lambda^*$ at rate $\eta$, with the per-episode constraint
violation serving as the subgradient. The standard bound for
projected subgradient descent with constant step
size~\citep[Section~3.2.3]{bertsekas2015convex} yields the result.
The $O(\sqrt{K})$ rate is achieved by choosing
$\eta \propto 1/\sqrt{K}$, which balances the bias-variance
tradeoff.
\end{proof}

%% file: suplementary/appD_implementation.tex

This appendix describes the Simple Spread environment used in our
experiments (Section~\ref{app:simple-spread}), the reward and
constraint decompositions
(Section~\ref{app:reward-constraint-decomp}), computational
complexity (Section~\ref{app:complexity}), and a discussion of
when the factorization introduces error
(Section~\ref{app:coverage-violation}). For hyperparameter values
and network architecture, see Appendix~\ref{app:implementation}.

\subsection{Simple Spread Environment}
\label{app:simple-spread}

The \texttt{simple\_spread} environment from
PettingZoo/MPE~\citep{lowe2017multi} consists of $N$ agents and $N$
landmarks in a 2D continuous space. Agents must cover all landmarks
(each landmark should have at least one nearby agent) while avoiding
inter-agent collisions.

\paragraph{State Space.}
Each agent $i$ observes its own velocity and position
$(\mathrm{vel}_i, \mathrm{pos}_i) \in \mathbb{R}^4$, the positions
of all landmarks $(\mathrm{pos}_{\ell_1}, \ldots,
\mathrm{pos}_{\ell_N}) \in \mathbb{R}^{2N}$, and the positions of
all other agents $(\mathrm{pos}_{k})_{k \neq i} \in
\mathbb{R}^{2(N-1)}$. The global state is the concatenation of all
agent states.

\paragraph{Action Space.}
Each agent applies a continuous 2D force. We discretize each
dimension into 5 bins (including no-op), yielding
$|\mathcal{A}^i| = 25$ actions per agent. For a pairwise region
$R = \{i,k\}$, the joint action space has
$|\mathcal{A}^R| = 625$ elements.

\paragraph{Verification of Assumptions.}
Simple Spread satisfies both structural assumptions:
\begin{description}[nosep]
\item[(A1)] \emph{Independent transitions.} Each agent's position
and velocity update depends only on its own state and action
(applied force). There are no physical rebounds on collision---agents
pass through each other, incurring only a reward penalty.
\item[(A2)] \emph{Additive reward.} The global reward decomposes
into individual landmark-distance terms and pairwise collision
penalties (see Section~\ref{app:reward-constraint-decomp}).
\end{description}

\paragraph{Scaling Experiments.}
We evaluate CG-CMARL with $N \in \{3, 4, 6, 10\}$ agents. The
observation dimension, number of pairwise regions, and joint action
space grow as shown in Table~\ref{tab:scaling}.

\begin{table}[h]
\centering
\small
\caption{Scaling of problem dimensions with $N$ agents.}
\label{tab:scaling}
\begin{tabular}{@{}rrrrl@{}}
\toprule
$N$ & $|\mathcal{C}_R|$ & $|o^R|$ &
$|\mathcal{A}|^N$ & Status \\
\midrule
3  & 3   & 16 & $1.6 \times 10^4$  & Tractable (brute force) \\
4  & 6   & 18 & $3.9 \times 10^5$  & Tractable (brute force) \\
6  & 15  & 22 & $2.4 \times 10^8$  & Intractable centrally \\
10 & 45  & 30 & $9.5 \times 10^{13}$ & Intractable centrally \\
\bottomrule
\end{tabular}
\end{table}

\subsection{Reward and Constraint Decomposition}
\label{app:reward-constraint-decomp}

\subsubsection{Primary Reward: Per-Agent Counterfactual Rewards}

The global coverage utility is the negative sum of minimum
agent-to-landmark distances:
\[
U(\mathcal{V}) = -\sum_{\ell=1}^{N} \min_{i \in \mathcal{V}}
\|\mathrm{pos}_i - \mathrm{pos}_\ell\|.
\]

For each pairwise region $R = \{i,k\}$, the primary reward is the
\emph{per-agent counterfactual
reward}~\citep{wolpert2001optimal}:
\begin{equation}
\label{eq:diff-reward-app-D}
r_0^R = \sum_{i \in R}
\bigl[ U(\mathcal{V}) - U(\mathcal{V} \setminus \{i\}) \bigr],
\end{equation}
which sums the marginal contributions of each agent in $R$ to
global coverage. Compared to the global counterfactual
$U(\mathcal{V}) - U(\mathcal{V} \setminus R)$, which removes
the entire region at once, the per-agent variant distinguishes
the contribution of each agent within the region.

\paragraph{Relation to Assumption~(A2).}
For non-overlapping coordination graphs (e.g., the pairwise graph
with $N = 2$), the per-agent counterfactual rewards satisfy the
additive factorization~\eqref{eq:reward-decomp} exactly: each agent
belongs to a single region, so marginal contributions do not overlap
(Proposition~\ref{prop:simple-spread-exact}(i)). For overlapping
pairwise graphs ($N \geq 3$), the marginal contributions overlap
and $\sum_R r_0^R \neq r_0$ in general. We use per-agent
counterfactual rewards in this regime following standard practice in
cooperative MARL~\citep{wolpert2001optimal}, given their strong
credit assignment properties: each agent is rewarded only for the
coverage improvement it individually produces. The resulting discrepancy is absorbed into
the structural error $\beta$, which the compositional bound
(Theorem~\ref{thm:total-error}) already characterizes. The
\texttt{local\_ratio} parameter (set to $0.0$) controls the mixture
between per-agent counterfactual rewards and local rewards; at $0.0$,
the primary reward contains no collision penalty.

\subsubsection{Collision Cost}

In the implementation, collision costs $c^R \geq 0$ are computed
directly; the constraint reward used in the theoretical formulation
(Sections~\ref{sec:two-head}--\ref{sec:lagrangian}) is
$r_1^R = -c^R$, so that the augmented
Q-function~\eqref{eq:augmented-q} with $+\lambda\,Q_{\mathrm{cost}}$
penalizes collisions.

For the cost head, each pairwise region $R = \{i,k\}$ observes a
binary collision indicator:
\begin{equation}
\label{eq:collision-cost-app}
c^R(s^R) = \mathbf{1}\!\left[\|\mathrm{pos}_i - \mathrm{pos}_k\|
< d_{\mathrm{thresh}}\right],
\end{equation}
where $d_{\mathrm{thresh}} = 0.2$ is the collision distance. The
cost head learns $Q_{\mathrm{cost},R}$, the expected discounted
collision frequency for the pair $(i,k)$.

\subsubsection{Constraint Definition}

The Lagrangian multiplier penalizes excessive collisions. Agent
$i$'s episode-level collision cost is
\[
c^i_{\mathrm{episode}} = \frac{1}{H} \sum_{t=1}^{H}
\sum_{k \neq i} \mathbf{1}\!\left[\|\mathrm{pos}_i(t) -
\mathrm{pos}_k(t)\| < d_{\mathrm{thresh}}\right],
\]
where $H$ is the episode length. In terms of the constraint reward
convention, $\hat{V}_1^i = -c^i_{\mathrm{episode}}$, so the dual
update~\eqref{eq:dual-update} increases $\lambda^i$ when
$c^i_{\mathrm{episode}} > c_{\mathrm{thresh}}$ (equivalently, when
$\hat{V}_1^i < -c_{\mathrm{thresh}}$) and decreases it otherwise.
With $c_{\mathrm{thresh}} = 0$, the constraint aims for zero
collisions; in practice, the clamp to $[0, \lambda_{\max}]$ limits
the multiplier's influence, allowing the Pareto sweep to explore
the full coverage--safety tradeoff.

\subsection{Computational Complexity}
\label{app:complexity}

Table~\ref{tab:complexity} compares the per-step computational cost
of centralized, factored (CG-CMARL), and fully independent
approaches.

\begin{table}[h]
\centering
\small
\caption{Computational complexity comparison. Here $|A|$ denotes
the per-agent action size, $|R|_{\max}$ the largest region, and $K$
the number of Max-Sum iterations.}
\label{tab:complexity}
\begin{tabular}{@{}lcc@{}}
\toprule
\textbf{Method} & \textbf{Action Selection} &
\textbf{Q-Update} \\
\midrule
Centralized &
$O(|A|^N)$ &
$O(|A|^N)$ per sample \\
Factored (CG-CMARL) &
$O(|\mathcal{C}_R| |A|^{|R|_{\max}} K)$ &
$O(|\mathcal{C}_R| |A|^{|R|_{\max}})$ per sample \\
Independent (IQL) &
$O(N |A|)$ &
$O(N |A|)$ per sample \\
\bottomrule
\end{tabular}
\end{table}

For the pairwise graph ($|R|_{\max} = 2$,
$|\mathcal{C}_R| = \binom{N}{2}$) with $|A| = 25$ and $K = 10$:

\begin{table}[h]
\centering
\small
\caption{Concrete action-selection cost for Simple Spread.}
\label{tab:concrete-complexity}
\begin{tabular}{@{}rrrr@{}}
\toprule
$N$ & Centralized & CG-CMARL & IQL \\
\midrule
3  & $1.6 \times 10^4$ & $1.9 \times 10^4$ & $75$ \\
4  & $3.9 \times 10^5$ & $3.8 \times 10^4$ & $100$ \\
6  & $2.4 \times 10^8$ & $9.4 \times 10^4$ & $150$ \\
10 & $9.5 \times 10^{13}$ & $2.8 \times 10^5$ & $250$ \\
\bottomrule
\end{tabular}
\end{table}

\emph{Interpretation.} For $N = 3$, CG-CMARL is comparable in cost to
centralized search. The advantage becomes dramatic for $N \geq 6$,
where centralized action selection is intractable while CG-CMARL
remains polynomial. The cost relative to IQL is the price of
coordination: CG-CMARL's $O(N^2)$ scaling reflects the pairwise
interaction structure, while IQL's $O(N)$ ignores it entirely.
With parameter sharing, the Q-update cost is independent of
$|\mathcal{C}_R|$ since all regions share a single network---the
$|\mathcal{C}_R|$ factor applies only to transition collection and
Max-Sum message passing.

\subsection{When Factorization Introduces Error}
\label{app:coverage-violation}

The structural error $\beta$ (Proposition~\ref{prop:simple-spread-exact})
arises from two sources, which we discuss in turn.

\paragraph{Source 1: Action Overlap in Pairwise Graphs.}
As shown in Proposition~\ref{prop:simple-spread-exact}(ii), for
pairwise coordination graphs with $N \geq 3$ agents, the
independent per-region max does not equal the joint max over all
agents' actions. This contributes $\beta > 0$ even when
Assumptions~(A1)--(A2) hold perfectly. The magnitude of this
contribution depends on the strength of coupling between overlapping
regions---specifically, how much the optimal action for agent $i$ in
region $R_{ij}$ differs from its optimal action considering all
regions $\{R_{ik}\}_{k \neq i}$ simultaneously. In Simple Spread,
this coupling is driven by collision penalties, which are typically
small relative to landmark rewards.

\paragraph{Source 2: Transition Dependencies.}
If the environment violates Assumption~(A1)---for instance, if
collisions cause physical rebounds or agents draw from shared
resources---then the regional Bellman operator
$\mathcal{T}_R$~\eqref{eq:regional-bellman-proof} is no longer
exact, because $P(s'^R \mid s, a) \neq P(s'^R \mid s^R, a^R)$.
This adds a transition-coupling component to $\beta$.

\paragraph{Source 3: Reward Dependencies Beyond the Graph.}
If the reward contains interactions not captured by the coordination
graph (e.g., three-body interactions with only pairwise regions),
then Assumption~(A2) is violated and $\beta$ includes a
reward-coverage component.

\begin{example}[Higher-Order Reward Interaction]
In a cooperative defense scenario where a target's state depends on
simultaneous attacks from three agents:
\[
r_{\mathrm{capture}}(s, a^1, a^2, a^3) \neq
\sum_{i < j} r^{ij}(s^{\{i,j\}}, a^{\{i,j\}}),
\]
a pairwise graph cannot represent this interaction. Extending
$\mathcal{C}_R$ to include triplet regions captures the dependency
at the cost of increasing the per-region action space from
$|A|^2$ to $|A|^3$.
\end{example}

\paragraph{Practical Guidance.}
The choice of coordination graph involves a
tractability--accuracy tradeoff:
\begin{description}[nosep]
\item[Richer graphs] (larger regions, more hyperedges): smaller
$\beta$, but larger per-region action spaces and denser factor
graphs (slower Max-Sum, more parameters).
\item[Sparser graphs] (pairwise or tree-structured): larger $\beta$
but faster action selection and simpler message passing ($K$ can be
smaller; $\epsilon_{\mathrm{MS}} = 0$ for trees).
\end{description}
The compositional error bound
(Theorem~\ref{thm:total-error}) provides a principled way to reason
about this tradeoff: increasing graph richness reduces $\beta$ but
may increase $\epsilon_{\mathrm{MS}}$ (denser cycles) and
$\epsilon_{\mathrm{NN}}$ (larger input/output dimensions).

\subsection{Pareto Evaluation on Fixed Landmark Scenarios}
\label{app:scenario-pareto}

The Pareto fronts in Section~\ref{sec:exp-results} are evaluated on
environments with randomized initial landmark positions. To further
probe how each method behaves under controlled layouts, we
additionally evaluate every trained model (CG-CMARL and all
baselines) on three fixed landmark scenarios from our scenario
registry (Appendix~\ref{app:simple-spread}):
\begin{itemize}[nosep]
\item \texttt{spread\_uniform}: $N$ landmarks placed evenly on a
circle of radius $0.7$ (regular polygon vertices). This is a
balanced layout with no geometric bias.
\item \texttt{line}: $N$ landmarks evenly spaced along a horizontal
corridor at $y = 0$. Tests whether agents can split along a linear
structure rather than a $2$D cluster.
\item \texttt{clustered\_pair}: two clusters of roughly $N/2$
landmarks, one centred at $(-0.5, -0.5)$ and one at $(0.5, 0.5)$.
Tests whether agents split into groups and commit to distant
subsets.
\end{itemize}
Training is unchanged (agents train on random scenarios); only
evaluation is run on these fixed layouts. For each scenario we
evaluate the same trained models used in Section~\ref{sec:exp-results}
at the same $\lambda$ sweep for CG-CMARL and at the same
reward-shaping settings for the baselines. Each operating point is
averaged over multiple evaluation episodes with fixed landmarks but
randomized initial agent positions.

Figures~\ref{fig:scenario-pareto-N3}--\ref{fig:scenario-pareto-N10}
show the resulting Pareto fronts, with rows indexed by team size
$N \in \{3, 4, 6, 10\}$ and columns by scenario
(\texttt{spread\_uniform}, \texttt{line}, \texttt{clustered\_pair}).
Each panel plots the average number of colliding agent pairs per
step (x-axis, inverted; lower is better) against coverage
(y-axis), with CG-CMARL's $\lambda$-sweep tracing a curve and
each baseline
appearing as a single operating point. The qualitative pattern
observed in the random-scenario evaluation persists across fixed
layouts: CG-CMARL's front traces or dominates the baselines,
particularly in the moderate-$\lambda$ regime. Performance on
\texttt{clustered\_pair} is the most sensitive to team size,
reflecting the difficulty of splitting agents between distant
subsets as $N$ grows.

\begin{figure}[htbp]
    \centering
    \begin{subfigure}{0.32\textwidth}
        \centering
        \includegraphics[width=\linewidth]{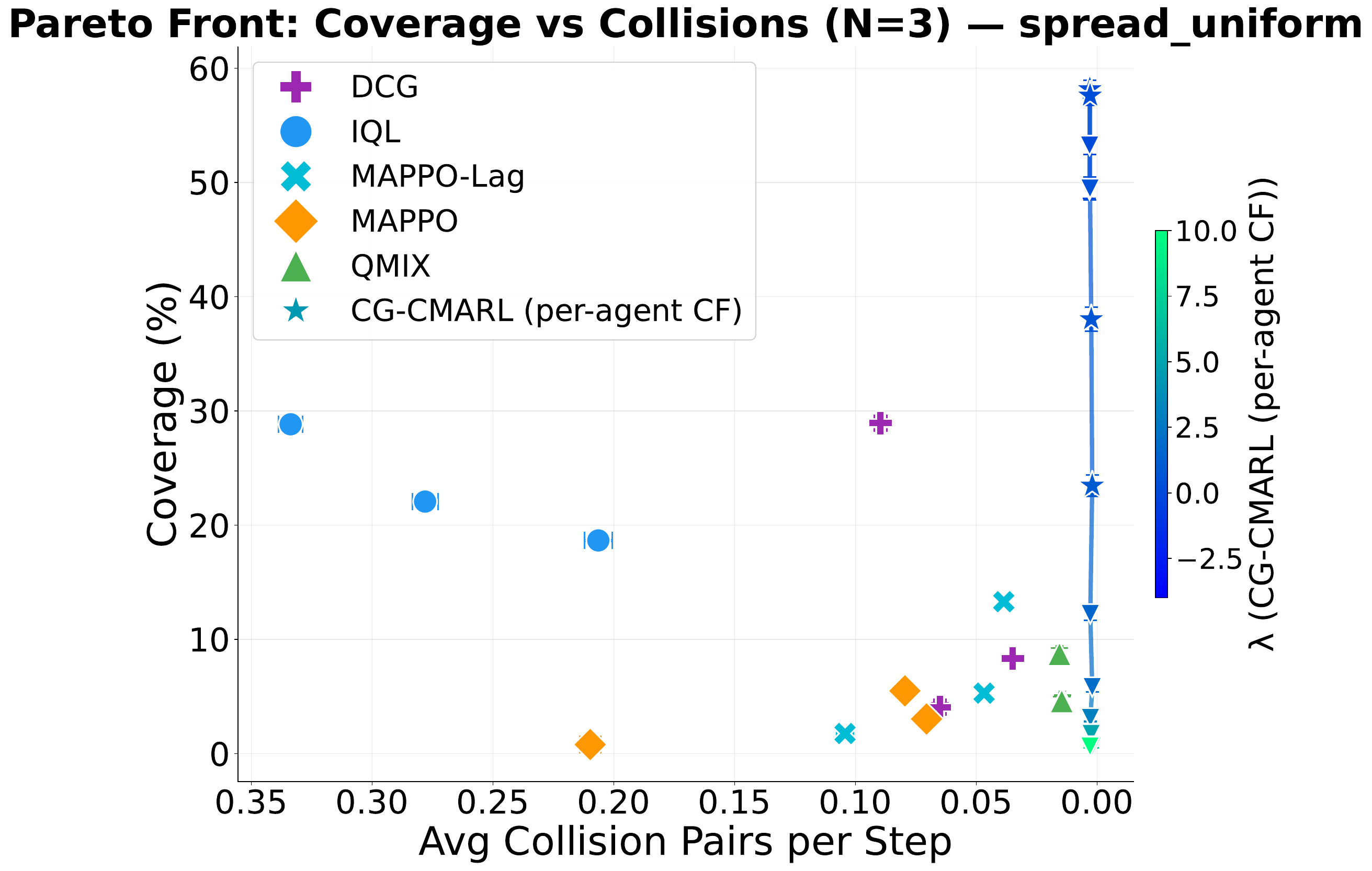}
        \caption{\texttt{spread\_uniform}}
    \end{subfigure}
    \hfill
    \begin{subfigure}{0.32\textwidth}
        \centering
        \includegraphics[width=\linewidth]{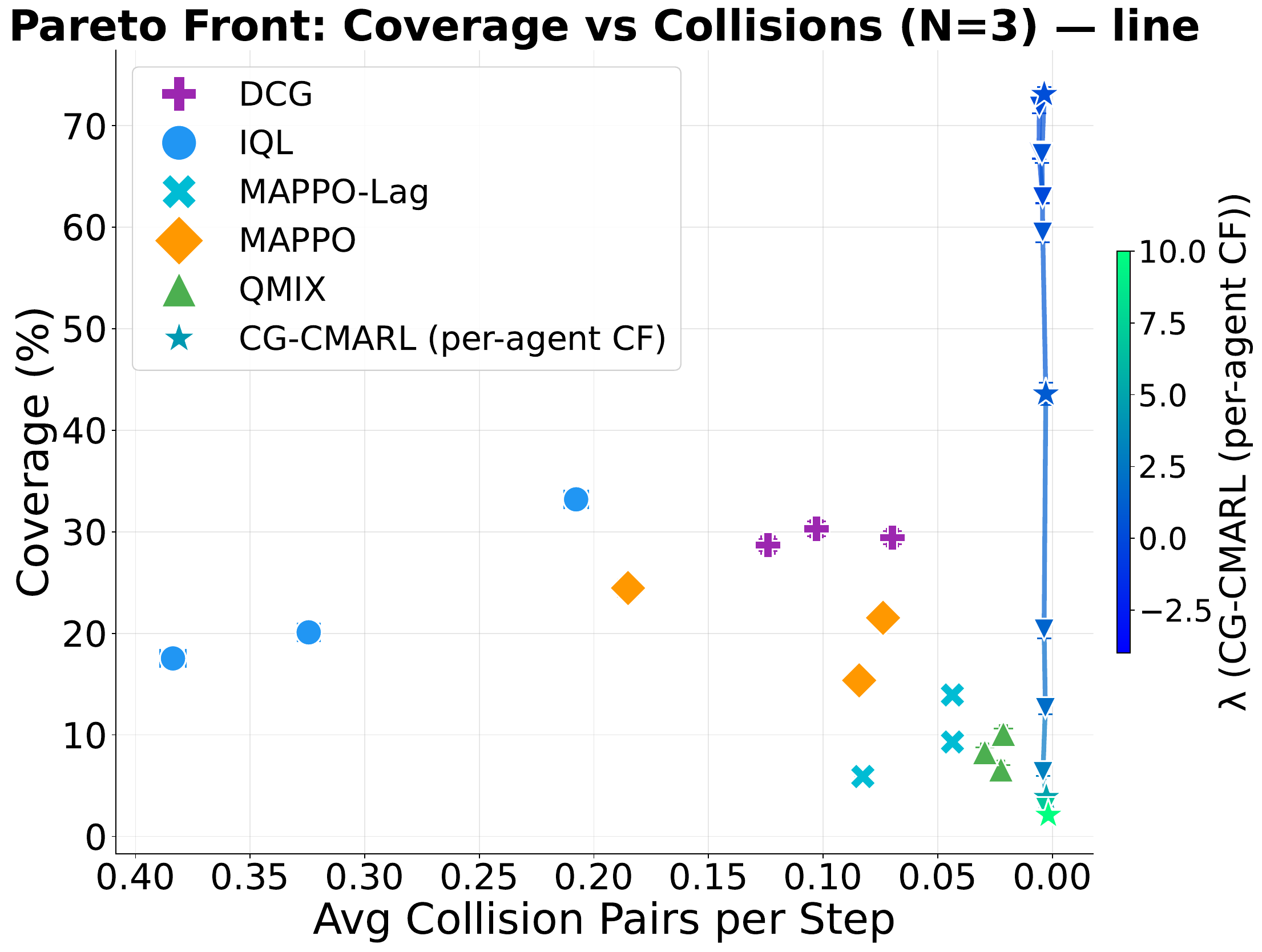}
        \caption{\texttt{line}}
    \end{subfigure}
    \hfill
    \begin{subfigure}{0.32\textwidth}
        \centering
        \includegraphics[width=\linewidth]{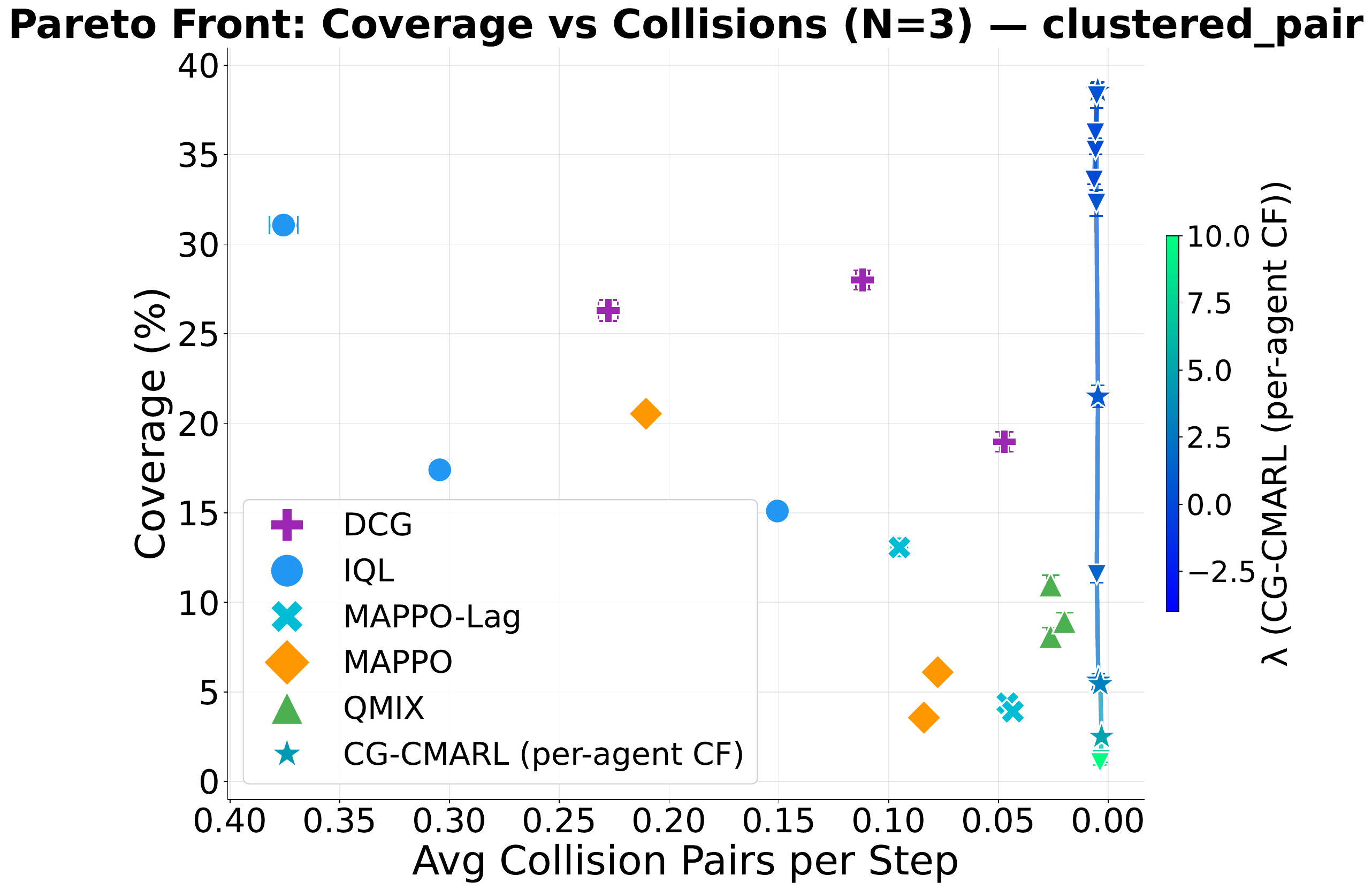}
        \caption{\texttt{clustered\_pair}}
    \end{subfigure}
    \caption{Pareto fronts on fixed landmark scenarios for $N=3$.}
    \label{fig:scenario-pareto-N3}
\end{figure}

\begin{figure}[htbp]
    \centering
    \begin{subfigure}{0.32\textwidth}
        \centering
        \includegraphics[width=\linewidth]{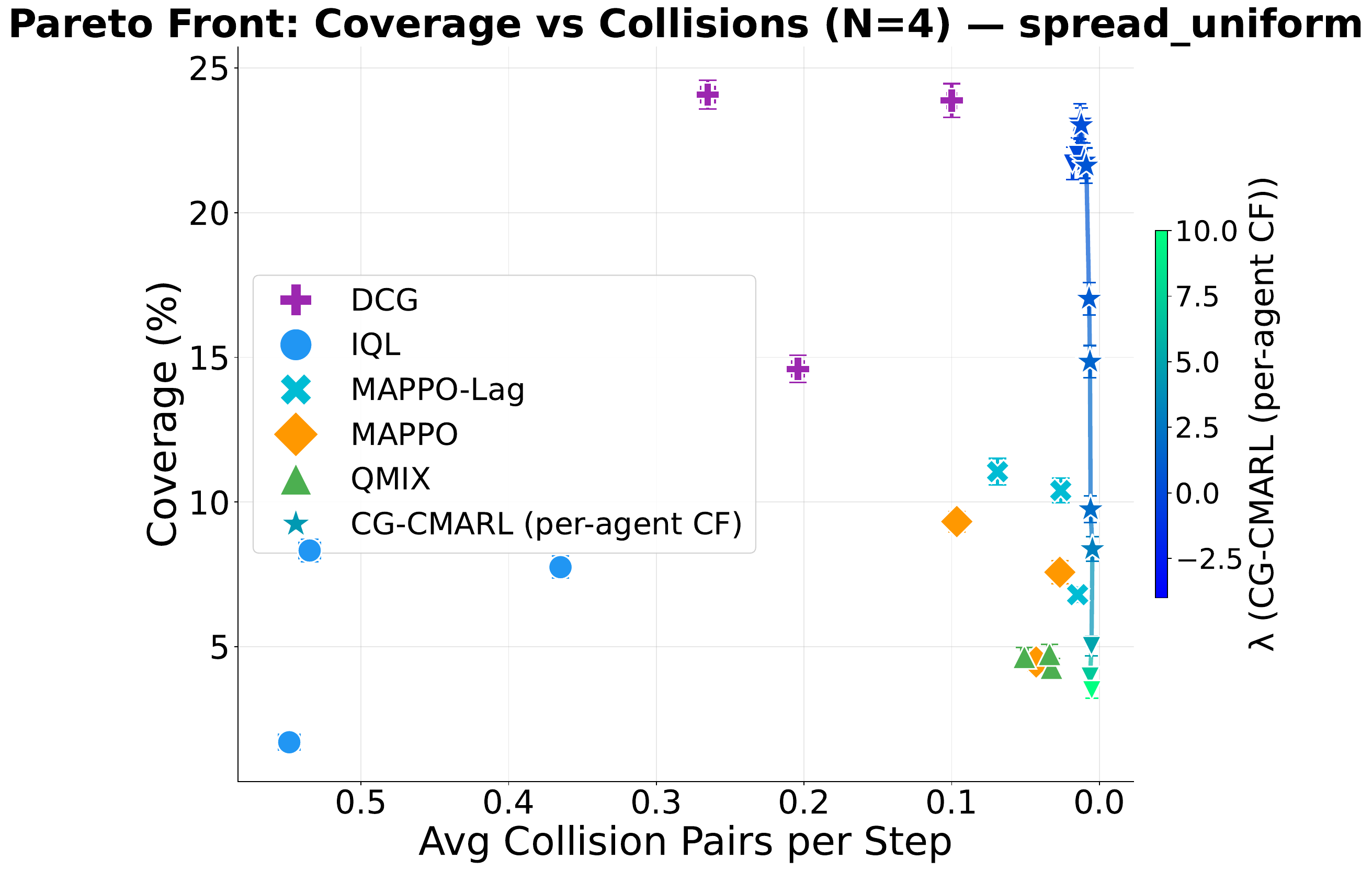}
        \caption{\texttt{spread\_uniform}}
    \end{subfigure}
    \hfill
    \begin{subfigure}{0.32\textwidth}
        \centering
        \includegraphics[width=\linewidth]{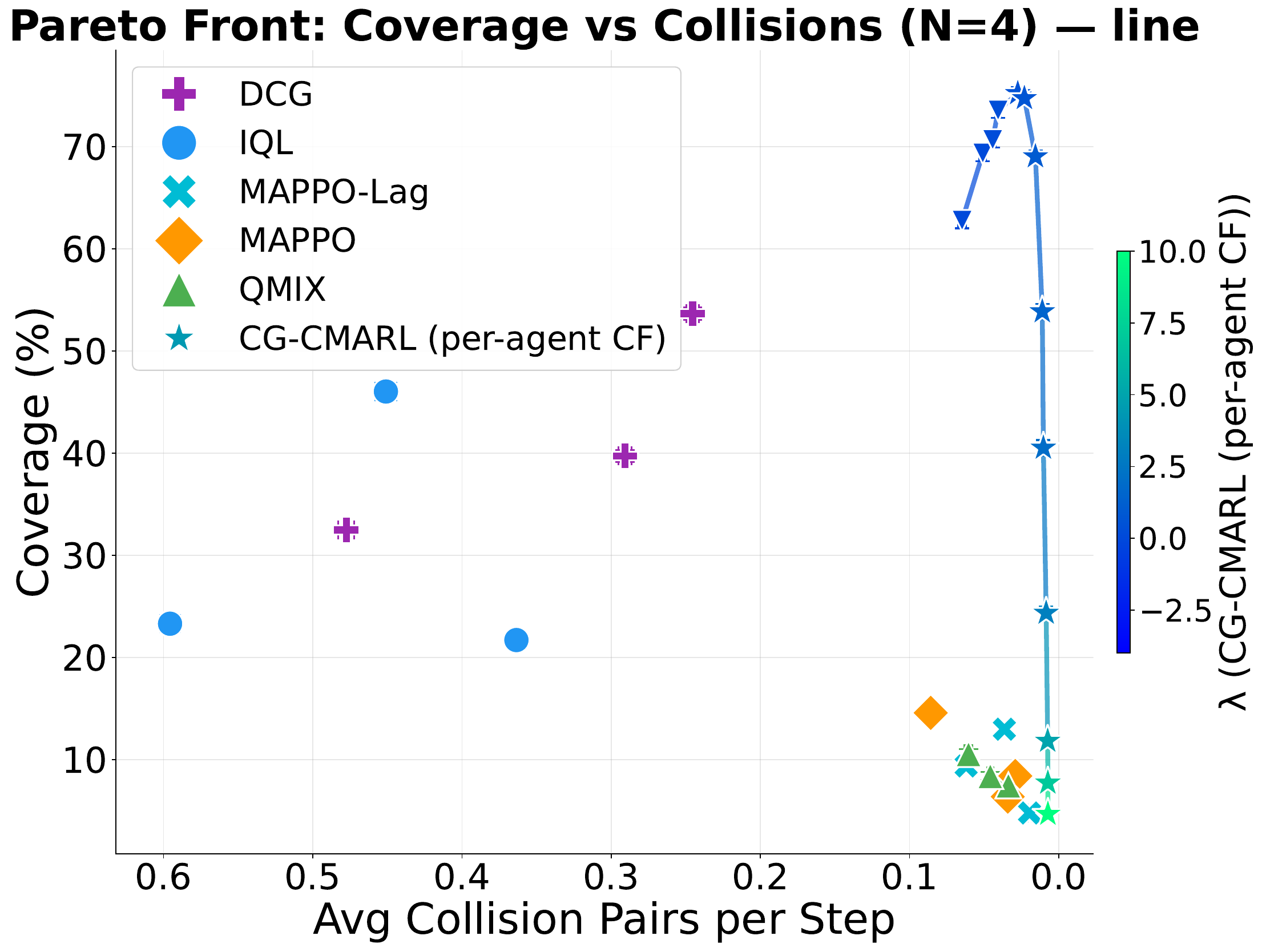}
        \caption{\texttt{line}}
    \end{subfigure}
    \hfill
    \begin{subfigure}{0.32\textwidth}
        \centering
        \includegraphics[width=\linewidth]{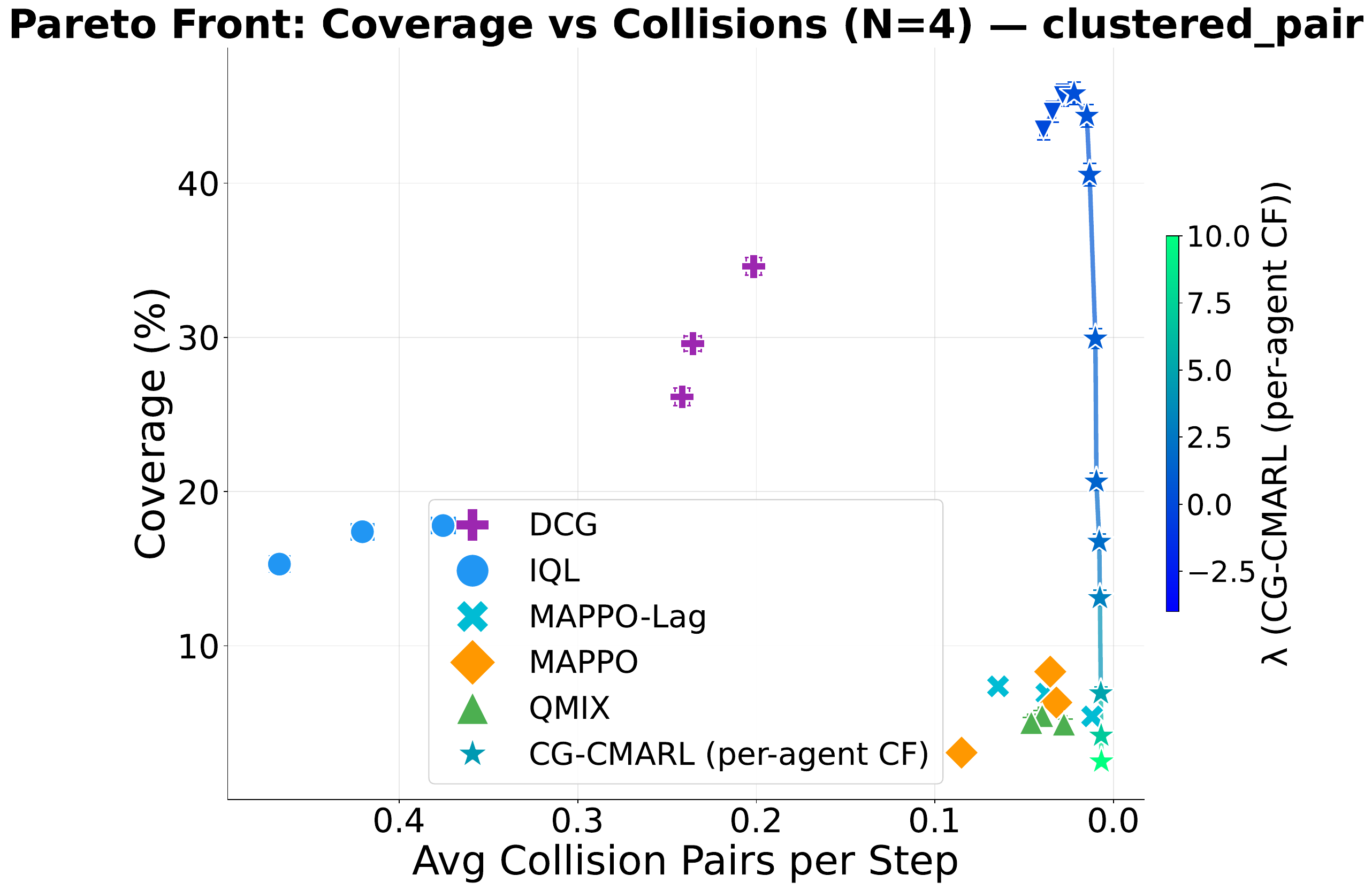}
        \caption{\texttt{clustered\_pair}}
    \end{subfigure}
    \caption{Pareto fronts on fixed landmark scenarios for $N=4$.}
    \label{fig:scenario-pareto-N4}
\end{figure}

\begin{figure}[htbp]
    \centering
    \begin{subfigure}{0.32\textwidth}
        \centering
        \includegraphics[width=\linewidth]{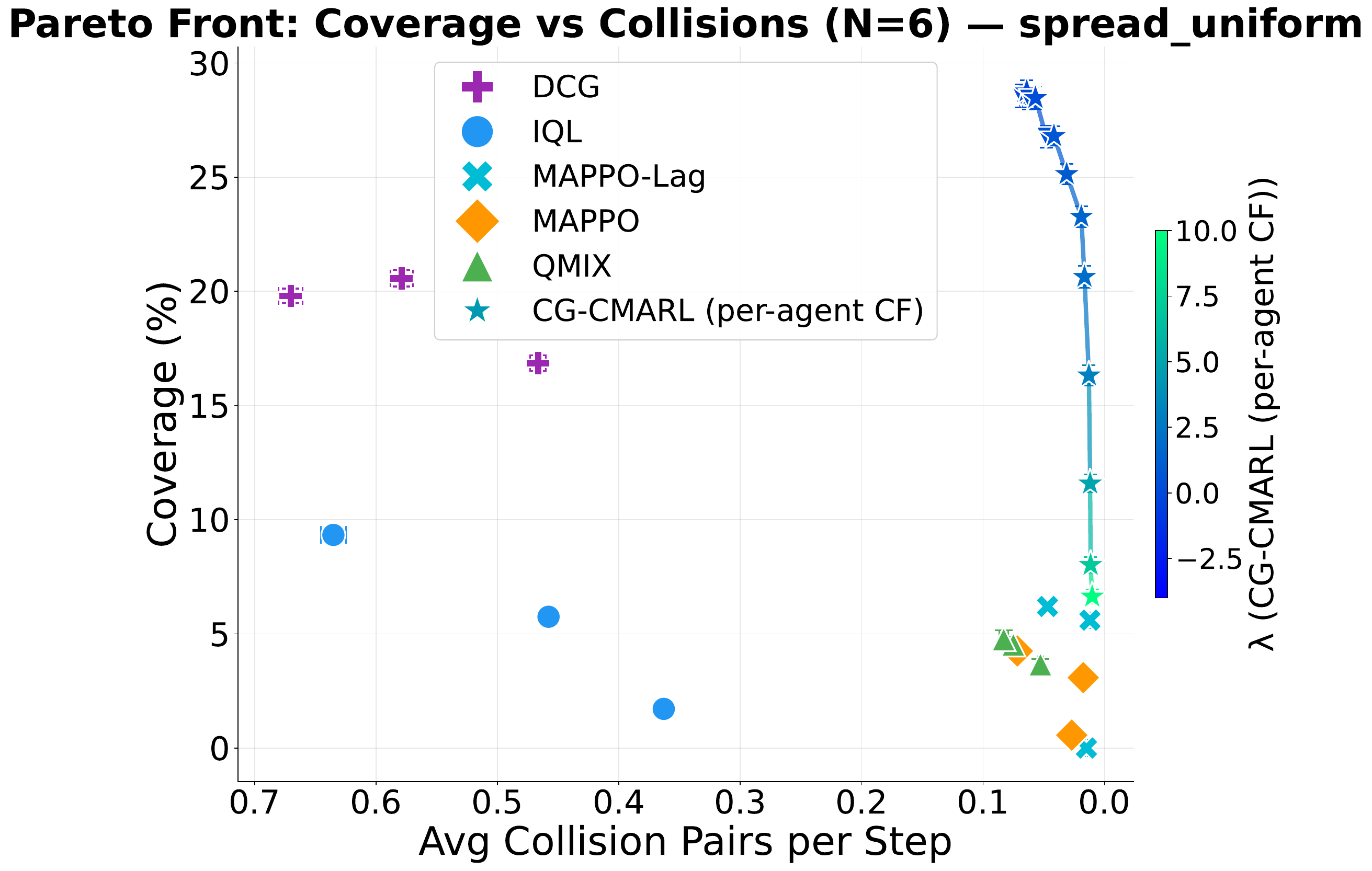}
        \caption{\texttt{spread\_uniform}}
    \end{subfigure}
    \hfill
    \begin{subfigure}{0.32\textwidth}
        \centering
        \includegraphics[width=\linewidth]{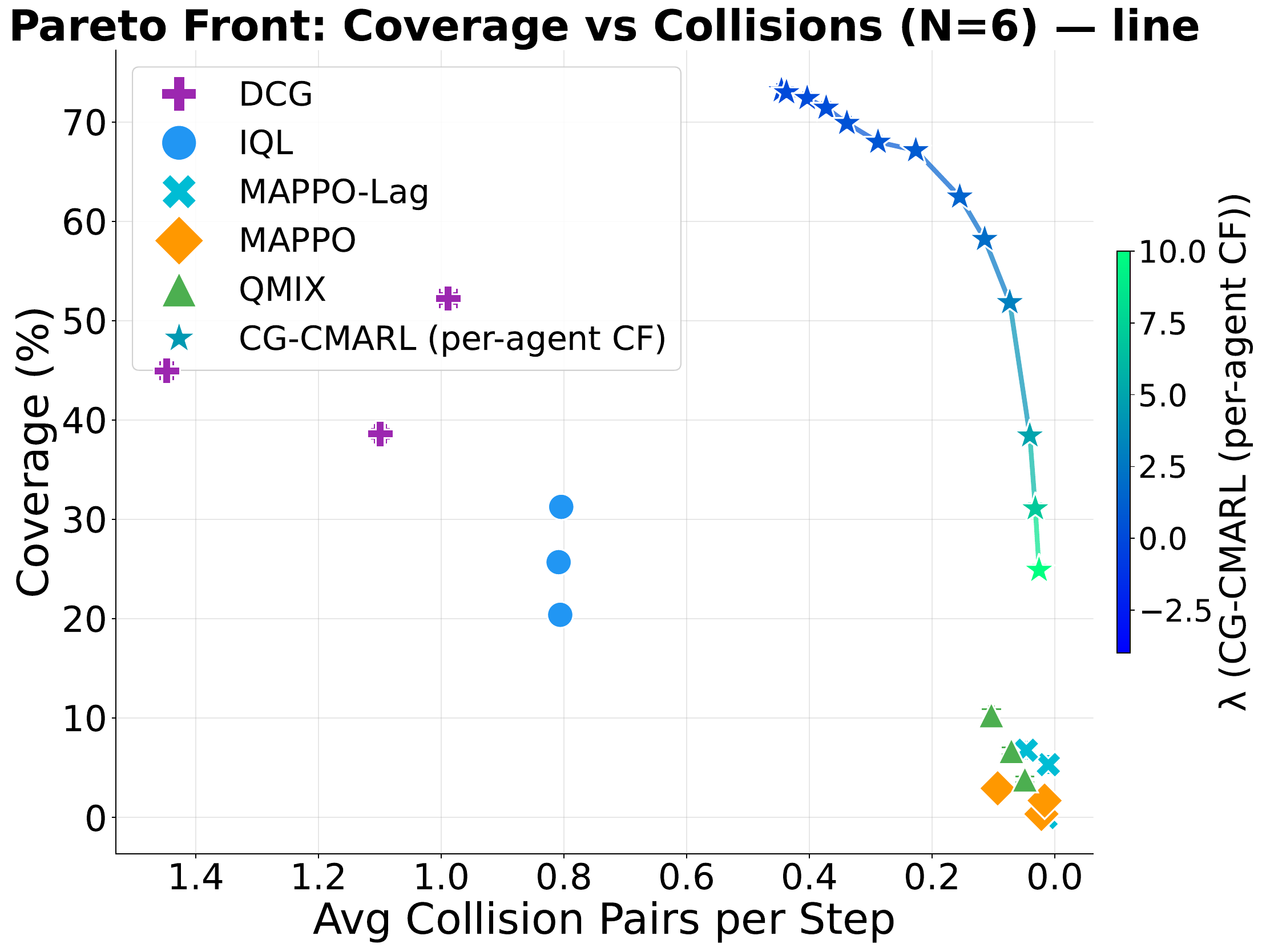}
        \caption{\texttt{line}}
    \end{subfigure}
    \hfill
    \begin{subfigure}{0.32\textwidth}
        \centering
        \includegraphics[width=\linewidth]{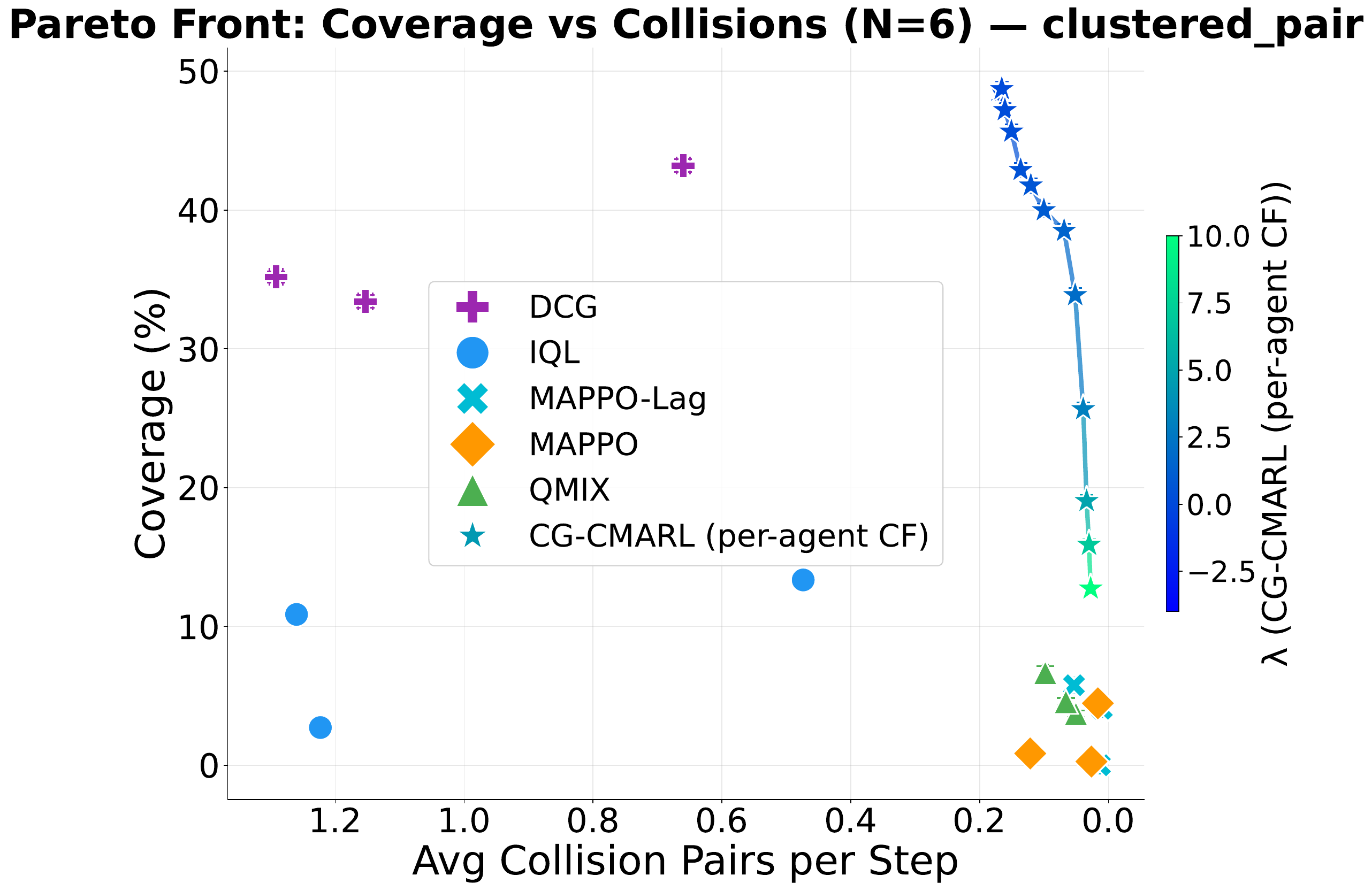}
        \caption{\texttt{clustered\_pair}}
    \end{subfigure}
    \caption{Pareto fronts on fixed landmark scenarios for $N=6$.}
    \label{fig:scenario-pareto-N6}
\end{figure}

\begin{figure}[htbp]
    \centering
    \begin{subfigure}{0.32\textwidth}
        \centering
        \includegraphics[width=\linewidth]{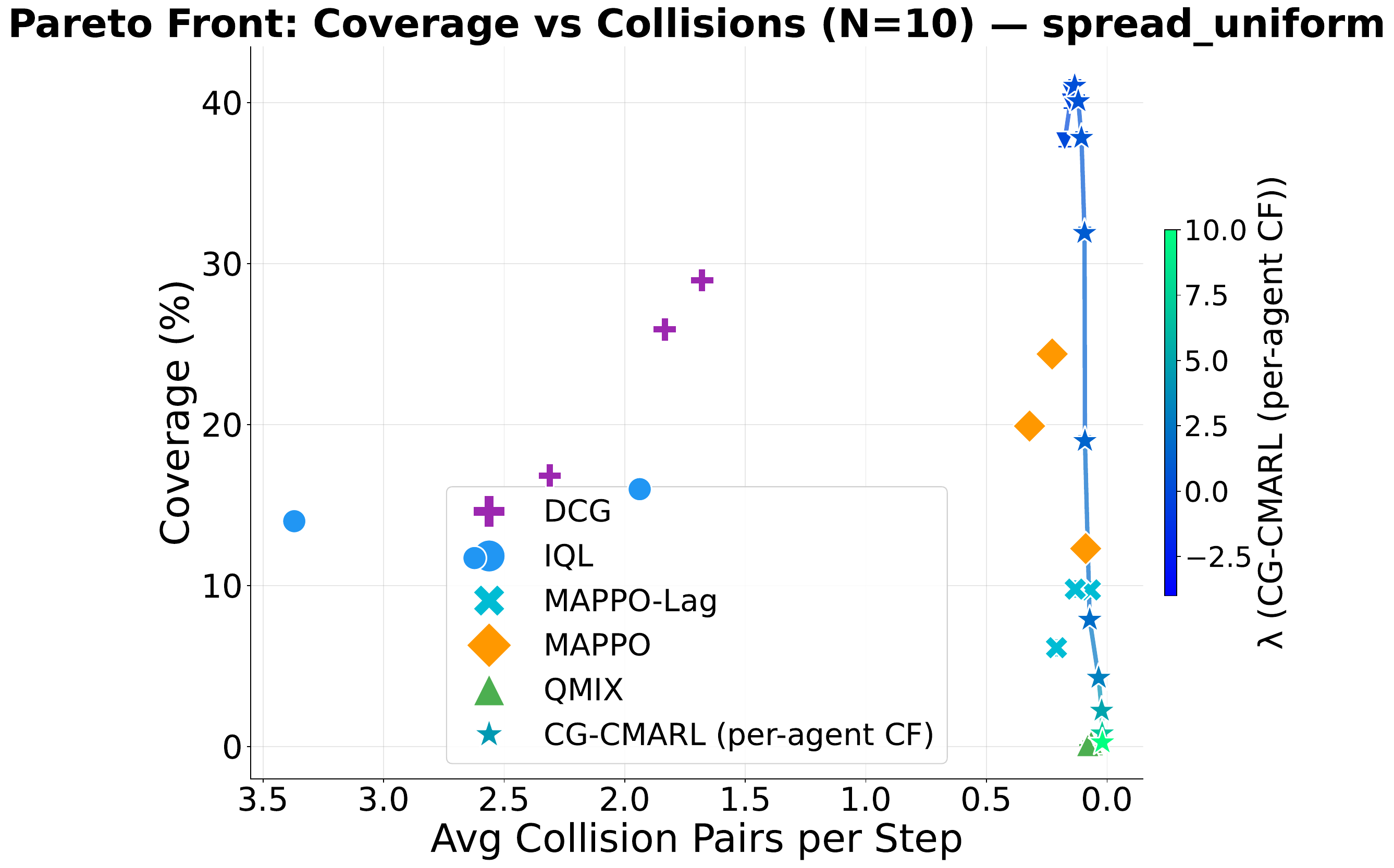}
        \caption{\texttt{spread\_uniform}}
    \end{subfigure}
    \hfill
    \begin{subfigure}{0.32\textwidth}
        \centering
        \includegraphics[width=\linewidth]{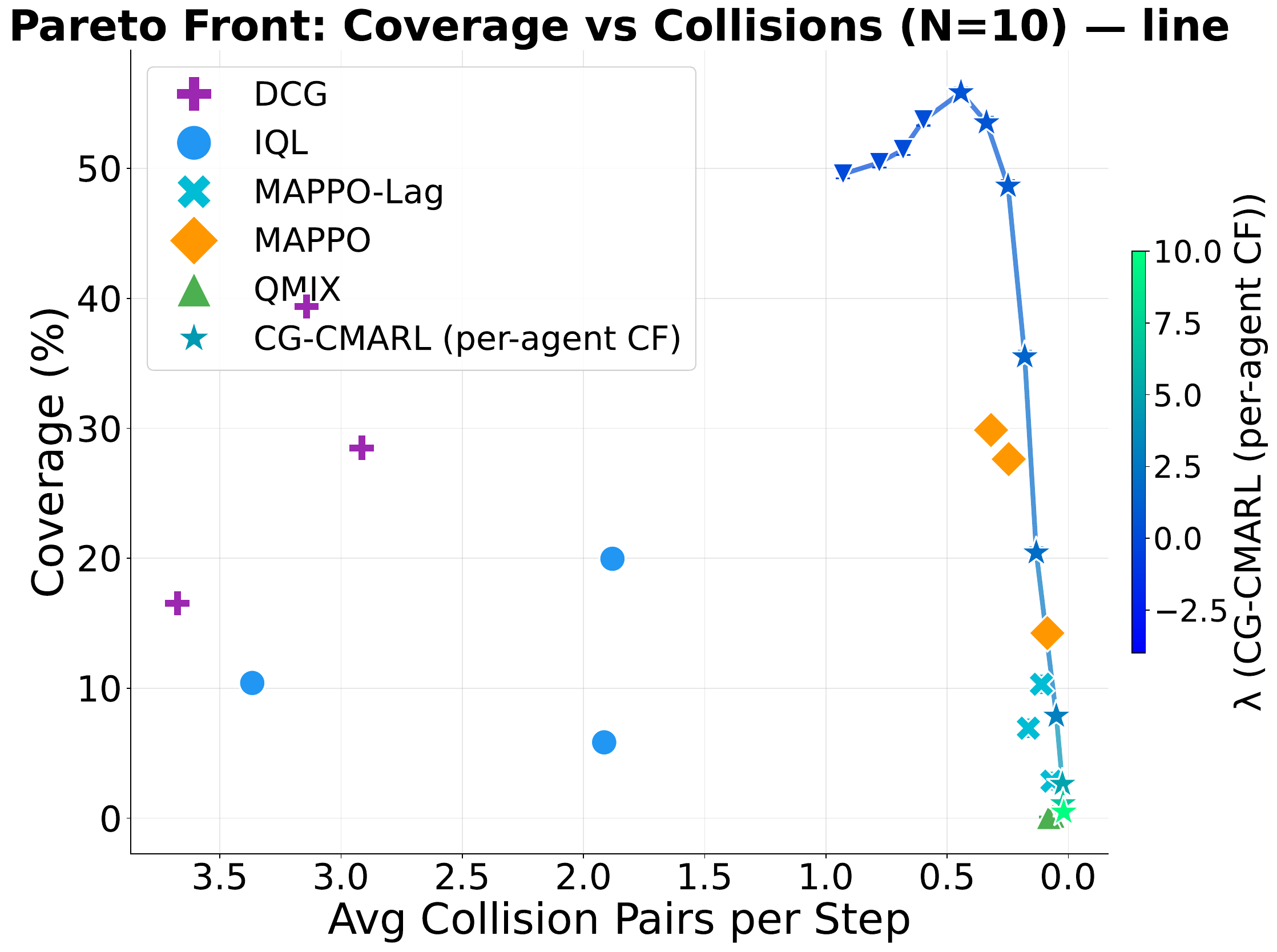}
        \caption{\texttt{line}}
    \end{subfigure}
    \hfill
    \begin{subfigure}{0.32\textwidth}
        \centering
        \includegraphics[width=\linewidth]{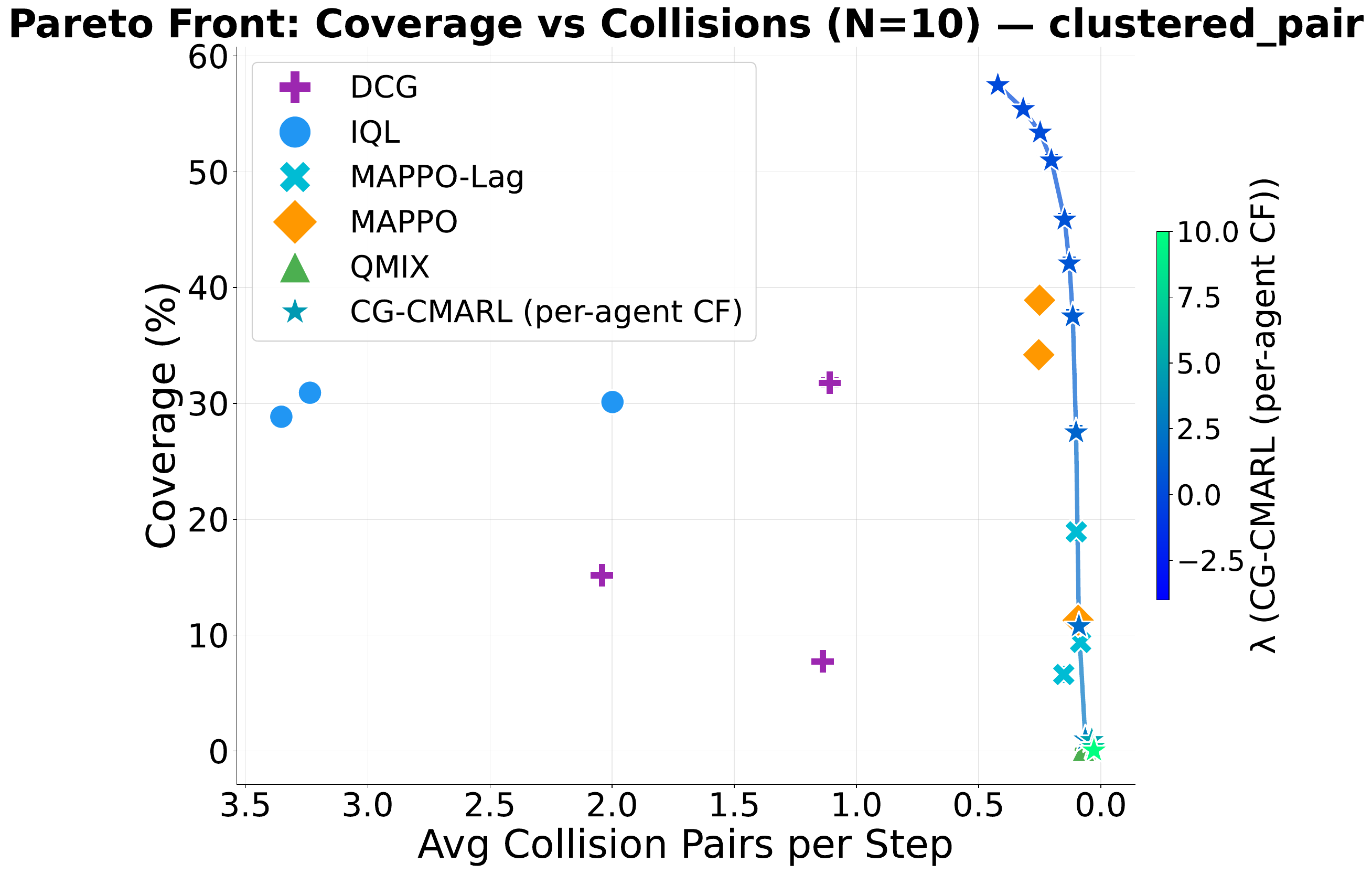}
        \caption{\texttt{clustered\_pair}}
    \end{subfigure}
    \caption{Pareto fronts on fixed landmark scenarios for $N=10$.}
    \label{fig:scenario-pareto-N10}
\end{figure}

\subsection{Safety Metric and Fairness Across Team Sizes}
\label{app:metric-fairness}

The safety metric reported throughout the paper is the
\emph{average number of overlapping agent pairs per step}: at each
timestep we count unordered pairs $(i,k)$ whose Euclidean distance
is below $2\,r_{\mathrm{agent}}$, then average over steps and
episodes. This quantity is exactly the cost signal the
Lagrangian multiplier penalizes during training (the cost head
predicts it per region), so reporting it as the evaluation metric
keeps the training objective and the evaluation objective aligned.

The metric has one property that is worth making explicit: its
maximum possible value is $\binom{N}{2}$, which grows quadratically
in the team size. A raw value of $0.05$ collisions per step therefore
carries different weight at $N=3$ (where the maximum is $3$) than at
$N=10$ (where the maximum is $45$). This means cross-$N$ comparisons
using the raw value understate how well the algorithm scales.

To disentangle ``total safety burden on the system'' from ``per-pair
safety,'' Tables~\ref{tab:baselines-all}--\ref{tab:cgcmarl-all}
report both the raw metric and a normalized \emph{per-pair} rate
defined as
\begin{equation}
\label{eq:per-pair-rate}
\text{per-pair rate} \;:=\;
\frac{\text{Collisions}}{\binom{N}{2}}.
\end{equation}
The per-pair rate admits a clean interpretation: it is the
probability that a uniformly-random pair of agents is in collision
at a uniformly-random timestep.

\paragraph{Empirical scaling of CG-CMARL.}
Table~\ref{tab:per-pair-scaling} compares CG-CMARL at the same
multiplier ($\lambda = 1$) across team sizes. In raw units the
collision count grows by roughly $15\times$ from $N=3$ to $N=10$,
suggesting that safety degrades sharply with team size. The per-pair
rate tells a very different story: it stays essentially flat, between
$0.0011$ and $0.0019$. That is, CG-CMARL maintains a nearly
constant per-pair collision probability as the team scales; the
apparent degradation in the raw metric is almost entirely a
combinatorial artifact of there being more pairs to track.

\begin{table}[h]
\centering
\small
\caption{CG-CMARL at $\lambda = 1$: raw collisions vs.\ per-pair
rate as team size grows. The raw metric's $\approx 15\times$ increase
reflects the quadratic growth in $\binom{N}{2}$, not a genuine
degradation of per-pair safety.}
\label{tab:per-pair-scaling}
\begin{tabular}{@{}rrcr@{}}
\toprule
$N$ & $\binom{N}{2}$ & \textbf{Collisions} & \textbf{Per-pair rate} \\
\midrule
 3 &  3 & 0.0034 & 0.00113 \\
 4 &  6 & 0.0103 & 0.00172 \\
 6 & 15 & 0.0289 & 0.00193 \\
10 & 45 & 0.0534 & 0.00119 \\
\bottomrule
\end{tabular}
\end{table}

We keep the raw pair count as the primary metric because it matches
the training objective, but the per-pair rate is the more
appropriate quantity for cross-$N$ comparisons and is reported
alongside in the tables below.

\subsection{Full Numerical Results}
\label{app:numerical-results}

Tables~\ref{tab:baselines-all}--\ref{tab:cgcmarl-all} report the
full numerical results summarized by the Pareto front plots in
Section~\ref{sec:exp-results}. Collisions denotes the average number
of colliding agent pairs per step (lower is better), matching the
x-axis of the Pareto plots; coverage is reported as a percentage.
All values are averaged over 2 random seeds; $\pm$ denotes one
sample standard deviation. For CG-CMARL, $\bigstar$ marks
Pareto-optimal operating points. These tables are auto-generated by
\texttt{scripts/generate\_appendix\_tables.py} directly from the
per-seed \texttt{pareto\_sweep.json} files.

\input{figures/tables/baselines_all.tex}

\input{figures/tables/cgcmarl_all.tex}

%% file: figures/tables/baselines_all.tex
{\small
\begin{longtable}{@{}ll r@{$\,\pm\,$}l r@{$\,\pm\,$}l r@{$\,\pm\,$}l@{}}
\caption{Baseline results across all team sizes. ``lr'' =
\texttt{local\_ratio}; ``cl'' = \texttt{cost\_limit} (for
MAPPO-Lagrangian). Collisions is the average number of
colliding agent pairs per step (lower is better);
per-pair rate is Collisions divided by $\binom{N}{2}$ and
allows fair comparison across $N$ (see
Appendix~\ref{app:metric-fairness}); coverage is reported
as a percentage.}
\label{tab:baselines-all} \\
\toprule
\textbf{Algorithm} & \textbf{Param.} &
\multicolumn{2}{c}{\textbf{Collisions}} &
\multicolumn{2}{c}{\textbf{Per-pair}} &
\multicolumn{2}{c}{\textbf{Coverage (\%)}} \\
\midrule
\endfirsthead
\multicolumn{8}{@{}l}{\small\itshape (continued from previous page)} \\
\toprule
\textbf{Algorithm} & \textbf{Param.} &
\multicolumn{2}{c}{\textbf{Collisions}} &
\multicolumn{2}{c}{\textbf{Per-pair}} &
\multicolumn{2}{c}{\textbf{Coverage (\%)}} \\
\midrule
\endhead
\midrule
\multicolumn{8}{r@{}}{\small\itshape (continued on next page)} \\
\endfoot
\bottomrule
\endlastfoot
\multicolumn{8}{@{}l}{\textit{N = 3 (3 pairwise regions)}} \\
\addlinespace[2pt]
IQL         & lr=0.0 & 0.0519 & 0.0081 & 0.0173 & 0.0027 & 20.58 & 0.82 \\
IQL         & lr=0.3 & 0.0623 & 0.0075 & 0.0208 & 0.0025 & 17.42 & 3.42 \\
IQL         & lr=0.6 & 0.0483 & 0.0001 & 0.0161 & 0.0000 & 15.67 & 0.47 \\
QMIX        & lr=0.0 & 0.0052 & 0.0008 & 0.0017 & 0.0003 &  7.83 & 0.42 \\
QMIX        & lr=0.3 & 0.0060 & 0.0005 & 0.0020 & 0.0002 &  5.60 & 0.66 \\
QMIX        & lr=0.6 & 0.0071 & 0.0009 & 0.0024 & 0.0003 &  5.83 & 0.24 \\
DCG         & lr=0.0 & 0.0450 & 0.0040 & 0.0150 & 0.0013 & 21.58 & 3.42 \\
DCG         & lr=0.3 & 0.0439 & 0.0081 & 0.0146 & 0.0027 & 23.50 & 1.89 \\
DCG         & lr=0.6 & 0.0322 & 0.0003 & 0.0107 & 0.0001 & 24.17 & 2.36 \\
MAPPO       & lr=0.0 & 0.0232 & 0.0240 & 0.0077 & 0.0080 &  9.84 & 5.33 \\
MAPPO       & lr=0.3 & 0.0126 & 0.0003 & 0.0042 & 0.0001 &  8.00 & 2.36 \\
MAPPO       & lr=0.6 & 0.0084 & 0.0058 & 0.0028 & 0.0019 & 10.90 & 3.64 \\
MAPPO-Lag   & cl=0.1 & 0.0185 & 0.0047 & 0.0062 & 0.0016 & 10.53 & 1.51 \\
MAPPO-Lag   & cl=1.0 & 0.0067 & 0.0012 & 0.0022 & 0.0004 &  9.97 & 0.33 \\
\addlinespace[4pt]
\midrule
\multicolumn{8}{@{}l}{\textit{N = 4 (6 pairwise regions)}} \\
\addlinespace[2pt]
IQL         & lr=0.0 & 0.0837 & 0.0100 & 0.0140 & 0.0017 & 17.19 & 0.97 \\
IQL         & lr=0.3 & 0.0868 & 0.0011 & 0.0145 & 0.0002 & 16.00 & 1.94 \\
IQL         & lr=0.6 & 0.0703 & 0.0047 & 0.0117 & 0.0008 & 16.56 & 1.68 \\
QMIX        & lr=0.0 & 0.0182 & 0.0005 & 0.0030 & 0.0001 &  3.70 & 0.35 \\
QMIX        & lr=0.3 & 0.0177 & 0.0008 & 0.0030 & 0.0001 &  3.95 & 0.14 \\
QMIX        & lr=0.6 & 0.0187 & 0.0021 & 0.0031 & 0.0004 &  4.72 & 0.25 \\
DCG         & lr=0.0 & 0.0856 & 0.0181 & 0.0143 & 0.0030 & 21.50 & 4.42 \\
DCG         & lr=0.3 & 0.0858 & 0.0051 & 0.0143 & 0.0008 & 24.69 & 1.33 \\
DCG         & lr=0.6 & 0.0756 & 0.0020 & 0.0126 & 0.0003 & 23.56 & 0.62 \\
MAPPO       & lr=0.0 & 0.0161 & 0.0096 & 0.0027 & 0.0016 & 10.61 & 3.79 \\
MAPPO       & lr=0.3 & 0.0171 & 0.0123 & 0.0029 & 0.0021 &  5.19 & 2.21 \\
MAPPO       & lr=0.6 & 0.0172 & 0.0183 & 0.0029 & 0.0031 &  9.88 & 2.47 \\
MAPPO-Lag   & cl=0.1 & 0.0122 & 0.0032 & 0.0020 & 0.0005 &  7.97 & 0.18 \\
MAPPO-Lag   & cl=1.0 & 0.0139 & 0.0119 & 0.0023 & 0.0020 & 11.18 & 2.65 \\
\addlinespace[4pt]
\midrule
\multicolumn{8}{@{}l}{\textit{N = 6 (15 pairwise regions)}} \\
\addlinespace[2pt]
IQL         & lr=0.0 & 0.1058 & 0.0060 & 0.0071 & 0.0004 & 15.83 & 2.36 \\
IQL         & lr=0.3 & 0.0960 & 0.0141 & 0.0064 & 0.0009 & 14.21 & 0.53 \\
IQL         & lr=0.6 & 0.0802 & 0.0032 & 0.0053 & 0.0002 & 12.87 & 0.06 \\
QMIX        & lr=0.0 & 0.0217 & 0.0021 & 0.0014 & 0.0001 &  1.85 & 0.07 \\
QMIX        & lr=0.3 & 0.0243 & 0.0021 & 0.0016 & 0.0001 &  1.80 & 0.14 \\
QMIX        & lr=0.6 & 0.0313 & 0.0013 & 0.0021 & 0.0001 &  2.88 & 0.21 \\
DCG         & lr=0.0 & 0.1425 & 0.0193 & 0.0095 & 0.0013 & 24.17 & 0.24 \\
DCG         & lr=0.3 & 0.1458 & 0.0051 & 0.0097 & 0.0003 & 24.00 & 0.47 \\
DCG         & lr=0.6 & 0.1544 & 0.0025 & 0.0103 & 0.0002 & 24.17 & 1.65 \\
MAPPO       & lr=0.0 & 0.0107 & 0.0020 & 0.0007 & 0.0001 &  2.82 & 3.65 \\
MAPPO       & lr=0.3 & 0.0174 & 0.0044 & 0.0012 & 0.0003 &  2.88 & 2.77 \\
MAPPO       & lr=0.6 & 0.0183 & 0.0078 & 0.0012 & 0.0005 &  7.23 & 3.96 \\
MAPPO-Lag   & cl=0.1 & 0.0188 & 0.0083 & 0.0013 & 0.0006 &  4.48 & 3.13 \\
MAPPO-Lag   & cl=1.0 & 0.0310 & 0.0055 & 0.0021 & 0.0004 &  7.37 & 3.63 \\
\addlinespace[4pt]
\midrule
\multicolumn{8}{@{}l}{\textit{N = 10 (45 pairwise regions)}} \\
\addlinespace[2pt]
IQL         & lr=0.0 & 0.2087 & 0.0010 & 0.0046 & 0.0000 & 16.73 & 1.87 \\
IQL         & lr=0.3 & 0.2035 & 0.0068 & 0.0045 & 0.0002 & 15.80 & 0.14 \\
IQL         & lr=0.6 & 0.1759 & 0.0029 & 0.0039 & 0.0001 & 15.12 & 2.09 \\
QMIX        & lr=0.0 & 0.0365 & 0.0090 & 0.0008 & 0.0002 &  0.86 & 1.10 \\
QMIX        & lr=0.3 & 0.0193 & 0.0017 & 0.0004 & 0.0000 &  0.00 & 0.00 \\
QMIX        & lr=0.6 & 0.0222 & 0.0038 & 0.0005 & 0.0001 &  0.00 & 0.00 \\
DCG         & lr=0.0 & 0.2272 & 0.0083 & 0.0050 & 0.0002 & 22.43 & 1.17 \\
DCG         & lr=0.3 & 0.2263 & 0.0015 & 0.0050 & 0.0000 & 21.02 & 0.95 \\
DCG         & lr=0.6 & 0.2145 & 0.0054 & 0.0048 & 0.0001 & 19.28 & 0.25 \\
MAPPO       & lr=0.0 & 0.0798 & 0.0408 & 0.0018 & 0.0009 & 17.03 & 3.21 \\
MAPPO       & lr=0.6 & 0.0387 & 0.0239 & 0.0009 & 0.0005 & 17.40 & 4.27 \\
MAPPO-Lag   & cl=0.5 & 0.0247 & 0.0070 & 0.0005 & 0.0002 & 12.59 & 12.26 \\
\end{longtable}
}

%% file: figures/tables/cgcmarl_all.tex
\begin{table}[htbp]
\centering
\small
\caption{CG-CMARL results across all team sizes. Each row
corresponds to a single trained model evaluated at the indicated
$\lambda$. Collisions is the average number of colliding
agent pairs per step (lower is better); per-pair rate is
Collisions divided by $\binom{N}{2}$ for fair cross-$N$
comparison (see Appendix~\ref{app:metric-fairness}).
$\bigstar$ = Pareto-optimal operating point.}
\label{tab:cgcmarl-all}
\begin{tabular}{@{}cl r@{$\,\pm\,$}l r@{$\,\pm\,$}l r@{$\,\pm\,$}l c@{}}
\toprule
$N$ & $\lambda$ &
\multicolumn{2}{c}{\textbf{Collisions}} &
\multicolumn{2}{c}{\textbf{Per-pair}} &
\multicolumn{2}{c}{\textbf{Coverage (\%)}} &
\textbf{Pareto} \\
\midrule
\multirow{9}{*}{3} & 0.00  & 0.0055 & 0.0001 & 0.0018 & 0.0000 & 47.58 & 2.24 & $\bigstar$ \\
 & 0.05  & 0.0054 & 0.0003 & 0.0018 & 0.0001 & 46.58 & 6.01 &  \\
 & 0.10  & 0.0052 & 0.0006 & 0.0017 & 0.0002 & 47.33 & 2.59 & $\bigstar$ \\
 & 0.20  & 0.0037 & 0.0004 & 0.0012 & 0.0001 & 46.17 & 0.94 & $\bigstar$ \\
 & 0.50  & 0.0042 & 0.0006 & 0.0014 & 0.0002 & 38.08 & 0.12 &  \\
 & 1.00  & 0.0034 & 0.0003 & 0.0011 & 0.0001 & 17.50 & 2.83 & $\bigstar$ \\
 & 2.00  & 0.0030 & 0.0014 & 0.0010 & 0.0005 &  8.25 & 1.53 & $\bigstar$ \\
 & 5.00  & 0.0017 & 0.0001 & 0.0006 & 0.0000 &  3.92 & 2.24 & $\bigstar$ \\
 & 10.00  & 0.0022 & 0.0006 & 0.0007 & 0.0002 &  1.33 & 1.18 &  \\
\addlinespace[4pt]
\midrule
\multirow{9}{*}{4} & 0.00  & 0.0243 & 0.0016 & 0.0041 & 0.0003 & 37.50 & 0.88 &  \\
 & 0.05  & 0.0186 & 0.0034 & 0.0031 & 0.0006 & 38.94 & 1.86 & $\bigstar$ \\
 & 0.10  & 0.0171 & 0.0027 & 0.0029 & 0.0004 & 36.56 & 2.03 &  \\
 & 0.20  & 0.0137 & 0.0007 & 0.0023 & 0.0001 & 38.50 & 0.53 & $\bigstar$ \\
 & 0.50  & 0.0162 & 0.0006 & 0.0027 & 0.0001 & 32.06 & 3.45 &  \\
 & 1.00  & 0.0103 & 0.0021 & 0.0017 & 0.0004 & 26.81 & 1.33 & $\bigstar$ \\
 & 2.00  & 0.0070 & 0.0011 & 0.0012 & 0.0002 & 16.50 & 3.89 & $\bigstar$ \\
 & 5.00  & 0.0075 & 0.0013 & 0.0013 & 0.0002 &  6.75 & 0.53 &  \\
 & 10.00  & 0.0064 & 0.0011 & 0.0011 & 0.0002 &  4.38 & 1.24 & $\bigstar$ \\
\addlinespace[4pt]
\midrule
\multirow{9}{*}{6} & 0.00  & 0.0413 & 0.0022 & 0.0028 & 0.0001 & 33.29 & 3.01 &  \\
 & 0.05  & 0.0457 & 0.0006 & 0.0030 & 0.0000 & 33.75 & 1.65 &  \\
 & 0.10  & 0.0390 & 0.0012 & 0.0026 & 0.0001 & 34.83 & 0.24 & $\bigstar$ \\
 & 0.20  & 0.0428 & 0.0021 & 0.0029 & 0.0001 & 33.92 & 1.89 &  \\
 & 0.50  & 0.0299 & 0.0020 & 0.0020 & 0.0001 & 33.21 & 0.77 & $\bigstar$ \\
 & 1.00  & 0.0289 & 0.0015 & 0.0019 & 0.0001 & 28.67 & 0.35 & $\bigstar$ \\
 & 2.00  & 0.0172 & 0.0016 & 0.0011 & 0.0001 & 22.38 & 2.89 & $\bigstar$ \\
 & 5.00  & 0.0153 & 0.0012 & 0.0010 & 0.0001 & 11.75 & 1.30 & $\bigstar$ \\
 & 10.00  & 0.0120 & 0.0002 & 0.0008 & 0.0000 &  7.33 & 0.94 & $\bigstar$ \\
\addlinespace[4pt]
\midrule
\multirow{8}{*}{10} & 0.00  & 0.1223 & 0.0177 & 0.0027 & 0.0004 & 27.14 & 1.47 &  \\
 & 0.10  & 0.1167 & 0.0130 & 0.0026 & 0.0003 & 27.18 & 0.03 &  \\
 & 0.20  & 0.1074 & 0.0124 & 0.0024 & 0.0003 & 27.64 & 1.41 & $\bigstar$ \\
 & 0.50  & 0.0815 & 0.0144 & 0.0018 & 0.0003 & 24.69 & 0.18 & $\bigstar$ \\
 & 1.00  & 0.0534 & 0.0114 & 0.0012 & 0.0003 & 21.80 & 0.71 & $\bigstar$ \\
 & 2.00  & 0.0343 & 0.0035 & 0.0008 & 0.0001 & 12.24 & 1.73 & $\bigstar$ \\
 & 5.00  & 0.0198 & 0.0010 & 0.0004 & 0.0000 &  1.85 & 0.41 & $\bigstar$ \\
 & 10.00  & 0.0157 & 0.0002 & 0.0003 & 0.0000 &  0.52 & 0.14 & $\bigstar$ \\
\bottomrule
\end{tabular}
\end{table}